\documentclass{article}

\usepackage{iclr2027_conference,times}

\usepackage[T1]{fontenc}
\usepackage[utf8]{inputenc}
\usepackage[final]{microtype}
\usepackage[dvipsnames]{xcolor}
\usepackage{amsmath,amssymb,amsfonts,amsthm,mathtools}
\usepackage{graphicx}
\usepackage{subcaption}
\usepackage{booktabs}
\usepackage{multirow}
\usepackage{makecell}
\usepackage{array}
\usepackage{enumitem}
\usepackage{float}
\usepackage{placeins}
\usepackage[normalem]{ulem}
\usepackage{algorithm}
\usepackage{tabularx}
\usepackage{algpseudocode}
\usepackage{tikz}
\usetikzlibrary{
  arrows.meta,
  calc,
  decorations.pathreplacing,
  backgrounds,
  intersections
}
\usepackage{capt-of}
\usepackage[colorlinks=true,linkcolor=blue,citecolor=blue,urlcolor=blue]{hyperref}
\usepackage[nameinlink,noabbrev]{cleveref}
\usepackage{wrapfig}
\usepackage{stmaryrd}
\newcolumntype{Y}{>{\raggedright\arraybackslash}X}
\usepackage{eso-pic}

\definecolor{surfblue}{RGB}{20,86,196}
\definecolor{spherefill}{RGB}{120,168,240}
\definecolor{hitgreen}{RGB}{22,148,58}
\definecolor{measred}{RGB}{206,32,32}
\definecolor{sensordk}{RGB}{20,20,20}
\definecolor{raygray}{RGB}{120,120,120}
\definecolor{acqraygray}{RGB}{80,80,80}
\definecolor{epsviolet}{RGB}{126,47,142}
\definecolor{certspace}{RGB}{248,226,188}
\theoremstyle{plain}
\newtheorem{proposition}{Proposition}

\theoremstyle{definition}

\newtheorem{remark}{Remark}

\newcommand{\norm}[1]{\left\lVert #1\right\rVert}

\newcommand{\Lip}{\operatorname{Lip}}

\newif\ifdraft
\draftfalse

\title{Surface Reconstruction from Range Scans\\
Under Visibility Constraints}

\author{\normalfont
\makebox[\textwidth][c]{%
  \begin{tabular}[t]{c}
    \textbf{Gabriel Lenain}\\
    Inria Paris \\
    École normale supérieure - PSL\\
    \texttt{gabriel.lenain@inria.fr}
  \end{tabular}\hfill
  \begin{tabular}[t]{c}
    \textbf{Edmond Boyer}\\
    Univ.\ Grenoble Alpes, Inria\\
    CNRS, Grenoble INP, LJK\\
    \texttt{edmond.boyer@inria.fr}
  \end{tabular}}\\[1.5em]
\makebox[\textwidth][c]{%
  \begin{tabular}[t]{c}
    \textbf{Jean Ponce}\\
    École normale supérieure - PSL\\
    New York University\\
    \texttt{jean.ponce@ens.fr}
  \end{tabular}}
}

\date{}
\iclrfinalcopy

\begin{document}
\maketitle
\AddToShipoutPictureFG*{
  \AtPageLowerLeft{
    \put(\LenToUnit{3.78cm},\LenToUnit{0.8cm}){%
      \footnotesize Project page:
      \url{https://gabriellenain.github.io/visibility-recon/}
    }
  }
}
\lhead{}
\fancyhead[L]{Preprint}

\begin{abstract}
Surface reconstruction from range scans is challenging in the presence of noise and outliers, as fidelity to observations must be balanced against sensitivity to measurement errors. Yet, most neural approaches discard the sensor positions and corresponding visibility information, compensating for this loss with geometric priors or regularization terms. We instead formulate surface reconstruction as robust regression of the predicted first surface intersection along each acquisition ray against its measured depth, which penalizes predicted surface intersections within the segment between the sensor and the observed point. We represent the surface as the zero set of a 1-Lipschitz multi-layer perceptron, allowing the use of linearly convergent sphere tracing to compute the first hit, and use a robust Huber loss for regression. Our formulation requires neither oriented normals, sampled off-surface points, nor an eikonal loss. Comparisons with the state of the art on synthetic and real data with varying levels of noise and outlier contamination demonstrate competitive performance on clean data and much improved fidelity and robustness under even moderate corruption. Our analysis further confirms a trade-off between fidelity and robustness: greater fitting capacity improves data fidelity in the clean case but increases sensitivity to the measurement errors inherent in real acquisitions, with a sweet spot typically emerging at intermediate capacity. 
\end{abstract}
\vspace{-5mm}
\begin{figure}[H]
    \centering
    \includegraphics[width=0.88\linewidth]{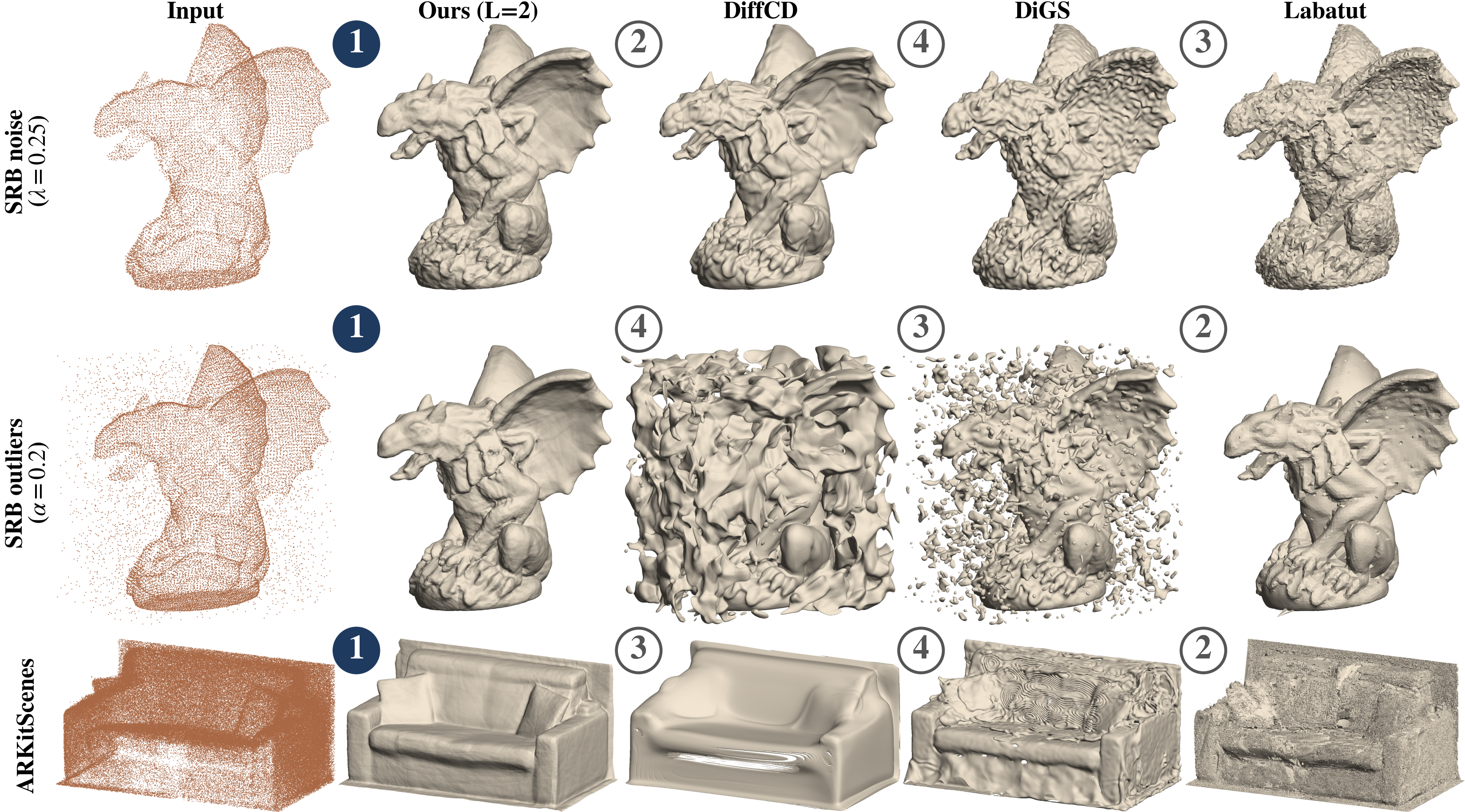}
    \caption{Surface reconstruction for synthetic data corrupted by (top) noise and
(center) outliers, and for (bottom) real LiDAR data. A single integer, the positional encoding level $L$, controls the
fidelity-robustness trade-off. We use $L=2$, a good compromise in our experiments. Circled
numbers rank methods by Chamfer distance for synthetic data and mean absolute depth error for real data.
}
    \label{fig:teaser}
\end{figure}
\vspace{-8mm}

\section{Introduction}

Reconstructing a continuous surface from range scans is a classical
inverse problem in geometry processing and 3D vision. Surface samples constrain
the unknown geometry only locally, leaving its topology, inside--outside
orientation, and behavior between or beyond the observations underdetermined.
Classical methods reduce these ambiguities through implicit surface reconstruction~\citep{taubin1994parameterized,hoppe1992surface}, geometric priors or acquisition-aware occupancy reasoning
\citep{faugeras1990representing,kazhdan2006poisson,kazhdan2013screened,labatut2009robust}, while neural
methods either optimize an implicit field independently for each point cloud
\citep{atzmon2020sal,gropp2020igr,ma2021neuralpull,benshabat2022digs} or learn
reconstruction priors across collections of shapes
\citep{park2019deepsdf,boulch2022poco}.
\\
However, most neural implicit methods discard the visibility information provided by the acquisition process. Under a calibrated first-return model (e.g., LiDAR), each
surface measurement is associated with a sensor, and the open segment between them is known
to contain no surface intersection. 
\\
Existing visibility-aware methods account for visibility through discrete volumetric reasoning~\citep{faugeras1990representing,labatut2009robust,jancosek2014exploiting} or auxiliary ray-based supervision~\citep{sulzer2022visibility,dai2022neuralpoisson}. In contrast, we explicitly enforce visibility by directly constraining the first surface intersection along each acquisition ray to match the measured depth.
 To find this intersection, we represent the reconstructed surface as the zero level set of a \(1\)-Lipschitz MLP \(f_\theta\), which allows us to use the sphere tracing algorithm \citep{hart1996sphere}, known to converge
to the first zero of \(f_\theta\) along the sensor--point ray at a linear rate. We then minimize the
discrepancy between these predicted depths and the measured ones. This approach implicitly ensures that the line segment between the sensor and the measured point is free of surface fragments.
\\
We evaluate the method on the Surface Reconstruction Benchmark (SRB) and ShapeNet under clean and corrupted acquisitions, as well as on real LiDAR measurements from ARKitScenes \citep{dehghan2021arkitscenes} and KITTI~\citep{geiger2013vision}, and on MVS reconstructions~\citep{aanaes2016large}. We achieve competitive results on clean data  and improved performance under noise and outliers.

Our contributions are:
\begin{enumerate}[leftmargin=*,itemsep=0pt,topsep=2pt]
    \item To the best of our knowledge, the first method to fit an implicit surface to range scans from scratch by directly regressing the first surface intersection along each acquisition ray against the measured depth.
    \item A principled single-loss optimization formulation requiring neither auxiliary geometric supervision nor regularization losses, while retaining competitive results on clean data
and improving robustness under noise and outliers.
    \item An empirical characterization of the fidelity--robustness trade-off, with the positional encoding
level $L$ providing a simple control over sensitivity to corruption.
\end{enumerate}

\section{Related Work}
\label{sec:related-work}

\paragraph{Point-cloud surface reconstruction.}
Early methods fit implicit functions or signed-distance fields to point samples \citep{taubin1994parameterized,hoppe1992surface}. Poisson reconstruction estimates an indicator function from oriented point samples, while Screened Poisson (SPSR) adds interpolation constraints to improve geometric fidelity; both rely on oriented normals, which can be unreliable under noisy or non-uniform sampling~\citep{kazhdan2006poisson,kazhdan2013screened}.
More recently, neural implicit methods have become widely used for surface reconstruction as they provide continuous, resolution-independent representations that handle complex geometry and arbitrary topology. These methods differ mainly in how the field is constrained. A first family learns reconstruction priors across collections of shapes, using a latent-conditioned SDF in DeepSDF, point-conditioned local representations in POCO, or a differentiable Poisson formulation in Shape As Points \citep{park2019deepsdf,boulch2022poco,peng2021sap}. In contrast, direct fitting methods optimize a neural field for each input point cloud. IGR shows that eikonal regularization provides an implicit bias toward plausible surfaces \citep{gropp2020igr}, while sign agnostic learning (SAL) extends this setting to unoriented inputs \citep{atzmon2020sal}. Subsequent methods introduce more explicit geometric mechanisms for constraining the fitted field. Neural-Pull samples queries around the input and pulls them toward nearby observations \citep{ma2021neuralpull}, DiGS regularizes the divergence of the gradient field to favor smooth functions~\citep{benshabat2022digs} and DiffCD uses a symmetric Chamfer-distance loss between the input point cloud and the surface \citep{diffcd}.

\paragraph{Visibility-aware reconstruction.}
Known sensor poses reveal the free-space segment preceding each measured surface
point. Classical methods exploit this evidence through volumetric reasoning:
\citet{faugeras1990representing} use visibility to identify free space in a constrained Delaunay
triangulation, \citet{labatut2009robust} infer
inside--outside labels on a Delaunay tetrahedralization using graph
cuts, while \citet{jancosek2014exploiting} refine to preserve weakly supported surfaces. More recent methods incorporate visibility into
learned reconstructions, through graph neural networks on Delaunay tetrahedralizations~\citep{sulzer2021dgnn}, or by augmenting point-based
models with ray-derived features and samples~\citep{sulzer2022visibility}. These methods have demonstrated strong robustness to
outliers and other acquisition defects
\citep{labatut2009robust,sulzer2021dgnn}.
Neural Poisson~\citep{dai2022neuralpoisson} and neural RGB-D mapping~\citep{ortiz2022isdf,zhu2022niceslam,wang2023coslam} also exploit sensor rays,
through losses on sampled points or volume rendering, whereas we regress the first zero of the implicit field directly.

\paragraph{Lipschitz neural networks.}
Lipschitz-constrained architectures were developed to control the worst-case
sensitivity of neural networks, for adversarial robustness, certification, and
Wasserstein objectives \citep{cisse2017parseval,arjovsky2017wgan}. Early
approaches constrain linear layers via spectral or orthogonality-based
regularization, but pointwise activations limit expressiveness. GroupSort
\citep{anil2019sorting} addresses this with norm-preserving sorting activations.
Scalable 1-Lipschitz architectures include convex potential layers
(CPL)~\citep{meunier2022cpl}, Sandwich layers~\citep{wang2023direct}, and
SDP-based Lipschitz layers (SLL)~\citep{araujo2023sll}. Closest to our work,
\citet{coiffier2024lipschitz} fit 1-Lipschitz neural distance fields to point
clouds with normals or to triangle soups.

\section{Method}
\label{sec:method}

\subsection{Problem statement}
\label{sec:problem}
Given a set $\mathcal D=\{(c_i,p_i)\}_{i=1}^{N}$ of \(N\) sensor--point pairs,
where \(c_i \in \mathbb{R}^3\) is the acquisition origin and \(p_i \in \mathbb{R}^3\) a possibly corrupted
measurement of an unknown true surface \(S^\star\), we seek parameters \(\theta\in\mathbb R^q\) of an
MLP \(f_\theta:\mathbb R^3\to\mathbb R\) whose zero
level set
$S_\theta=\{x\in\mathbb R^3:f_\theta(x)=0\}$
estimates \(S^\star\). We take \(f_\theta\) to be strictly positive outside the surface. 
\\
Let us first consider the case where noise is along the acquisition ray:
\begin{equation}
p_i=c_i+(t_i^\star+\eta_i)d_i,
\qquad 1\leq i\leq N,
\end{equation}
where $d_i=(p_i-c_i)/\norm{p_i-c_i}$ is the unit acquisition direction, \(p_i^\star=c_i+t_i^\star d_i\in S^\star\) the corresponding true surface point, \(\eta_i\in\mathbb R\) a scalar depth perturbation and \(t_i=\|p_i-c_i\|_2\) the measured depth. Surface fitting can then be viewed as robust regression of noisy scalar
first-hit depths, in the spirit of classical smoothing from noisy observations \citep{reinsch1967smoothing,craven1979smoothing}. We therefore consider the following regression problem
\begin{equation}
\label{eq:depth-objective}
\min_{\theta}\ \sum_{i=1}^N \rho\!\left(t_i(\theta)-t_i\right)
\quad\text{s.t.}\quad
t_i(\theta)=\inf\{t\geq0:f_\theta(c_i+t d_i)=0\},\quad i=1,\ldots,N,
\end{equation}
where $\rho$ is a robust penalty.
\\
\subsection{Incidence-Weighted First-Hit Regression}
More generally, under arbitrary 3D measurement noise, we have
$p_i=p_i^\star+\varepsilon_i$, with $p_i^\star\in S^\star$ and
$\varepsilon_i\in\mathbb R^3$. The depth residual $|t_i-t_i(\theta)|$ then depends on the incidence: near grazing angles, small surface-normal displacements produce large depth errors. Following incidence-dependent confidence weighting in range-image fusion and visibility-based reconstruction
\citep{curless1996volumetric,song2021vis2mesh}, we define
\begin{equation}
w_{\theta,i}
=
\frac{
\left|\nabla_x f_\theta(x_{\theta,i})^\top d_i\right|
}{
\|\nabla_x f_\theta(x_{\theta,i})\|_2
},
\label{eq:incidence-factor}
\end{equation}
where \(x_{\theta,i}=c_i+t_i(\theta)d_i\).
Since $p_i-x_{\theta,i}=[t_i-t_i(\theta)]d_i,$
\(w_{\theta,i}|t_i-t_i(\theta)|\) is the distance from \(p_i\) to the
tangent plane of \(S_\theta\) at \(x_{\theta,i}\). If \(f_\theta\) is
affine in a neighborhood of \(x_{\theta,i}\), with
\(\nabla_x f_\theta(x_{\theta,i})\neq0\), then for \(t_i\) sufficiently close to
\(t_i(\theta)\), we have
\begin{equation}
\operatorname{dist}(p_i,S_\theta)
=
w_{\theta,i}|t_i-t_i(\theta)|.
\label{eq:incidence-correction}
\end{equation}
A proof of this result is given in  Appendix~\ref{app:incidence-correction}. The objective we optimize is
\begin{equation}
\label{eq:depth-loss}
\mathcal L(\theta)
=
\mathbb E_{i\sim \mathcal{U}[\![1,N]\!]}
\left[
\rho_\delta\!\left(
w_{\theta,i}[t_i(\theta)-t_i]
\right)
\right],
\end{equation}
where $\rho_\delta$ is the Huber penalty with
$\delta>0$~\citep{huber1964robust}, known to limit the influence of gross outliers. When differentiating the loss, we keep $w_{\theta,i}$ fixed as it is a reweighting term, giving $\delta$ a consistent geometric meaning across incidence angles.

\subsection{Differentiable first-hit tracing}
Computing $t_i(\theta)$ requires locating the first zero of $f_\theta$
along the acquisition ray. For a $1$-Lipschitz MLP $f_\theta$ (Appendix~\ref{app:lipschitz-def}), a well-known algorithm to do so is sphere tracing, which is guaranteed to converge
at a linear rate at a non-grazing root,
$\nabla_x f_\theta(x_{\theta,i})^\top d_i \neq 0$~\citep{hart1996sphere}. This motivates our architectural choice in Section~\ref{sec:lip-model}. The sphere tracing algorithm is described below. 
Differentiable sphere tracing has been used to fit pretrained shape priors
to depth or images \citep{liu2020dist} and for multi-view reconstruction
\citep{yariv2020multiview}. We optionally refine converged sphere tracing intersections with two safeguarded Newton steps.
\begin{algorithm}[H]
\caption{Sphere tracing for $f_\theta$.}
\label{alg:sphere tracing}
\begin{algorithmic}[1]
\Require $i,\varepsilon,t_{\max},K$
\State $t\gets0,\; x\gets c_i,\; s\gets f_\theta(x)$
\For{$k=1,\ldots,K$}
    \State \textbf{if} $|s|\le\varepsilon$ \textbf{ then return} $x$
    \State $\tilde t\gets t+|s|$;\quad
           \textbf{if} $\tilde t>t_{\max}$ \textbf{ then return} $\emptyset$
    \State $t\gets\tilde t,\;x\gets c_i+t d_i,\;s\gets f_\theta(x)$
\EndFor
\State \Return $x$ if $|s|\le\varepsilon$, otherwise $\emptyset$
\end{algorithmic}
\end{algorithm}

\vspace{-5mm}

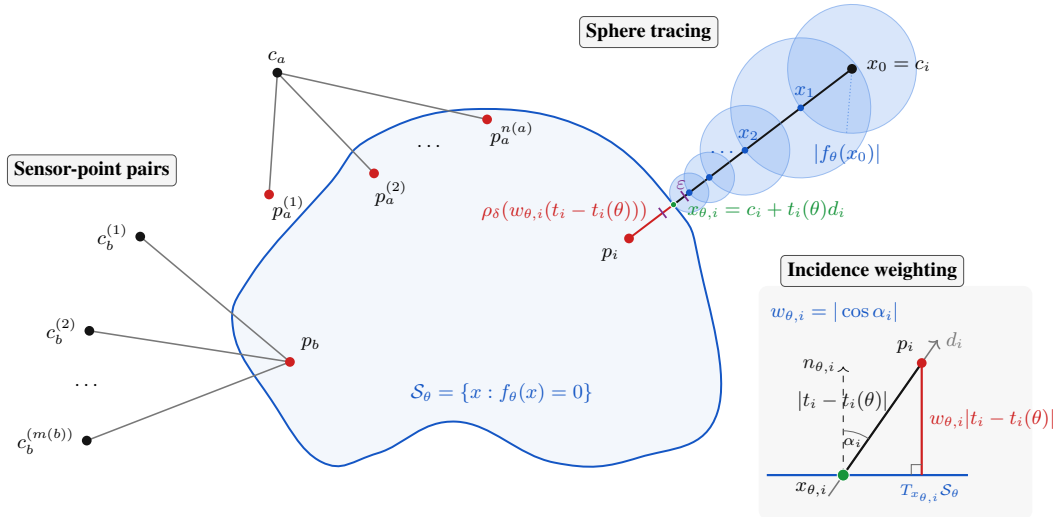
\begin{figure}[ht]
\centering

\resizebox{\linewidth}{!}{%
\begin{tikzpicture}[
  x=1cm,y=1cm,
  sensor/.style={circle,fill=sensordk,inner sep=0pt,minimum size=4.4pt},
  meas/.style={circle,fill=measred,inner sep=0pt,minimum size=4.4pt},
  hit/.style={circle,fill=hitgreen,draw=white,line width=.3pt,inner sep=0pt,minimum size=5.4pt},
  samp/.style={circle,fill=surfblue,inner sep=0pt,minimum size=3.0pt},
  lbl/.style={font=\small},
  slbl/.style={font=\footnotesize},
  ray/.style={sensordk,line width=.9pt},
  cray/.style={raygray,line width=.7pt},
]

\def\surfpts{
  (8.40,4.90)(9.75,4.98)(11.0,5.7)(10.72,8.02)(10.3,9.0)
  (8.98,10.18)(7.10,10.52)(5.52,10.08)(4.8,9.3)(3.6,7.9)
  (3.30,7.0)(3.9,5.9)(4.70,4.91)(5.6,5.05)(6.85,5.55)
}

\begin{scope}[on background layer]
  \fill[surfblue!5]
    plot[smooth cycle,tension=.72] coordinates{\surfpts};
\end{scope}

\draw[surfblue,line width=.9pt]
  plot[smooth cycle,tension=.72] coordinates{\surfpts};

\node[surfblue,lbl,anchor=west] at (6,6)
  {$\mathcal S_\theta=\{x:f_\theta(x)=0\}$};

\node[
  font=\footnotesize\bfseries,
  anchor=north west,
  draw=black!70,
  fill=black!4,
  rounded corners=2pt,
  inner sep=3pt
] at (-0.3,9.8)
{Sensor-point pairs};

\begin{scope}[shift={(-2,-0.5)}]

\coordinate (cS)  at (6,11.6);

\begin{scope}[rotate around={8:(cS)}] 

\coordinate (pS1) at (5.6,9.7);
\coordinate (pS2) at (7.3,9.8);
\coordinate (pS3) at (9.2,10.4);

\draw[cray] (cS)--(pS1);
\draw[cray] (cS)--(pS2);
\draw[cray] (cS)--(pS3);

\node[sensor] at (cS){};
\node[lbl,above=1pt] at (cS){$c_a$};

\foreach \p in {(pS1),(pS2),(pS3)}
  \node[meas] at \p{};

\node[slbl,anchor=west] at ($(pS1)+(-.1,-.2)$) {$p_a^{(1)}$};
\node[slbl,anchor=west] at ($(pS2)+(-.15,-.3)$) {$p_a^{(2)}$};
\node[slbl] at ($(pS2)!0.5!(pS3)+(0,0)$) {$\cdots$};
\node[slbl,anchor=west] at ($(pS3)+(-.1,-.22)$) {$p_a^{n(a)}$};

\end{scope}

\end{scope}

\begin{scope}[shift={(0,-1.5)}]

\coordinate (pM) at (4.2,8.0);

\begin{scope}[rotate around={-25:(pM)}]
  \coordinate (cM1) at (1.2,8.8);
  \coordinate (cM2) at (1.1,7.1);
  \coordinate (cM3) at (1.8,5.5);

  \draw[cray] (cM1)--(pM);
  \draw[cray] (cM2)--(pM);
  \draw[cray] (cM3)--(pM);

  \foreach \c in {(cM1),(cM2),(cM3)}
    \node[sensor] at \c{};

  \node[lbl,left=2pt] at (cM1){$c_b^{(1)}$};
  \node[lbl,left=2pt] at (cM2){$c_b^{(2)}$};
  \node[lbl] at ($(cM2)!0.5!(cM3)$) {$\cdots$};
  \node[lbl,left=2pt] at (cM3){$c_b^{(m(b))}$};
\end{scope}

\node[meas] at (pM){};
\node[lbl,anchor=south west]
  at ($(pM)+(.06,.06)$)
  {$p_b$};

\end{scope}

\node[
  font=\footnotesize\bfseries,
  anchor=north west,
  draw=black!70,
  fill=black!4,
  rounded corners=2pt,
  inner sep=3pt
] at (8.8,12)
{Sphere tracing};

\coordinate (xhit) at (10.30,9.00);

\begin{scope}[rotate around={10:(xhit)}]

\coordinate (cold) at (15.25,11.55);

\coordinate (ci)  at ($(cold)!0.360!(xhit)$);
\coordinate (xs1) at ($(cold)!0.544!(xhit)$);
\coordinate (xs2) at ($(cold)!0.744!(xhit)$);
\coordinate (xs3) at ($(cold)!0.872!(xhit)$);
\coordinate (xs4) at ($(cold)!0.944!(xhit)$);

\coordinate (pi) at ($(cold)!1.16!(xhit)$);

\draw[ray] (ci)--(xhit);

\begin{scope}[on background layer]
  \fill[spherefill,opacity=.28] (ci)  circle(1.024);
  \fill[spherefill,opacity=.28] (xs1) circle(1.1136);
  \fill[spherefill,opacity=.28] (xs2) circle(0.7128);
  \fill[spherefill,opacity=.30] (xs3) circle(0.4008);
  \fill[spherefill,opacity=.34] (xs4) circle(0.312);
\end{scope}

\foreach \c/\r in {
  (ci)/1.024,
  (xs1)/1.1136,
  (xs2)/0.7128,
  (xs3)/0.4008,
  (xs4)/0.312
}{
  \draw[spherefill!80!surfblue,opacity=.6,line width=.4pt]
    \c circle(\r);
}

\node[lbl,anchor=west]
  at ($(ci)+(.12,.03)$)
  {$x_0=c_i$};

\draw[spherefill!80!surfblue,line width=.5pt,densely dotted]
  (ci)--($(ci)+(-0.2624,-0.9896)$);

\node[surfblue,slbl,anchor=north]
  at ($(ci)+(-0.2624,-1.04)$)
  {$|f_\theta(x_0)|$};

\node[sensor] at (ci){};

\node[samp] at (xs1){};

\node[surfblue,slbl,anchor=south west]
  at ($(xs1)+(-0.21,0.08)$)
  {$x_1$};

\node[samp] at (xs2){};

\node[surfblue,slbl,anchor=south west]
  at ($(xs2)+(-0.22,0.08)$)
  {$x_2$};

\node[samp] at (xs3){};
\node[samp] at (xs4){};

\node[surfblue,slbl]
  at ($(xs2)!0.5!(xs3)+(-0.02,0.18)$)
  {$\cdots$};

\draw[measred,line width=.9pt]
  (xhit)--(pi);

\node[meas] at (pi){};

\node[lbl,anchor=north east]
  at ($(pi)+(-.04,-.03)$)
  {$p_i$};

\def\ux{-0.889}
\def\uy{-0.458}
\def\px{0.458}
\def\py{-0.889}

\coordinate (eA)
  at ($(xhit)+(-0.22*\ux,-0.22*\uy)$);

\coordinate (eB)
  at ($(xhit)+(0.22*\ux,0.22*\uy)$);

\draw[epsviolet,line width=.8pt]
  ($(eA)+(0.10*\px,0.10*\py)$)
  --
  ($(eA)-(0.10*\px,0.10*\py)$);

\draw[epsviolet,line width=.8pt]
  ($(eB)+(0.10*\px,0.10*\py)$)
  --
  ($(eB)-(0.10*\px,0.10*\py)$);

\node[epsviolet,slbl,anchor=west]
  at ($(xhit)+(-0.02,0.32)$)
  {$\varepsilon$};

\node[hit,minimum size=2.8pt] at (xhit){};

\node[hitgreen,slbl,anchor=west]
  at ($(xhit)+(0.08,-0.1)$)
  {$x_{\theta,i}=c_i+t_i(\theta)d_i$};

\node[measred,slbl,anchor=east]
  at ($(xhit)+(-0.3,-0.05)$)
  {$\rho_\delta\!\left(w_{\theta,i}(t_i-t_i(\theta))\right)$};

\end{scope}

\node[
  font=\footnotesize\bfseries,
  anchor=north west,
  draw=black!70,
  fill=black!4,
  rounded corners=2pt,
  inner sep=3pt
] at (12,8.2)
{Incidence weighting};

\begin{scope}[shift={(13,4.7)},scale=1.28]

\def\ain{35}
\def\Lin{1.7}

\coordinate (iO) at (0,0);
\coordinate (iP) at ({\Lin*sin(\ain)},{\Lin*cos(\ain)});
\coordinate (iF) at ({\Lin*sin(\ain)},0);
\coordinate (iR0) at ({-.32*sin(\ain)},{-.32*cos(\ain)});
\coordinate (iR1) at ({(\Lin+.35)*sin(\ain)},{(\Lin+.35)*cos(\ain)});

\begin{scope}[on background layer]
  \fill[black!3,rounded corners=4pt]
    (-1.05,-.55) rectangle (2.35,2.35);
\end{scope}

\node[surfblue,slbl,anchor=south west]
  at (-1,1.8)
  {$w_{\theta,i}=|\cos\alpha_i|$};

\draw[surfblue,line width=1pt]
  (-.95,0)--(1.55,0);

\node[surfblue,font=\scriptsize,anchor=north east]
  at (1.55,-.02)
  {$T_{x_{\theta,i}}\mathcal S_\theta$};

\draw[sensordk,dashed,->]
  (iO)--(0,1.35)
  node[left,slbl]{$n_{\theta,i}$};

\draw[raygray,->,line width=.7pt]
  (iR0)--(iR1)
  node[right,slbl]{$d_i$};

\draw[sensordk,line width=1pt]
  (iO)--(iP);

\node[sensordk,slbl,anchor=south east]
  at ($(iO)!0.5!(iP)+(0.17,0)$)
  {$|t_i-t_i(\theta)|$};

\draw[measred,line width=1pt]
  (iP)--(iF);

\node[measred,slbl,anchor=west]
  at ($(iP)!0.5!(iF)$)
  {$w_{\theta,i}|t_i-t_i(\theta)|$};

\draw[black!60]
  (iF)++(-.13,0)--++(0,.13)--++(.13,0);

\draw[black!60]
  (0,.55) arc (90:{90-\ain}:.55);

\node[font=\scriptsize]
  at (.12,.4)
  {$\alpha_i$};

\node[hit] at (iO){};

\node[slbl,anchor=north east]
  at ($(iO)+(-.1,0)$)
  {$x_{\theta,i}$};

\node[meas] at (iP){};

\node[slbl,anchor=south west]
  at ($(iP)+(-.4,.03)$)
  {$p_i$};

\end{scope}

\end{tikzpicture}%
}
 \vspace{-0.5mm}
\caption{Method overview. Sensor-point pairs define acquisition rays;
sphere tracing locates the predicted first surface intersection, and incidence
weighting locally converts the depth residual into a surface-distance residual.}

\end{figure}

In practice, at each optimization step, we estimate Eq.~\eqref{eq:depth-loss} from batches and retain only rays for which sphere tracing returns a finite,
non-grazing intersection satisfying $w_{\theta,i}>\tau_g$, with $\tau_g>0$. We do not differentiate through this filtering operation. For such hits, assuming that $f_\theta$ is differentiable with respect to
$x$ and $\theta$ in a neighborhood of $x_{\theta,i}$, implicit differentiation of
$f_\theta(x_{\theta,i})=0$ gives~\citep{yariv2020multiview}
\begin{equation}
\nabla_\theta t_i(\theta)
=
-\frac{
1
}{
\nabla_x f_\theta(x_{\theta,i})^\top d_i
}\nabla_\theta f_\theta(x_{\theta,i}).
\label{eq:implicit-hit-gradient}
\end{equation}
Thus, we avoid backpropagating through all sphere tracing iterations, which would be expensive.

\subsection{1-Lipschitz Model}
\label{sec:lip-model}
To ensure convergence of sphere tracing, we parameterize the field
\(f_\theta:\mathbb{R}^3\to\mathbb{R}\) as a composition of
\(1\)-Lipschitz functions. Let \(D\) denote the network's hidden dimension,
which we take to be even. The model is defined as
\begin{equation}
f_\theta(x)
=
v^\top \Phi_\theta(Ax)+\beta,
\end{equation}
where
$A\in\mathbb{R}^{D\times 3}$ satisfies
$A^\top A=I_3$,
\(\Phi_\theta:\mathbb{R}^D\to\mathbb{R}^D\) is a \(1\)-Lipschitz neural
network,
$v
=
\widetilde{v} /
\|\widetilde{v}\|_2
\in\mathbb{R}^D$
is the final unit-norm weight, and \(\beta\in\mathbb{R}\) is a scalar bias. As a composition of such functions, \(f_\theta:\mathbb{R}^3\to\mathbb{R}\) is \(1\)-Lipschitz. Our first-hit formulation is agnostic to the 1-Lipschitz
architecture used. In our experiments, $\Phi_\theta$ interleaves Convex Potential Layers (CPL)
and MaxMin activations~\citep{anil2019sorting,meunier2022cpl}.
Architectural details are given in
Appendix~\ref{sec:implementation}.

\subsection{Practical variants}
\noindent\textbf{Rescaling.} The architectural constraint guarantees
$\operatorname{Lip}(f_\theta)\leq 1$, but can leave substantial slack, yielding small field magnitudes and conservative sphere tracing steps. Eikonal regularization~\citep{gropp2020igr} and the
hinge--Kantorovich--Rubinstein (hKR) loss~\citep{coiffier2024lipschitz}
encourage unit-norm gradients; however, both
degrade reconstruction accuracy in our experiments  (Table~\ref{tab:full-ablations}). For the convex reconstruction domain $\Omega\subset\mathbb{R}^3$ with nonempty interior, we define
\begin{equation}
G_\theta(\Omega)
=
\operatorname{Lip}(f_\theta|_\Omega)
=
\sup_{x\in\Omega:\,\nabla f_\theta(x)\ \mathrm{exists}}
\|\nabla_x f_\theta(x)\|_2 .
\end{equation}
Rescaling by \(G_\theta(\Omega)>0\) preserves the zero set and yields
\(\operatorname{Lip}(f_\theta/G_\theta(\Omega)|_\Omega)=1\).
However, computing \(G_\theta(\Omega)\)
exactly is NP-hard in general \citep{virmaux2018lipschitz}. We therefore
estimate it every 50 steps from 32,000 probes, sampled half uniformly in the
tracing ball and half near observed returns, and rescale
\(f_\theta\leftarrow f_\theta/\widehat G_\theta(\Omega)\).
Since \(\widehat G_\theta(\Omega)\le G_\theta(\Omega)\), this variant is no longer formally
certified. Empirically, sampled slopes remain below one (Figure~\ref{fig:grad-slope-curves}).

\noindent\textbf{Positional encoding.} To increase fitting capacity, we additionally consider Fourier positional
encoding, which mitigates the spectral bias of coordinate MLPs and improves
their ability to represent high-frequency
detail~\citep{mildenhall2020nerf,tancik2020fourier}. For \(L\geq1\),
\begin{equation}
\gamma_L(x)
=
\left[
x,
\left(
\sin(2^k x_j),\cos(2^k x_j)
\right)_{\substack{j=1,\ldots,3\\k=0,\ldots,L-1}}
\right].
\label{eq:positional-encoding}
\end{equation}
The implicit field becomes
$h_\theta :x\mapsto v^\top\Phi_\theta(A\gamma_L(x))+\beta$ with \(A\in\mathbb{R}^{D\times(6L+3)}\) and $\Lip(h_\theta)\leq \Lip(\gamma_L)
=
(1+(4^L-1) /3)^{1/2}$ for the \(\ell_2\) norm.
This variant is therefore no longer formally 1-Lipschitz. Dividing the sphere tracing steps by
$\Lip(\gamma_L)$ yields highly conservative steps. We therefore empirically rescale $h_\theta$ as described above. 

\section{Experiments}
\label{sec:experiments}

\subsection{Architecture and optimization.} Unless stated otherwise, we fit an 8-layer, width-256 CPL network with
MaxMin activations using Adam with cosine learning-rate decay and an initial learning rate of $10^{-4}$ for $L=0$ and $10^{-2}$ for $L>0$.
Ours ($L=k$) denotes the empirically rescaled variant with PE
level $k\geq 0$, where $L=0$ corresponds to no PE.
We use $\delta=0.05$ by default and $\delta=0.005$ when corruption
is detected from point-spacing statistics, both relative to the scene radius (Appendix~\ref{app:huber-delta}).
Optimization runs for $5\text{k}$ steps on ShapeNet and $25\text{k}$ elsewhere.
Sphere tracing uses $K=192$, $\varepsilon=5\times10^{-4}r_{\mathrm{init}}$ and  $\tau_g=0.1$, where $r_{\mathrm{init}}$ is the radius of the enclosing sphere.
The field is initialized as an approximate SDF of an enclosing sphere. Further
implementation details are given in Appendix~\ref{sec:implementation}.

\subsection{Synthetic datasets and evaluation}

\noindent\textbf{Surface Reconstruction Benchmark (SRB).} We use the five SRB shapes and synthetic range scans of
\citet{sulzer2025survey}, retaining each point's sensor origin (Figure~\ref{fig:srb-acquisition}). These scans were generated with a modified implementation of \citet{berger2013benchmark}'s optical laser-triangulation simulator.

\noindent\textbf{ShapeNet.} For synthetic experiments on
ShapeNet~\citep{chang2015shapenet}, we use the 13-category split of \citet{choy2016r2n2} adopted by \citet{sulzer2025survey}, with five shapes per
category. Following \citet{sulzer2025survey}, point clouds are generated by
casting rays from virtual sensors and retaining first surface intersections
(Figure~\ref{fig:shapenet-acquisition}).

\noindent\textbf{Baselines.} We compare against neural implicit methods SAL \citep{atzmon2020sal},
Neural-Pull \citep{ma2021neuralpull}, DiGS \citep{benshabat2022digs},
DiffCD \citep{diffcd}, and NDF \citep{coiffier2024lipschitz}, and classical methods
SPSR \citep{kazhdan2013screened} and visibility-aware Delaunay reconstruction
\citep{labatut2009robust}. We adapt NDF to visibility-derived labels and call it NDF-vis (Appendix~\ref{app:ndf-vis}). Neural Poisson \citep{dai2022neuralpoisson}
is omitted because no public implementation is available. All methods use the
same input points; SPSR additionally uses estimated oriented normals, while Labatut,
NDF-vis, and Ours use sensor origins. 

\begin{wrapfigure}[10]{r}{0.4\linewidth}
    \vspace{-1.4\baselineskip}
    \centering
    \includegraphics[width=\linewidth]{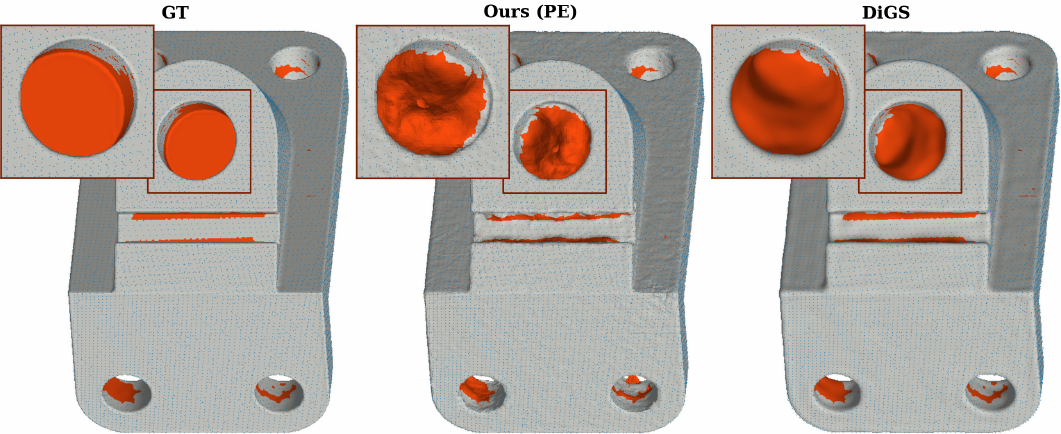}
    \caption{Observation-supported SRB evaluation. Orange denotes surface regions
excluded by the $2h_0$ support mask.}
    \label{fig:obs-mask}
\end{wrapfigure}

\noindent\textbf{Metrics.} Evaluation uses symmetric Chamfer distance ($\mathrm{CD}$),
normal consistency ($\mathrm{NC}$), and volumetric IoU. Following DiGS~\citep{benshabat2022digs}, SRB uses unsquared \(\mathrm{CD}_1\) in benchmark coordinates and ShapeNet squared \(\mathrm{CD}_2\) in normalized coordinates; NC follows \citet{sulzer2025survey}, and IoU follows the DiGS ShapeNet protocol.
However, because some surfaces contain regions unsupported by any acquisition,
these global metrics conflate reconstruction of observed geometry with extrapolation
in unobserved regions. We therefore also evaluate the observation-supported metrics
$\mathrm{CD}^{\mathrm{obs}}$ and $\mathrm{NC}^{\mathrm{obs}}$, retaining predicted
and reference surface samples within $2h_0$ of the corrupted and clean
acquisitions, respectively, where $h_0$ is the median nearest-neighbor spacing
\citep{marchand2021evaluating} (Appendix~\ref{sec:metrics}). 

\subsection{Robustness to measurement errors} 
\setlength{\columnsep}{8pt}
\setlength{\intextsep}{3pt}
\noindent\textbf{Scanner noise (SRB).}
We parameterize the scanner-noise level by $\lambda\in[0,1]$, with
$\lambda=0$ corresponding to the unperturbed scan and $\lambda=1$ to the
original noisy configuration of \citet{sulzer2025survey}. Intermediate
values linearly scale the simulator noise magnitude. $\lambda=1$ induces range errors with a standard deviation of about $1\%$ of the bounding-box diagonal. The $\lambda=1$ setting uses \citet{berger2013benchmark}'s simulator with noise
magnitude 0.5, near the upper end of its typical range $0$--$0.6$.

\begin{table}[!ht]
\centering
\caption{SRB results under scanner noise. Mean $\pm$ std. over three seeds for
stochastic methods; best/second-best neural implicit results are
bold/underlined. Per-shape results are in Appendix~\ref{app:extended-results}.}
\label{tab:srb-noise}
\vspace{-6pt}
\setlength{\tabcolsep}{2.2pt}
\renewcommand{\arraystretch}{1.03}

\resizebox{\linewidth}{!}{%
\begin{tabular}{lcccc@{\hspace{3mm}}cccc}
\toprule
&
\multicolumn{4}{c}{\textbf{Clean} ($\lambda=0$)}
&
\multicolumn{4}{c}{\textbf{Noise} ($\lambda=1$)}
\\
\cmidrule(lr){2-5}
\cmidrule(lr){6-9}
\textbf{Method}
& $\mathrm{CD}_1\downarrow$
& $\mathrm{NC}\uparrow$
& $\mathrm{CD}_1^{\mathrm{obs}}\downarrow$
& $\mathrm{NC}^{\mathrm{obs}}\uparrow$
& $\mathrm{CD}_1\downarrow$
& $\mathrm{NC}\uparrow$
& $\mathrm{CD}_1^{\mathrm{obs}}\downarrow$
& $\mathrm{NC}^{\mathrm{obs}}\uparrow$
\\
\midrule
\multicolumn{9}{l}{\textit{Neural implicit}} \\

Ours ($L=0$)
& 0.163 {\scriptsize $\pm$ 0.005}
& 0.958 {\scriptsize $\pm$ 0.001}
& 0.101 {\scriptsize $\pm$ 0.003}
& 0.970 {\scriptsize $\pm$ 0.001}
& \textbf{0.310} {\scriptsize $\pm$ 0.016}
& \textbf{0.906} {\scriptsize $\pm$ 0.006}
& \textbf{0.226} {\scriptsize $\pm$ 0.005}
& \textbf{0.921} {\scriptsize $\pm$ 0.004}
\\

Ours ($L=2$)
& \underline{0.127} {\scriptsize $\pm$ 0.002}
& 0.970 {\scriptsize $\pm$ 0.001}
& 0.073 {\scriptsize $\pm$ 0.001}
& 0.982 {\scriptsize $\pm$ 0.0002}
& \underline{0.381} {\scriptsize $\pm$ 0.002}
& 0.781 {\scriptsize $\pm$ 0.002}
& \underline{0.296} {\scriptsize $\pm$ 0.004}
& 0.794 {\scriptsize $\pm$ 0.002}
\\

Ours ($L=6$)
& 0.133 {\scriptsize $\pm$ 0.003}
& \underline{0.974} {\scriptsize $\pm$ 0.002}
& \textbf{0.064} {\scriptsize $\pm$ 0.002}
& \underline{0.987} {\scriptsize $\pm$ 0.003}
& 0.542 {\scriptsize $\pm$ 0.009}
& 0.543 {\scriptsize $\pm$ 0.008}
& 0.409 {\scriptsize $\pm$ 0.007}
& 0.544 {\scriptsize $\pm$ 0.006}
\\

NDF-vis
& 0.613 {\scriptsize $\pm$ 0.063}
& 0.893 {\scriptsize $\pm$ 0.003}
& 0.423 {\scriptsize $\pm$ 0.037}
& 0.920 {\scriptsize $\pm$ 0.003}
& 1.098 {\scriptsize $\pm$ 0.194}
& \underline{0.871} {\scriptsize $\pm$ 0.003}
& 0.834 {\scriptsize $\pm$ 0.149}
& \underline{0.901} {\scriptsize $\pm$ 0.0004}
\\

Neural-Pull
& 0.208 {\scriptsize $\pm$ 0.051}
& 0.945 {\scriptsize $\pm$ 0.004}
& 0.126 {\scriptsize $\pm$ 0.031}
& 0.959 {\scriptsize $\pm$ 0.005}
& 0.450 {\scriptsize $\pm$ 0.003}
& 0.717 {\scriptsize $\pm$ 0.003}
& 0.394 {\scriptsize $\pm$ 0.001}
& 0.718 {\scriptsize $\pm$ 0.003}
\\

DiffCD
& 0.128 {\scriptsize $\pm$ 0.001}
& 0.973 {\scriptsize $\pm$ 0.001}
& \underline{0.070} {\scriptsize $\pm$ 0.0003}
& 0.986 {\scriptsize $\pm$ 0.0003}
& 0.415 {\scriptsize $\pm$ 0.006}
& 0.781 {\scriptsize $\pm$ 0.002}
& 0.354 {\scriptsize $\pm$ 0.003}
& 0.789 {\scriptsize $\pm$ 0.001}
\\

SAL
& 0.731 {\scriptsize $\pm$ 0.148}
& 0.885 {\scriptsize $\pm$ 0.013}
& 0.567 {\scriptsize $\pm$ 0.122}
& 0.914 {\scriptsize $\pm$ 0.010}
& 1.125 {\scriptsize $\pm$ 0.322}
& 0.845 {\scriptsize $\pm$ 0.003}
& 0.970 {\scriptsize $\pm$ 0.319}
& 0.867 {\scriptsize $\pm$ 0.001}
\\

DiGS
& \textbf{0.108} {\scriptsize $\pm$ 0.004}
& \textbf{0.984} {\scriptsize $\pm$ 0.001}
& 0.071 {\scriptsize $\pm$ 0.001}
& \textbf{0.993} {\scriptsize $\pm$ 0.0002}
& 0.610 {\scriptsize $\pm$ 0.006}
& 0.578 {\scriptsize $\pm$ 0.004}
& 0.488 {\scriptsize $\pm$ 0.001}
& 0.575 {\scriptsize $\pm$ 0.004}
\\

\midrule
\multicolumn{9}{l}{\textit{Classical}} \\

SPSR
& 0.204 & 0.956 & 0.116 & 0.970
& 0.358 & 0.881 & 0.271 & 0.890
\\

Labatut
& 0.137 & 0.953 & 0.074 & 0.968
& 0.410 & 0.639 & 0.363 & 0.639
\\

\bottomrule
\end{tabular}%
}
\end{table}

On clean data, increasing $L$ improves fidelity to the observations. Ours
($L=6$) achieves lower observation-supported $\mathrm{CD}_1$ than DiGS
($0.064$ vs.\ $0.071$) and a smaller input-to-surface distance
(\(0.011\) vs.\ \(0.032\); Table~\ref{tab:data-fidelity}), despite a higher $\mathrm{CD}_1$
($0.133$ vs.\ $0.108$). The discrepancy is concentrated in weakly observed
regions, notably Daratech (Table~\ref{tab:srb-results-noise}), where stronger
geometric priors can improve global metrics by extrapolating beyond the
measurements.

Under scanner noise, Ours ($L=0$) outperforms
DiGS both globally ($0.310$ vs.\ $0.610$ $\mathrm{CD}_1$) and on
observation-supported geometry ($0.226$ vs.\ $0.488$ $\mathrm{CD}_1^\mathrm{obs}$), while larger $L$
degrades more rapidly. Because this degradation also occurs in observed regions,
larger $L$ increasingly fits corrupted measurements rather than merely
extrapolating poorly in weakly constrained regions. Figure~\ref{fig:srb-noise-qual} shows the corresponding qualitative
behavior.
\begin{figure}[!htbp]
    \centering
    \includegraphics[width=\linewidth]{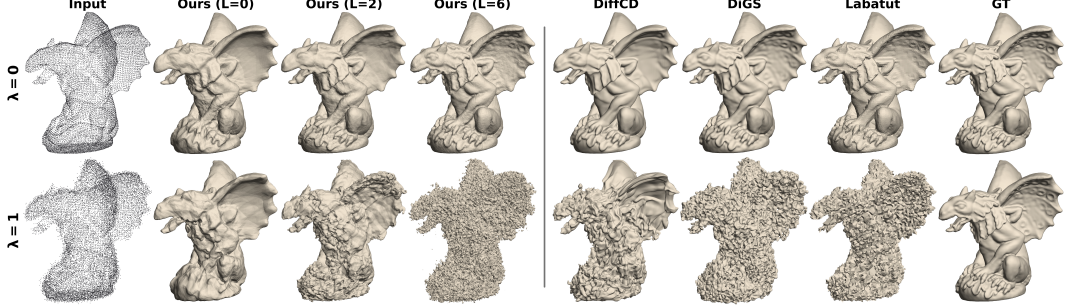}
    \captionsetup{skip=2pt}
    \caption{Qualitative Gargoyle reconstructions under scanner noise:
    $\lambda=0$ (top) and $\lambda=1$ (bottom).}
    \label{fig:srb-noise-qual}
\end{figure}
\vspace{2pt}
\begin{figure}[!htbp]
    \centering
\includegraphics[width=\textwidth]{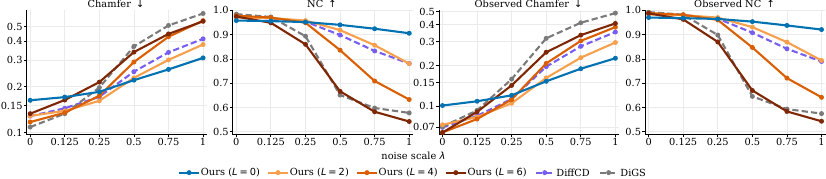}
\captionsetup{skip=2pt}
    \caption{SRB performance across scanner-noise levels.
Larger $L$ improves clean fidelity but reduces robustness. Ours ($L=2$) provides a good compromise.}
    \label{fig:srb_noise_tradeoff}
\end{figure}

\noindent\textbf{Gaussian noise (ShapeNet).} To test robustness beyond the SRB corruption model, we evaluate on ShapeNet
using 3K inputs. As in~\citet{peng2021sap}, we consider
isotropic Gaussian noise, and additionally ray-aligned range noise as a
first-order model of depth uncertainty in direct time-of-flight
sensors~\citep{koerner2021models}:
\begin{equation}
\left\{
\begin{aligned}
\tilde p_i &= p_i+\varepsilon_i,
& \varepsilon_i &\sim \mathcal N(0,\sigma_{\mathrm{iso}}^2 I_3),
&& \text{(isotropic)},\\
\tilde p_i &= c_i+(t_i+\eta_i)d_i,
& \eta_i &\sim \mathcal N(0,\sigma_{\mathrm{ray}}^2),
&& \text{(ray-aligned)}.
\end{aligned}
\right.
\end{equation}

In both settings, we set $\sigma_{\mathrm{iso}}^2=\sigma_{\mathrm{ray}}^2=2.5 \times 10^{-5}$ in the unit system of $\mathrm{CD}_2$, following
prior work~\citep{erler2020points2surf}. Figure \ref{fig:shapenet-gaussian} shows that the robustness gains persist under Gaussian noise.

\begin{table}[H]
\centering
\caption{ShapeNet reconstruction under Gaussian noise with 3K input points.
Mean over 65 test shapes, except SAL, which is averaged over successfully scored shapes. Best/second-best neural implicit results are shown in bold/underlined.}
\label{tab:shapenet-gaussian}
\vspace{-6pt}

{%
\setlength{\tabcolsep}{3pt}
\renewcommand{\arraystretch}{1.05}
\resizebox{\linewidth}{!}{%
\begin{tabular}{clcccccccc!{\vrule width 0.4pt}cc}
\toprule
& \textbf{Metric}
& Ours ($L=0$)
& Ours ($L=2$)
& Ours ($L=6$)
& NDF-vis
& DiffCD
& N.-Pull
& SAL
& DiGS
& SPSR
& Labat. \\
\midrule

\multirow{3}{*}{\textbf{Iso}}
& $\mathrm{CD}_2\,(10^{-4})\downarrow$
& 2.58 & \textbf{2.26} & \underline{2.47} & 5.51 & 2.53
& 10.1 & 112 & 5.89 & 46.1 & 3.04 \\
& $\mathrm{NC}\uparrow$
& \underline{0.91} & \textbf{0.91} & 0.90 & 0.89 & 0.84
& 0.76 & 0.85 & 0.68 & 0.84 & 0.79 \\
& $\mathrm{IoU}\uparrow$
& \underline{0.84} & \textbf{0.85} & 0.82 & 0.72 & 0.75
& 0.64 & 0.63 & 0.46 & 0.63 & 0.79 \\

\midrule

\multirow{3}{*}{\textbf{Ray}}
& $\mathrm{CD}_2\,(10^{-4})\downarrow$
& 2.23 & \textbf{2.12} & \underline{2.18} & 6.66 & 2.45
& 14.8 & 150 & 5.83 & 45.0 & 3.02 \\
& $\mathrm{NC}\uparrow$
& \underline{0.92} & \textbf{0.92} & 0.91 & 0.89 & 0.88
& 0.79 & 0.85 & 0.71 & 0.84 & 0.81 \\
& $\mathrm{IoU}\uparrow$
& \underline{0.86} & \textbf{0.86} & 0.85 & 0.70 & 0.78
& 0.63 & 0.62 & 0.58 & 0.64 & 0.79 \\

\bottomrule
\end{tabular}%
}
}
\end{table}

\begin{figure}[!ht]
\centering
\includegraphics[width=\linewidth]{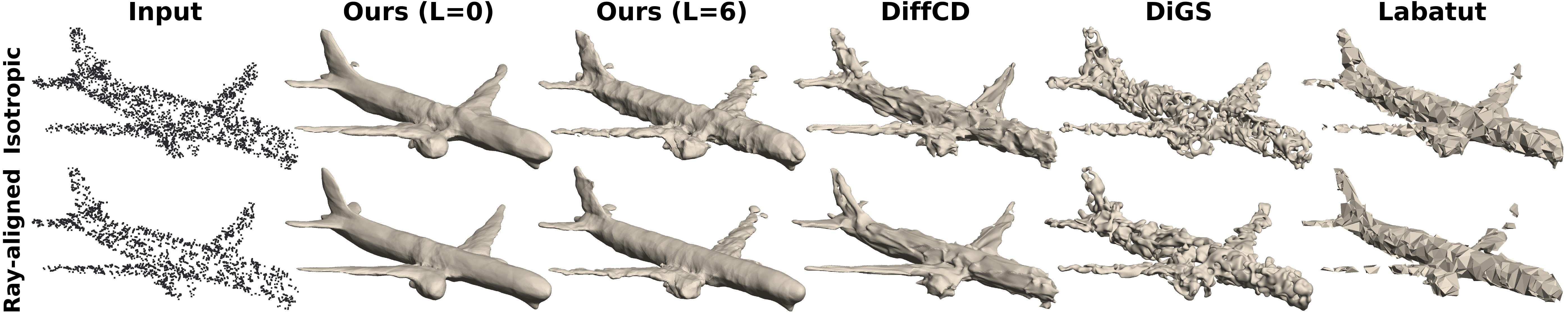}
\caption{Qualitative ShapeNet airplane results under
isotropic and ray-aligned Gaussian noise.}
\label{fig:shapenet-gaussian}
\end{figure}

\noindent\textbf{Outliers (SRB).} Following~\citet{sulzer2025survey}, we augment each clean SRB scan
with $N_{\mathrm{out}}=\lfloor\alpha N_{\mathrm{in}}\rfloor$ outliers,
sampled uniformly within the clean point cloud's bounding box
and uniformly assigned to a sensor. \(\alpha=0\) corresponds to the full clean scan, while \(\lambda=0\) uses the same set of rays as \(\lambda=1\). The two clean settings are therefore not the same, explaining the difference in the results. 

\hyphenpenalty=3000
\exhyphenpenalty=3000

\begin{table}[H]
\caption{SRB results under outliers. Mean $\pm$ std. over three seeds for stochastic
methods; best/second-best neural implicit results are bold/underlined.
Per-shape results are in Appendix~\ref{app:extended-results}.}
\label{tab:srb-outliers}
\centering
\vspace{-6pt}
\setlength{\tabcolsep}{2.2pt}
\renewcommand{\arraystretch}{1.03}

\resizebox{\linewidth}{!}{%
\begin{tabular}{lcccc@{\hspace{3mm}}cccc}
\toprule
&
\multicolumn{4}{c}{\textbf{Clean} ($\alpha=0$)}
&
\multicolumn{4}{c}{\textbf{Outliers} ($\alpha=0.2$)}
\\
\cmidrule(lr){2-5}
\cmidrule(lr){6-9}

\textbf{Method}
& $\mathrm{CD}_1\downarrow$
& $\mathrm{NC}\uparrow$
& $\mathrm{CD}_1^{\mathrm{obs}}\downarrow$
& $\mathrm{NC}^{\mathrm{obs}}\uparrow$
& $\mathrm{CD}_1\downarrow$
& $\mathrm{NC}\uparrow$
& $\mathrm{CD}_1^{\mathrm{obs}}\downarrow$
& $\mathrm{NC}^{\mathrm{obs}}\uparrow$
\\
\midrule

\multicolumn{9}{l}{\textit{Neural implicit}} \\

Ours ($L=0$)
& 0.155 {\scriptsize $\pm$ 0.003}
& 0.961 {\scriptsize $\pm$ 0.002}
& 0.098 {\scriptsize $\pm$ 0.002}
& 0.973 {\scriptsize $\pm$ 0.001}
& \underline{0.173} {\scriptsize $\pm$ 0.013}
& \underline{0.954} {\scriptsize $\pm$ 0.003}
& \underline{0.107} {\scriptsize $\pm$ 0.005}
& \underline{0.968} {\scriptsize $\pm$ 0.002}
\\

Ours ($L=2$)
& \underline{0.119} {\scriptsize $\pm$ 0.003}
& 0.974 {\scriptsize $\pm$ 0.001}
& 0.070 {\scriptsize $\pm$ 0.002}
& 0.985 {\scriptsize $\pm$ 0.001}
& \textbf{0.122} {\scriptsize $\pm$ 0.004}
& \textbf{0.968} {\scriptsize $\pm$ 0.001}
& \textbf{0.080} {\scriptsize $\pm$ 0.002}
& \textbf{0.981} {\scriptsize $\pm$ 0.001}
\\

Ours ($L=6$)
& 0.127 {\scriptsize $\pm$ 0.004}
& \underline{0.977} {\scriptsize $\pm$ 0.0009}
& \textbf{0.061} {\scriptsize $\pm$ 0.002}
& \underline{0.989} {\scriptsize $\pm$ 0.0006}
& 0.414 {\scriptsize $\pm$ 0.018}
& 0.818 {\scriptsize $\pm$ 0.0018}
& 0.151 {\scriptsize $\pm$ 0.007}
& 0.870 {\scriptsize $\pm$ 0.0003}
\\

NDF-vis
& 0.697 {\scriptsize $\pm$ 0.085}
& 0.885 {\scriptsize $\pm$ 0.003}
& 0.463 {\scriptsize $\pm$ 0.054}
& 0.916 {\scriptsize $\pm$ 0.003}
& 5.492 {\scriptsize $\pm$ 0.280}
& 0.723 {\scriptsize $\pm$ 0.010}
& 2.144 {\scriptsize $\pm$ 0.268}
& 0.832 {\scriptsize $\pm$ 0.009}
\\

DiffCD
& 0.123 {\scriptsize $\pm$ 0.002}
& 0.975 {\scriptsize $\pm$ 0.002}
& \underline{0.068} {\scriptsize $\pm$ 0.002}
& 0.988 {\scriptsize $\pm$ 0.001}
& 3.354 {\scriptsize $\pm$ 0.028}
& 0.806 {\scriptsize $\pm$ 0.001}
& 1.614 {\scriptsize $\pm$ 0.007}
& 0.903 {\scriptsize $\pm$ 0.001}
\\

Neural-Pull
& 0.169 {\scriptsize $\pm$ 0.025}
& 0.951 {\scriptsize $\pm$ 0.003}
& 0.108 {\scriptsize $\pm$ 0.017}
& 0.964 {\scriptsize $\pm$ 0.005}
& 3.42 {\scriptsize $\pm$ 0.15}
& 0.756 {\scriptsize $\pm$ 0.015}
& 1.096 {\scriptsize $\pm$ 0.128}
& 0.855 {\scriptsize $\pm$ 0.014}
\\

SAL
& 0.704 {\scriptsize $\pm$ 0.068}
& 0.889 {\scriptsize $\pm$ 0.009}
& 0.512 {\scriptsize $\pm$ 0.054}
& 0.919 {\scriptsize $\pm$ 0.007}
& 5.731 {\scriptsize $\pm$ 2.134}
& 0.652 {\scriptsize $\pm$ 0.031}
& 3.159 {\scriptsize $\pm$ 2.164}
& 0.756 {\scriptsize $\pm$ 0.034}
\\

DiGS
& \textbf{0.106} {\scriptsize $\pm$ 0.005}
& \textbf{0.985} {\scriptsize $\pm$ 0.001}
& 0.069 {\scriptsize $\pm$ 0.002}
& \textbf{0.994} {\scriptsize $\pm$ 0.0003}
& 2.927 {\scriptsize $\pm$ 0.063}
& 0.814 {\scriptsize $\pm$ 0.002}
& 1.650 {\scriptsize $\pm$ 0.017}
& 0.897 {\scriptsize $\pm$ 0.001}
\\

\midrule
\multicolumn{9}{l}{\textit{Classical}} \\

SPSR
& 0.187 & 0.960 & 0.104 & 0.974
& 4.917 & 0.855 & 0.322 & 0.970
\\

Labatut
& 0.128 & 0.960 & 0.068 & 0.975
& 0.144 & 0.946 & 0.079 & 0.963
\\

\bottomrule
\end{tabular}%
}
\end{table}

Neural implicit baselines deteriorate sharply with increasing outliers, whereas our variants remain stable. In particular, Ours (\(L=2\)) retains the same performance at \(\alpha=0.2\) as on clean data and achieves lower \(\mathrm{CD}_1\) than Labatut at every tested outlier level (Figure~\ref{fig:srb-outlier-sweep}). Ours (\(L=6\)) also remains robust, but degrades more strongly as \(\alpha\) increases, illustrating the fidelity--robustness trade-off induced by fitting capacity. This behavior follows from the combination of first-hit visibility and the Huber loss. Across our PE and no-PE variants, roughly \(60\%\) of outlier rays fall in Huber's linear regime and about \(40\%\) do not yield a valid traced hit, while over \(91\%\) of inlier rays remain in the quadratic regime and only $0.8$--$2.4\%$ are dropped (Table~\ref{tab:outlier_ray_stats}). First-hit supervision makes many inconsistent outlier measurements difficult to fit without introducing premature intersections along other rays, so these outliers cease to yield valid hits. For those that remain, the Huber loss limits their influence.

\subsection{Real data experiments}
\noindent\textbf{ARKitScenes (iPad LiDAR).} We evaluate on an ARKitScenes sofa scene, using raw 256 × 192 depth maps from the consumer-grade LiDAR of a 2020 iPad Pro with the corresponding sensor poses. Since the FARO reference is provided as view-aligned depth rather than a surface mesh, Chamfer evaluation would discard the original ray correspondences. We instead render each reconstruction from the FARO viewpoints and compare first-hit depths to the reference depth maps, reporting mean absolute depth error (MAE), hit rate, and the fraction of rays with error below 2 cm (I@2 cm) (Appendix~\ref{sec:arkitscenes-metrics}).

\begin{table}[!ht]
\centering
\caption{ARKitScenes sofa reconstruction from 2.34M iPad LiDAR points.
Evaluation over 36 FARO frames; mean $\pm$ standard deviation across frames.
Best/second-best results are bold/underlined.}
\label{tab:arkit-sofa}
\vspace{-6pt}
\scriptsize
\setlength{\tabcolsep}{2.5pt}
\renewcommand{\arraystretch}{1.05}

\resizebox{\linewidth}{!}{%
\begin{tabular}{lccccccccc}
\toprule
&
\multicolumn{7}{c}{\textit{Neural implicit}} &
\multicolumn{2}{c}{\textit{Classical}} \\
\cmidrule(lr){2-8}\cmidrule(lr){9-10}
&
Ours ($L=0$) &
Ours ($L=2$) &
NDF-vis &
DiffCD &
Neural-Pull &
SAL &
DiGS &
SPSR &
Labatut \\
\midrule

MAE (cm) $\downarrow$
& \underline{2.24 {\scriptsize $\pm$ 0.17}}
& \textbf{2.07 {\scriptsize $\pm$ 0.17}}
& 11.30 {\scriptsize $\pm$ 0.52}
& 4.84 {\scriptsize $\pm$ 0.32}
& 5.56 {\scriptsize $\pm$ 0.30}
& 8.34 {\scriptsize $\pm$ 0.39}
& 5.66 {\scriptsize $\pm$ 0.29}
& 2.84 {\scriptsize $\pm$ 0.25}
& 4.07 {\scriptsize $\pm$ 0.28}
\\

Hit (\%) $\uparrow$
& \textbf{99.7 {\scriptsize $\pm$ 0.13}}
& 99.4 {\scriptsize $\pm$ 0.24}
& 83.4 {\scriptsize $\pm$ 3.62}
& \textbf{99.7 {\scriptsize $\pm$ 0.15}}
& 90.3 {\scriptsize $\pm$ 2.72}
& 96.3 {\scriptsize $\pm$ 0.58}
& \underline{99.5 {\scriptsize $\pm$ 0.09}}
& \textbf{99.7 {\scriptsize $\pm$ 0.08}}
& 98.8 {\scriptsize $\pm$ 0.40}
\\

$I@2\,\mathrm{cm}$ (\%) $\uparrow$
& \underline{61.9 {\scriptsize $\pm$ 3.7}}
& \textbf{66.0 {\scriptsize $\pm$ 4.0}}
& 1.3 {\scriptsize $\pm$ 0.5}
& 33.1 {\scriptsize $\pm$ 3.3}
& 20.5 {\scriptsize $\pm$ 1.2}
& 6.5 {\scriptsize $\pm$ 0.8}
& 25.0 {\scriptsize $\pm$ 2.5}
& 58.5 {\scriptsize $\pm$ 3.8}
& 45.5 {\scriptsize $\pm$ 3.3}
\\

\bottomrule
\end{tabular}
}
\end{table}

\noindent\textbf{KITTI Raw (automotive LiDAR).} KITTI Raw~\citep{geiger2013vision} is a standard dataset providing real-world automotive LiDAR acquisitions with known sensor poses. We use one
urban driving sequence, pairing each return with its acquisition ray, and reconstruct a static urban crop. As KITTI provides no reference surface,
we report qualitative results only (Figure~\ref{fig:kitti}). Supplementary MVS
reconstructions are reported in Appendix~\ref{app:additional-evaluations}.

{\setlength{\textfloatsep}{8pt}
\begin{figure}[t]
    \centering
    \includegraphics[width=\linewidth]{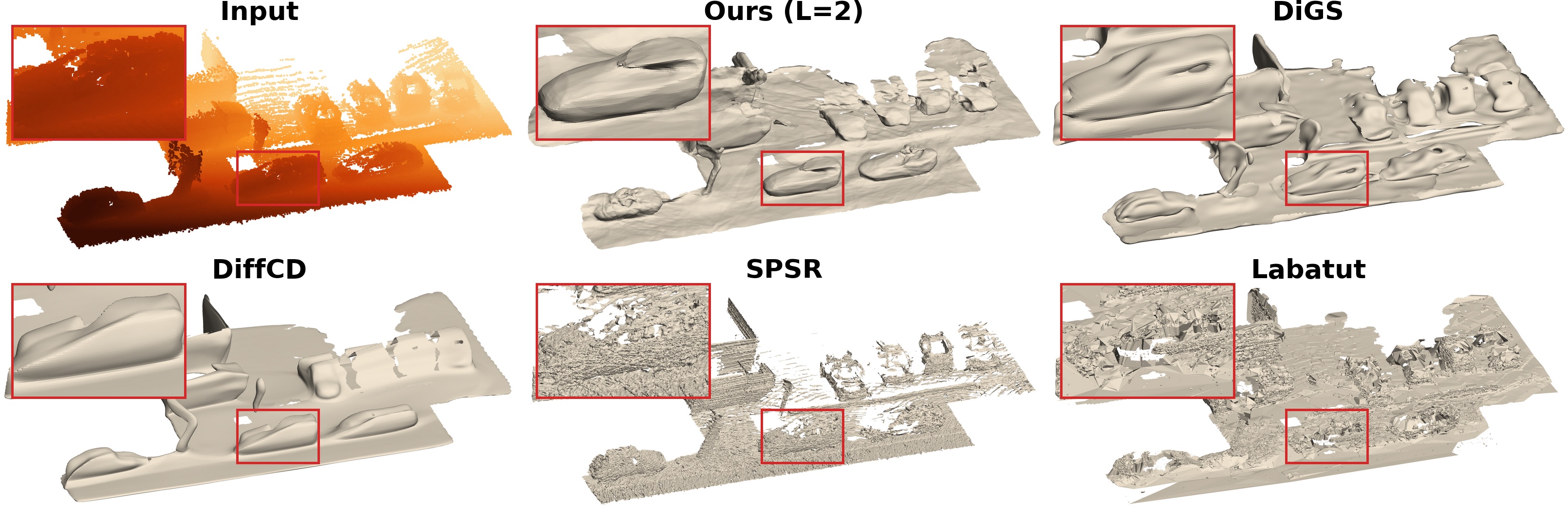}
    \caption{Qualitative KITTI reconstruction.}
    \label{fig:kitti}
\end{figure}

\section{Discussion}
\label{sec:discussion}

\noindent\textbf{Fidelity--robustness trade-off.} 
DiGS and DiffCD combine several weighted losses that must be balanced during
optimization and satisfying each objective is not guaranteed. For example, 
despite eikonal regularization, the eikonal condition is not enforced
in practice (Figure~\ref{fig:grad-slope-curves}).
In our case, $L$ provides a single parameter to study sensitivity to corruption: larger values fit clean observations more closely and smaller ones are more robust to corruption. Since the noise characteristics of a sensor are typically known, or can be
calibrated once on a reference object, $L$ can in principle be set accordingly.
Across our experiments, $L=2$
provides a good intermediate value.

\noindent\textbf{Role of visibility.} To test sensitivity to inaccurate sensor provenance, we perturb the origins $c_i$ while keeping the points $p_i$ fixed. 
We observe that moderate pose perturbations have little effect, whereas random sensor
reassignment makes most ray--point pairs inconsistent and causes
the loss of valid first hits (Appendix~\ref{sec:sensor-shuffle}), showing that acquisition rays carry geometric information. 

\noindent\textbf{Ablations.} Table~\ref{tab:ablations} summarizes the main ablations. Robust depth penalties
matter mainly under outliers: $\ell_1$ remains stable whereas $\ell_2$ degrades
strongly, and adding a miss penalty is harmful as it forces
the surface to pass through outliers. Incidence weighting primarily improves outlier robustness, while empirical rescaling improves performance across all three regimes.

\begin{table}[H]
\caption{SRB ablations, averaged over the five shapes; mean $\pm$ standard
deviation over three seeds.}
\label{tab:ablations}
\centering
\scriptsize
\vspace{-6pt}
\setlength{\tabcolsep}{4pt}
\renewcommand{\arraystretch}{1.2}

\resizebox{\linewidth}{!}{%
\begin{tabular}{lcccccc}
\toprule
\textbf{Variant}
& \multicolumn{2}{c}{\textbf{Clean ($\lambda=0$)}}
& \multicolumn{2}{c}{\textbf{Noise ($\lambda=1$)}}
& \multicolumn{2}{c}{\textbf{Outliers ($\alpha=0.2$)}} \\
\cmidrule(lr){2-3}
\cmidrule(lr){4-5}
\cmidrule(lr){6-7}
&
$\mathrm{CD}_1\downarrow$ & $\mathrm{NC}\uparrow$
& $\mathrm{CD}_1\downarrow$ & $\mathrm{NC}\uparrow$
& $\mathrm{CD}_1\downarrow$ & $\mathrm{NC}\uparrow$ \\
\midrule

Ours ($L=0$)
& $0.163_{\scriptscriptstyle \pm 0.005}$ & $0.958_{\scriptscriptstyle \pm 0.001}$
& $0.310_{\scriptscriptstyle \pm 0.016}$ & $0.906_{\scriptscriptstyle \pm 0.006}$
& $\mathbf{0.173}_{\scriptscriptstyle \pm 0.013}$ & $\mathbf{0.954}_{\scriptscriptstyle \pm 0.003}$ \\

$\ell_1$ loss
& $0.188_{\scriptscriptstyle \pm 0.0004}$ & $0.950_{\scriptscriptstyle \pm 0.0007}$
& $0.296_{\scriptscriptstyle \pm 0.001}$ & $0.911_{\scriptscriptstyle \pm 0.004}$
& $0.185_{\scriptscriptstyle \pm 0.004}$ & $0.950_{\scriptscriptstyle \pm 0.002}$ \\

$\ell_2$ loss
& $0.195_{\scriptscriptstyle \pm 0.002}$ & $0.951_{\scriptscriptstyle \pm 0.001}$
& $0.308_{\scriptscriptstyle \pm 0.002}$ & $\mathbf{0.921}_{\scriptscriptstyle \pm 0.0002}$
& $1.111_{\scriptscriptstyle \pm 0.010}$ & $0.708_{\scriptscriptstyle \pm 0.0015}$ \\

$-\,w_{\theta,i}$
& $0.177_{\scriptscriptstyle \pm 0.046}$ & $0.955_{\scriptscriptstyle \pm 0.016}$
& $0.302_{\scriptscriptstyle \pm 0.036}$ & $0.917_{\scriptscriptstyle \pm 0.013}$
& $0.237_{\scriptscriptstyle \pm 0.006}$ & $0.950_{\scriptscriptstyle \pm 0.0015}$ \\

$-$ rescaling
& $0.199_{\scriptscriptstyle \pm 0.004}$ & $0.947_{\scriptscriptstyle \pm 0.0007}$
& $0.339_{\scriptscriptstyle \pm 0.002}$ & $0.902_{\scriptscriptstyle \pm 0.0005}$
& $0.231_{\scriptscriptstyle \pm 0.009}$ & $0.929_{\scriptscriptstyle \pm 0.003}$ \\

$+$ miss penalty
& $\mathbf{0.158}_{\scriptscriptstyle \pm 0.0005}$ & $\mathbf{0.961}_{\scriptscriptstyle \pm 0.0004}$
& $\mathbf{0.284}_{\scriptscriptstyle \pm 0.007}$ & $0.917_{\scriptscriptstyle \pm 0.004}$
& $1.240_{\scriptscriptstyle \pm 0.102}$ & $0.834_{\scriptscriptstyle \pm 0.002}$ \\

\bottomrule
\end{tabular}%
}
\vspace{-4pt}
\end{table}

\noindent\textbf{Computational cost.}
Each sphere-tracing iteration requires a network evaluation, making the
forward pass costly, while implicit differentiation keeps the backward pass
efficient. Full-budget training is more expensive in wall-clock time
($\sim1$h$40$min, versus up to $40$min for the baselines;
Table~\ref{tab:compute}).
However, Ours ($L=6$) reaches within $2.8\%$ of its final $\mathrm{CD}_1$
after only $40\%$ of training ($\sim40$min).
Peak memory remains $3\times$ below DiGS and DiffCD.
\vspace{-0.4em}
\section{\vspace{-0.6em}Conclusion}

We formulate visibility-aware surface reconstruction as a robust first-hit depth
regression problem. Acquisition rays constrain both surface
position and preceding free space without auxiliary geometric supervision or regularization losses. Experiments show competitive
results on clean data and consistent gains under noise and outliers. This minimal formulation allows us to investigate how increasing $L$ improves fitting the input but also increases sensitivity to corruption.
Sphere tracing
also increases the optimization cost, although most of the accuracy is reached
early in the step budget (Table~\ref{tab:compute}). Natural extensions include jointly refining sensor poses, with respect to which the loss is differentiable, and exploring other 1-Lipschitz architectures. More fundamentally, finite
first-hit measurements do not uniquely determine the underlying surface.
Relating recoverable geometric detail to sampling density, surface regularity,
and model capacity is an important direction for future work.

\section*{Acknowledgments}
This work was supported in part by the French government under management of Agence Nationale de la Recherche as part of the “France 2030” program, PR[AI]RIE-PSAI project, reference ANR-23- IACL-0008. J.P. was supported in part by the Louis Vuitton/ENS chair in artificial intelligence, the Institute of Information \& Communications Technology Planning \& Evaluation (IITP) grant funded by the Korean Government (MSIT) (No. RS-2024-00457882, National AI Research Lab Project), and a Global Distinguished Professorship at the Courant Institute of Mathematical Sciences and the Center for Data Science at New York University.

\clearpage
\raggedbottom
\appendix

\section{Experimental Details}
\label{app:experimental-details}
\label{app:additional-experiments}

\subsection{Acquisition protocols}
\label{app:acquisition}

The acquisition geometries used for SRB and ShapeNet are illustrated in
Figures~\ref{fig:srb-acquisition} and~\ref{fig:shapenet-acquisition}.

\tikzset{
  sensor/.style={
    circle,
    fill=sensordk,
    inner sep=0pt,
    minimum size=4.4pt
  },
  meas/.style={
    circle,
    fill=measred,
    inner sep=0pt,
    minimum size=4.6pt
  },
  lbl/.style={font=\small},
  slbl/.style={font=\footnotesize},
  sphere/.style={
    acqraygray!68,
    line width=.65pt,
    densely dashed
  },
}


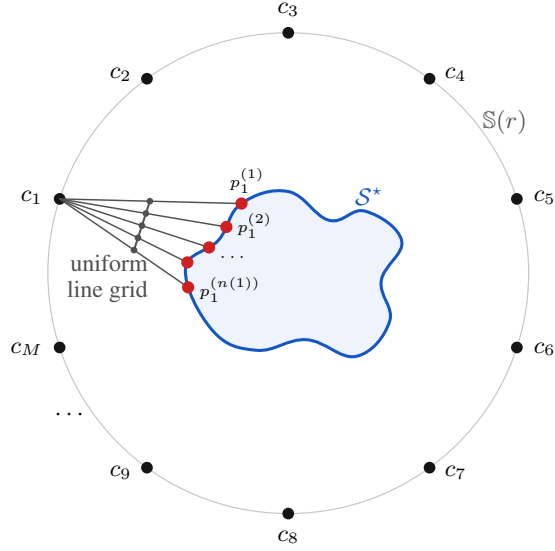
\begin{figure}[H]
\centering

\begin{tikzpicture}[x=1cm,y=1cm]

\def\shapeB{
  (10.58,2.74) (11.16,2.48) (11.82,2.62) (12.48,2.40)
  (12.92,2.86) (12.81,3.40) (13.08,3.93) (12.60,4.32)
  (12.12,4.20) (11.58,4.58) (10.96,4.42) (10.66,3.94)
  (10.22,3.55)
}

\coordinate (oB) at (11.58,3.50);
\def\rscan{3.18}

\draw[acqraygray!32,line width=.45pt] (oB) circle(\rscan);

\node[acqraygray,slbl,anchor=west]
  at ($(oB)+(2.45,2.00)$)
  {$\mathbb S(r)$};

\begin{scope}[on background layer]
  \fill[surfblue!7]
    plot[smooth cycle,tension=.70]
    coordinates{\shapeB};
\end{scope}

\draw[name path=surfB,surfblue,line width=1.15pt]
  plot[smooth cycle,tension=.70]
  coordinates{\shapeB};

\node[surfblue,slbl,anchor=west]
  at (12.34,4.52)
  {$\mathcal S^{\star}$};

\coordinate (cB1)  at ($(oB)+(162:\rscan)$);
\coordinate (cB2)  at ($(oB)+(126:\rscan)$);
\coordinate (cB3)  at ($(oB)+(90:\rscan)$);
\coordinate (cB4)  at ($(oB)+(54:\rscan)$);
\coordinate (cB5)  at ($(oB)+(18:\rscan)$);
\coordinate (cB6)  at ($(oB)+(-18:\rscan)$);
\coordinate (cB7)  at ($(oB)+(-54:\rscan)$);
\coordinate (cB8)  at ($(oB)+(-90:\rscan)$);
\coordinate (cB9)  at ($(oB)+(-126:\rscan)$);
\coordinate (cB10) at ($(oB)+(-162:\rscan)$);

\foreach \sensorpos in
  {cB1,cB2,cB3,cB4,cB5,cB6,cB7,cB8,cB9,cB10}{
  \node[sensor] at (\sensorpos){};
}

\node[slbl,anchor=east]  at ($(cB1)+(-0.10,0.02)$) {$c_1$};
\node[slbl,anchor=east]  at ($(cB2)+(-0.08,0.06)$) {$c_2$};
\node[slbl,anchor=south] at ($(cB3)+(0.00,0.10)$) {$c_3$};
\node[slbl,anchor=west]  at ($(cB4)+(0.08,0.06)$) {$c_4$};
\node[slbl,anchor=west]  at ($(cB5)+(0.10,0.02)$) {$c_5$};
\node[slbl,anchor=west]  at ($(cB6)+(0.10,-0.02)$) {$c_6$};
\node[slbl,anchor=west]  at ($(cB7)+(0.08,-0.06)$) {$c_7$};
\node[slbl,anchor=north] at ($(cB8)+(0.00,-0.10)$) {$c_8$};
\node[slbl,anchor=east]  at ($(cB9)+(-0.08,-0.06)$) {$c_9$};
\node[slbl] at ($(cB9)+(-1,0.7)$) {$\cdots$};
\node[slbl,anchor=east]
  at ($(cB10)+(-0.10,-0.02)$) {$c_M$};

\coordinate (g0) at ($(cB1)!0.36!(oB)$);

\coordinate (g1) at ($(g0)+(252:0.34)$);
\coordinate (g2) at ($(g0)+(252:0.17)$);
\coordinate (g3) at (g0);
\coordinate (g4) at ($(g0)+(72:0.17)$);
\coordinate (g5) at ($(g0)+(72:0.34)$);

\draw[acqraygray,line width=.75pt] (g1)--(g5);

\foreach \sample in {g1,g2,g3,g4,g5}{
  \fill[acqraygray] (\sample) circle(1.25pt);
}

\node[
  acqraygray,
  slbl,
  anchor=east,
  align=right
] at ($(g0)+(0.2,-0.7)$)
  {uniform\\line grid};

\path[name path=rayB1] (cB1)--($(cB1)!4!(g1)$);
\path[name path=rayB2] (cB1)--($(cB1)!4!(g2)$);
\path[name path=rayB3] (cB1)--($(cB1)!4!(g3)$);
\path[name path=rayB4] (cB1)--($(cB1)!4!(g4)$);
\path[name path=rayB5] (cB1)--($(cB1)!4!(g5)$);

\path[name intersections={
  of=rayB1 and surfB,
  sort by=rayB1,
  by={pB1,pB1far}
}];

\path[name intersections={
  of=rayB2 and surfB,
  sort by=rayB2,
  by={pB2,pB2far}
}];

\path[name intersections={
  of=rayB3 and surfB,
  sort by=rayB3,
  by={pB3,pB3far}
}];

\path[name intersections={
  of=rayB4 and surfB,
  sort by=rayB4,
  by={pB4,pB4far}
}];

\path[name intersections={
  of=rayB5 and surfB,
  sort by=rayB5,
  by={pB5,pB5far}
}];

\foreach \pt in {pB1,pB2,pB3,pB4,pB5}{
  \draw[acqraygray,line width=.55pt] (cB1)--(\pt);
}

\node[meas] at (pB1){};
\node[meas] at (pB2){};
\node[meas] at (pB3){};
\node[meas] at (pB4){};
\node[meas] at (pB5){};

\node[lbl,anchor=south west]
  at ($(pB3)+(0.15,0.6)$)
  {\tiny $p_1^{(1)}$};

\node[lbl,anchor=south west]
  at ($(pB3)+(0.25,0)$)
  {\tiny $p_1^{(2)}$};

\node[lbl,anchor=south west]
  at ($(pB3)+(0.02,-0.3)$)
  {\tiny $\cdots$};

\node[lbl,anchor=south west]
  at ($(pB3)+(-0.25,-0.85)$)
  {\tiny $p_1^{(n(1))}$};

\end{tikzpicture}

\caption{SRB acquisition protocol. Scanner poses lie on
\(\mathbb S(r)\); rays are cast through a uniform grid, and their first
surface intersections are retained.}
\label{fig:srb-acquisition}
\end{figure}


\begin{figure}[H]
\centering

\begin{tikzpicture}[x=1cm,y=1cm]

\def\shapeA{
  (3.02,2.74) (3.60,2.48) (4.26,2.62) (4.92,2.40)
  (5.36,2.86) (5.25,3.40) (5.52,3.93) (5.04,4.32)
  (4.56,4.20) (4.02,4.58) (3.40,4.42) (3.10,3.94)
  (2.66,3.55)
}

\coordinate (oA) at (4.02,3.50);
\def\rone{2.82}
\def\rtwo{3.18}
\def\raim{1.7}

\draw[acqraygray!42,line width=.45pt] (oA) circle(\rone);
\draw[acqraygray!28,line width=.45pt] (oA) circle(\rtwo);
\draw[name path=aimA,sphere] (oA) circle(\raim);

\node[acqraygray,slbl,anchor=west]
  at ($(oA)+(1.1,1.95)$)
  {$\mathbb{S}(r_1)$};

\node[acqraygray,slbl,anchor=west]
  at ($(oA)+(1.45,2.88)$)
  {$\mathbb{S}(r_2)$};

\node[acqraygray,slbl,anchor=west]
  at ($(2,4.8)$)
  {$\mathbb S_{\mathrm{aim}}$};

\begin{scope}[on background layer]
  \fill[surfblue!7]
    plot[smooth cycle,tension=.70]
    coordinates{\shapeA};
\end{scope}

\draw[surfblue,line width=1.15pt]
  plot[smooth cycle,tension=.70]
  coordinates{\shapeA};

\node[surfblue,slbl,anchor=west]
  at (4.78,4.52)
  {$\mathcal S^{\star}$};

\coordinate (cA1) at ($(oA)+(145:\rone)$);
\coordinate (cA2) at ($(oA)+(95:\rone)$);
\coordinate (cA3) at ($(oA)+(25:\rone)$);
\coordinate (cA4) at ($(oA)+(-45:\rone)$);
\coordinate (cA5) at ($(oA)+(-125:\rone)$);

\coordinate (cA6)  at ($(oA)+(165:\rtwo)$);
\coordinate (cA7)  at ($(oA)+(120:\rtwo)$);
\coordinate (cA8)  at ($(oA)+(40:\rtwo)$);
\coordinate (cA9)  at ($(oA)+(-70:\rtwo)$);
\coordinate (cA10) at ($(oA)+(-150:\rtwo)$);

\foreach \sensorpos in
  {cA1,cA2,cA3,cA4,cA5,cA6,cA7,cA8,cA9,cA10}{
  \node[sensor] at (\sensorpos){};
}

\node[slbl,anchor=east]  at ($(cA1)+(-0.08,0.02)$) {$c_1$};
\node[slbl,anchor=south] at ($(cA2)+(0,0.10)$) {$c_2$};
\node[slbl,anchor=west]  at ($(cA3)+(0.08,0.02)$) {$c_3$};
\node[slbl,anchor=west]  at ($(cA4)+(0.10,-0.02)$) {$c_4$};
\node[slbl,anchor=east]  at ($(cA5)+(-0.10,-0.02)$) {$c_5$};
\node[slbl,anchor=east]  at ($(cA6)+(-0.10,0.02)$) {$c_6$};
\node[slbl,anchor=south] at ($(cA7)+(0,0.10)$) {$c_7$};
\node[slbl,anchor=west]  at ($(cA8)+(0.08,0.02)$) {$c_8$};
\node[slbl,anchor=west]  at ($(cA9)+(0.10,-0.02)$) {$c_9$};
\node[slbl,anchor=east]  at ($(cA10)+(-0.10,-0.02)$) {$c_{10}$};

\coordinate (pA1) at (3.02,2.74);
\coordinate (pA2) at (2.66,3.55);
\coordinate (pA3) at (3.10,3.94);
\coordinate (pA4) at (3.40,4.42);

\foreach \pt in {pA1,pA2,pA3,pA4}{
  \draw[acqraygray,line width=.55pt] (cA6)--(\pt);
}

\node[meas] at (pA1){};
\node[meas] at (pA2){};
\node[meas] at (pA3){};
\node[meas] at (pA4){};

\path[name path=rayA1] (cA6)--(pA1);
\path[name path=rayA2] (cA6)--(pA2);
\path[name path=rayA3] (cA6)--(pA3);
\path[name path=rayA4] (cA6)--(pA4);

\path[name intersections={of=rayA1 and aimA, by=qA1}];
\path[name intersections={of=rayA2 and aimA, by=qA2}];
\path[name intersections={of=rayA3 and aimA, by=qA3}];
\path[name intersections={of=rayA4 and aimA, by=qA4}];

\foreach \qt in {qA1,qA2,qA3,qA4}{
  \fill[acqraygray] (\qt) circle(1.25pt);
}

\node[slbl,anchor=south west]
  at ($(pA3)+(0.25,0.1)$)
  {\tiny $p_6^{(1)}$};

\node[slbl,anchor=south west]
  at ($(pA3)+(0,-0.5)$)
  {\tiny $\cdots$};

\node[slbl,anchor=south west]
  at ($(pA3)+(-0.05,-1.4)$)
  {\tiny $p_6^{(n(6))}$};

\node[acqraygray,slbl,anchor=south west]
  at ($(qA3)+(0.2,0.24)$)
  {\tiny $q_6^{(1)}$};

\node[acqraygray,slbl,anchor=south west]
  at ($(qA3)+(-0.13,-0.4)$)
  {\tiny $\cdots$};

\node[acqraygray,slbl,anchor=south west]
  at ($(qA3)+(-1,-1.25)$)
  {\tiny $q_6^{(n(6))}$};

\node[
  draw=black!35,
  fill=black!4,
  rounded corners=.8pt,
  inner sep=.8pt,
  font=\tiny,
  anchor=west
] at (3,1.52)
  {$q_6^{(r)}\sim\mathrm{Unif}(\mathbb S_{\mathrm{aim}})$};

\end{tikzpicture}

\caption{ShapeNet acquisition protocol. Virtual sensors lie on two concentric
spheres; uniformly sampled targets on the enclosing sphere define ray
directions, and the first surface intersections are retained.}
\label{fig:shapenet-acquisition}
\end{figure}
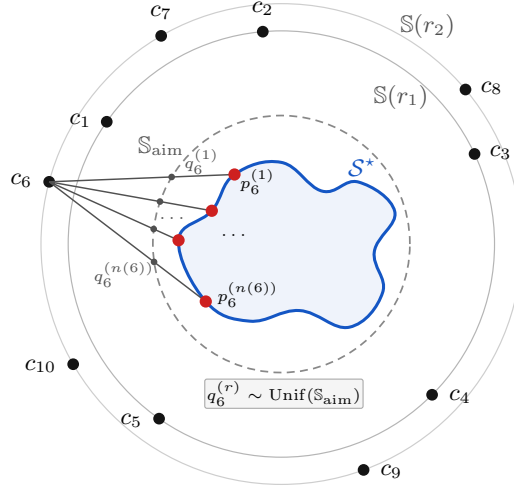

\subsection{Additional real-world acquisitions}
\label{app:additional-evaluations}

\paragraph{MVS reconstruction.}
We further evaluate on four DTU scenes~\citep{aanaes2016large}, using point
clouds reconstructed by fusing depth maps estimated with
ACMMP~\citep{xu2023acmmp}.

\begin{figure}[H]
\centering
\includegraphics[width=\linewidth]{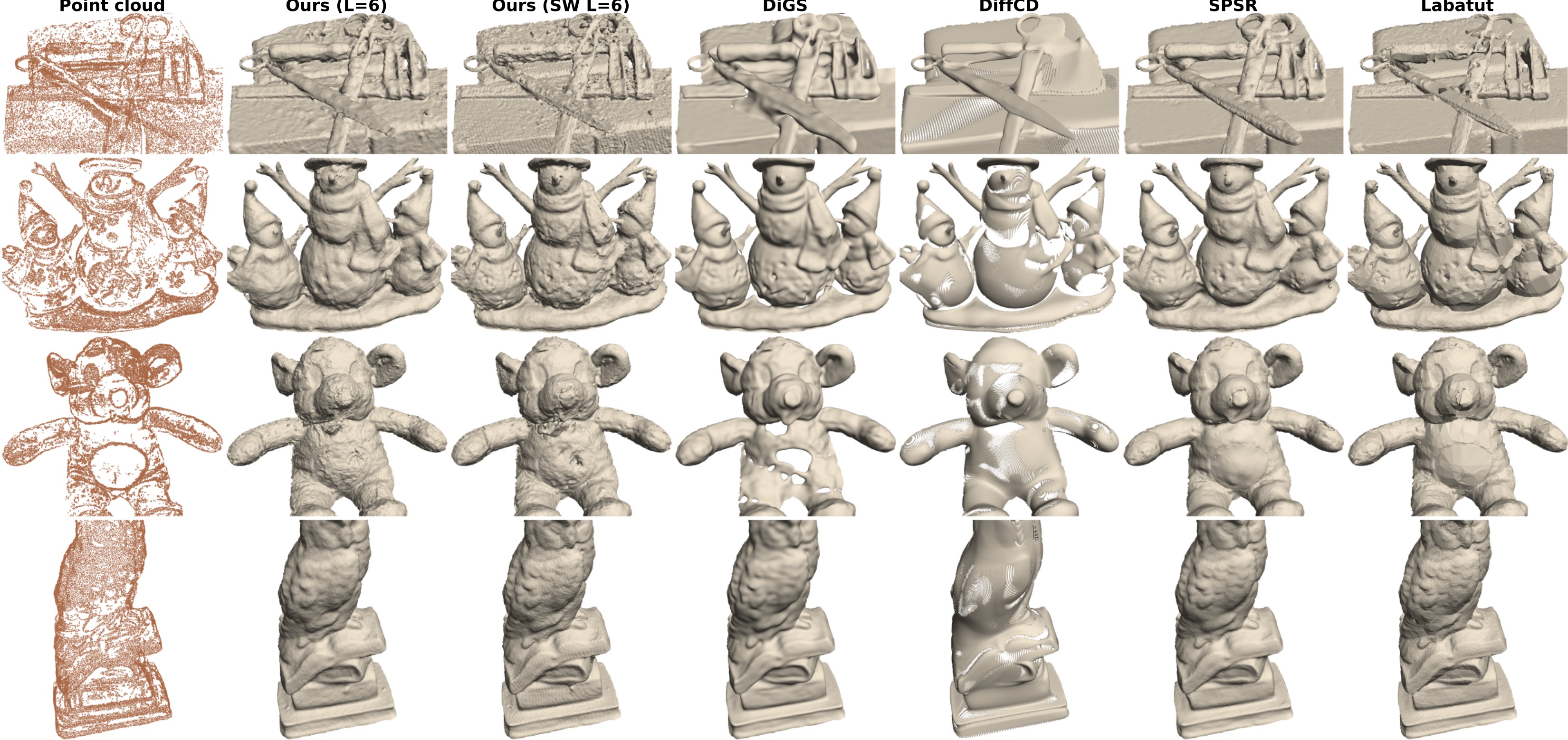}
\caption{Qualitative reconstruction on DTU MVS scenes 37, 69, 105 and 122. Ours (SW $L=6$) uses Sandwich layers (SW; Appendix~\ref{sec:sandwich-layer}).}
\label{fig:dtu}
\end{figure}

\subsection{Architecture and implementation}
\label{sec:implementation}

\paragraph{Input embedding.}
For our CPL implementation, following~\citet{meunier2022cpl}, we instantiate \(x\mapsto Ax\) as the fixed
zero-padding map
\begin{equation}
A=
\begin{bmatrix}
I_3\\
0
\end{bmatrix}.
\end{equation}
This embedding is isometric since
\(A^\top A=I_3\).

The backbone \(\Phi_\theta\) interleaves two types of \(1\)-Lipschitz layers.

\paragraph{MaxMin activation.}
For \(z\in\mathbb{R}^D\), we apply the MaxMin operator, which sorts coordinates
in pairs:
\begin{equation}
\operatorname{MaxMin}(z)
=
\bigl(
\max(z_{2k-1},z_{2k}),
\min(z_{2k-1},z_{2k})
\bigr)_{k=1}^{D/2}.
\end{equation}
MaxMin is norm-preserving and \(1\)-Lipschitz with respect to the
\(\ell_2\)-norm~\citep{anil2019sorting}.

\paragraph{Convex Potential Layers.}
CPL layers are parameterized as
\begin{equation}
\ell(z)
=
z-
\frac{2}{\|W\|_2^2}
W^\top\sigma(Wz+b),
\end{equation}
where \(\sigma=\operatorname{ReLU}\),
\(W\in\mathbb{R}^{m\times D}\), and \(\|W\|_2\) is the spectral norm.
By \citet{meunier2022cpl}, this residual mapping is \(1\)-Lipschitz.
We use CPL as a practical compromise between expressivity, memory, speed, and
certifiable Lipschitz control~\citep{prach2024layers}.

In practice, we approximate each spectral norm by power
iteration~\citep{meunier2022cpl}. Starting from a unit vector \(u\), we iterate
\begin{equation}
v
\leftarrow
\frac{Wu}{\|Wu\|_2},
\qquad
u
\leftarrow
\frac{W^\top v}{\|W^\top v\|_2},
\end{equation}
and estimate the dominant singular value as
\begin{equation}
\widehat{\sigma}(W)=v^\top Wu.
\end{equation}
With one power-iteration step per gradient step, $\widehat{\sigma}(W)$
matches the exact spectral norm $\|W\|_2$ to a median relative error of
$3\times10^{-13}\%$ (worst case $0.06\%$) across all CPL layers of 49
trained SRB models at their final checkpoint.

\paragraph{Implementation details.}
We use batches of \(|\mathcal D|\) rays for SRB and $16384$ for ShapeNet and real acquisitions and, following
DiGS and IGR~\citep{gropp2020igr,benshabat2022digs},
we extract meshes with marching cubes~\citep{lorensen1987marching} using a grid whose shortest axis
has 512 samples.

\paragraph{Sphere initialization.}

For a random initialization \(\theta_0\) of \(\Phi_\theta\), we sample \(M\) points \(x_j\sim\mathcal U([-ar_{\mathrm{init}},ar_{\mathrm{init}}]^3)\), with \(a>1\), and fit

$$ (\widetilde v,\widetilde\beta) \in\arg\min_{v,\beta}\sum_{j=1}^M \Bigl( v^\top\Phi_{\theta_0}(Ax_j)+\beta -(\|x_j\|_2-r_{\mathrm{init}}) \Bigr)^2. $$

We then set

$$ f_{\theta_0}(x) = \frac{ \widetilde v^\top\Phi_{\theta_0}(Ax)+\widetilde\beta }{ \|\widetilde v\|_2 }. $$

This preserves the fitted zero set and ensures \(\Lip(f_{\theta_0})\leq1\).

\subsection{Baseline details}
\label{app:baseline-details}

\paragraph{Visibility-supervised NDF.}
\label{app:ndf-vis}
We adapt \citet{coiffier2024lipschitz}, replacing its normal-derived winding-number supervision with visibility constraints. Points before $p_i$ are labeled outside,
\begin{equation}
x^{\mathrm{out}}_i
=
p_i-\varepsilon_i d_i,
\end{equation}
whereas points immediately behind it are labeled inside,
\begin{equation}
x^{\mathrm{in}}_i
=
p_i+\varepsilon_i d_i.
\end{equation}
Here,
\(\varepsilon_i=u_i s_i\),
\(u_i\sim\mathcal U(0.5,2)\), and \(s_i\) is the local nearest-neighbor
spacing. We additionally sample free-space points along
\(c_i+t(p_i-c_i)\), with \(t\in(0,1)\).
The original \(1\)-Lipschitz Semi-definite Programming Lipschitz Layer (SLL)
architecture and hKR objective are retained, with margin \(m=0.002\).

\subsection{Evaluation metrics}
\label{sec:metrics}

Let \(P\) and \(Q\) denote point sets sampled from the predicted and reference
surfaces, respectively.

\paragraph{Chamfer distance.}
We report the symmetric Chamfer distance
\begin{equation}
\label{eq:chamfer}
\mathrm{CD}_{q}(P,Q)
=
\frac{1}{2|P|}
\sum_{x\in P}
\min_{y\in Q}
\|x-y\|_2^{q}
+
\frac{1}{2|Q|}
\sum_{y\in Q}
\min_{x\in P}
\|y-x\|_2^{q}.
\end{equation}
Following DiGS~\citep{benshabat2022digs}, we use the unsquared distance
(\(q=1\)) on SRB and the squared distance (\(q=2\)) on ShapeNet.
We sample \(10^6\) points uniformly from each reconstructed mesh and
reference mesh. The reference surface is the reconbench mesh from which the
SRB scans were generated~\citep{berger2013benchmark}, and the watertight mesh
of \citet{mescheder2019occupancy} on ShapeNet.

\paragraph{Volumetric IoU.}
Volumetric intersection-over-union is estimated using a common set
\(\mathcal X\) of query points:
\begin{equation}
\label{eq:iou}
\mathrm{IoU}
=
\frac{
\sum_{z\in\mathcal X}
\mathbf{1}
\!\left[
z\in\Omega_{\mathrm{pred}}\cap\Omega_{\mathrm{gt}}
\right]
}{
\sum_{z\in\mathcal X}
\mathbf{1}
\!\left[
z\in\Omega_{\mathrm{pred}}\cup\Omega_{\mathrm{gt}}
\right]
}.
\end{equation}
We report IoU only on ShapeNet, following the DiGS occupancy protocol.

\paragraph{Normal consistency.}
Following~\citet{sulzer2025survey}, we sample \(10^5\)
oriented point--normal pairs from each surface and compute
\begin{equation}
\label{eq:normal-consistency}
\begin{aligned}
\mathrm{NC}(P,Q)
={}&
\frac{1}{2|P|}
\sum_{x\in P}
\left|
n_P(x)^\top
n_Q\!\left(\operatorname{NN}_{Q}(x)\right)
\right|
\\
&+
\frac{1}{2|Q|}
\sum_{y\in Q}
\left|
n_Q(y)^\top
n_P\!\left(\operatorname{NN}_{P}(y)\right)
\right|.
\end{aligned}
\end{equation}
Here, \(\operatorname{NN}_{Q}(x)\) denotes the nearest neighbor of \(x\) in
\(Q\), and analogously for \(\operatorname{NN}_{P}(y)\).

\paragraph{Observation-supported metrics.}
Inspired by the observation-supported evaluation of
\citet{marchand2021evaluating}, we additionally evaluate only surface samples
supported by the acquisition. Let \(A\) denote the clean acquisition point cloud and
\(\widetilde A\) the corresponding corrupted point cloud. We filter samples rather
than faces to avoid triangulation-dependent masking. For surface samples
\(P\) and \(Q\) from the predicted and reference meshes, respectively, we retain
\begin{equation}
P^{\mathrm{obs}}
=
\{x\in P:d(x,\widetilde A)<\beta h_0\},
\qquad
Q^{\mathrm{obs}}
=
\{y\in Q:d(y,A)<\beta h_0\},
\end{equation}
where
\begin{equation}
h_0
=
\operatorname{median}_{p_i\in A}
d\!\left(p_i,A\setminus\{p_i\}\right).
\end{equation}
We use \(\beta=2\) unless stated otherwise, and compute
\(\mathrm{CD}^{\mathrm{obs}}\) and
\(\mathrm{NC}^{\mathrm{obs}}\) by applying
Eqs.~\eqref{eq:chamfer} and~\eqref{eq:normal-consistency} to the retained
samples.

\begin{figure}[!ht]
\centering
\includegraphics[width=0.48\linewidth]{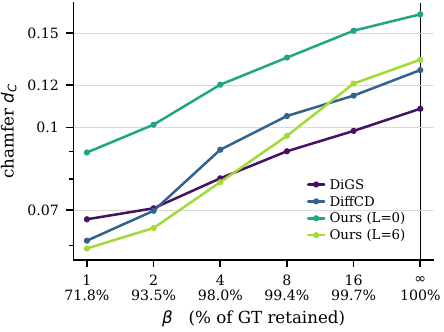}
\caption{Sensitivity of $\mathrm{CD}_1^{\mathrm{obs}}$ to the support
threshold \(\beta h_0\) on clean SRB (\(\lambda=0\)), averaged over the five
shapes; \(\beta=\infty\) recovers $\mathrm{CD}_1$.}
\label{fig:srb-beta-sweep}
\end{figure}

\begin{figure}[!ht]
\centering
\includegraphics[width=0.9\linewidth]{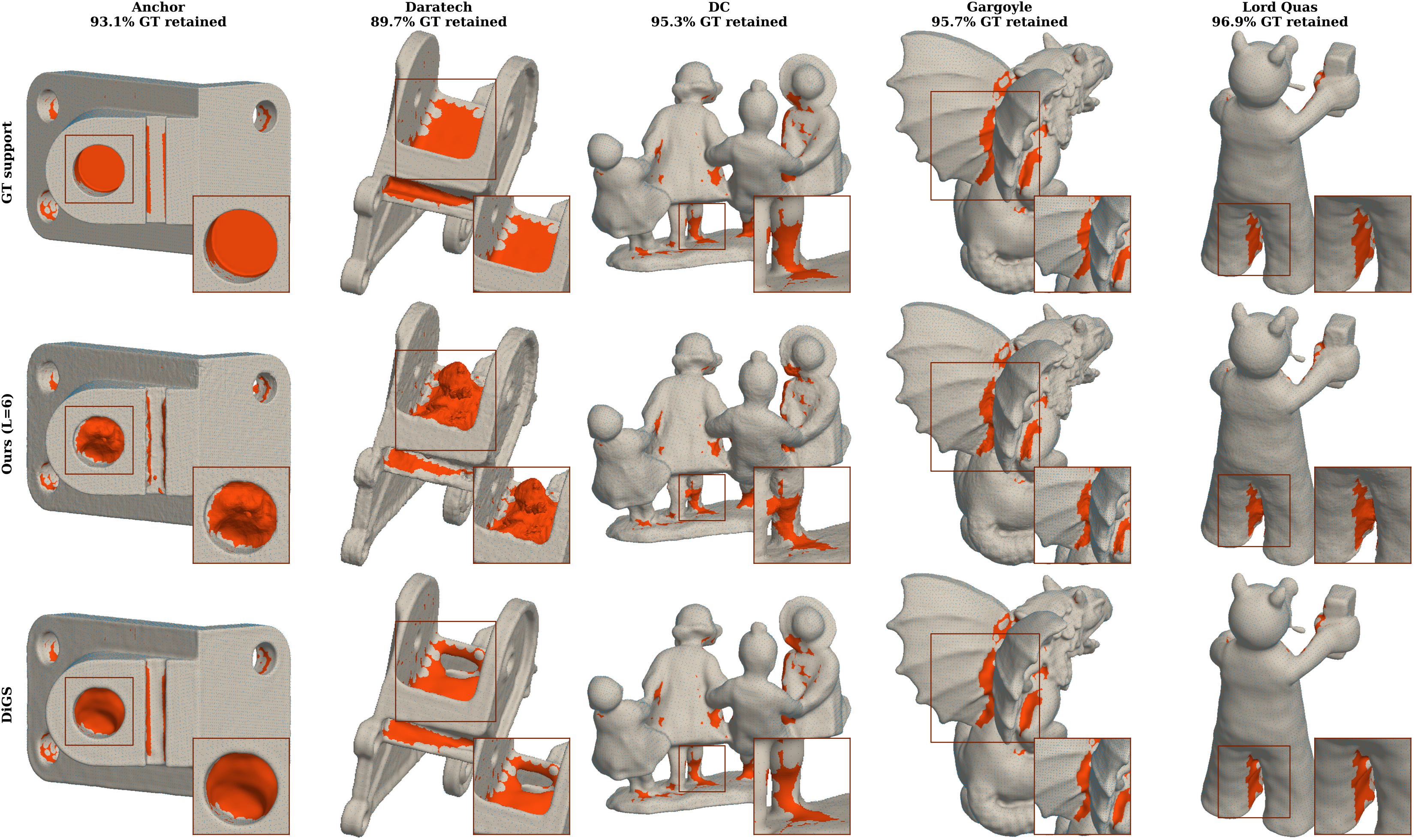}
\caption{Observation-supported SRB regions under the \(2h_0\) support mask;
orange denotes removed surface regions.}
\label{fig:observed-regions}
\end{figure}

\paragraph{ARKitScenes metrics.}
\label{sec:arkitscenes-metrics}
We evaluate rendered depth as in neural reconstruction~\citep{zhu2022niceslam,wang2023coslam}, ray-casting each predicted
mesh with Embree along FARO-derived reference rays.

For a ray \((o_i,d_i)\) with reference depth \(t_i^{\rm GT}\), when an intersection exists, let
\begin{equation}
t_i^{\rm pred}
=
\inf\{t\ge0:o_i+t d_i\in\mathcal S_{\rm pred}\},
\end{equation}
and let \(H\) denote the set of such rays.
We report
\begin{equation}
\mathrm{MAE}
=
\frac{1}{|H|}
\sum_{i\in H}
|t_i^{\rm pred}-t_i^{\rm GT}|,
\qquad
\mathrm{HitRate}
=
\frac{|H|}{N},
\end{equation}
and the all-ray depth inlier rate
\begin{equation}
I@\tau
=
\frac{1}{N}
\sum_{i=1}^{N}
\mathbf{1}
\!\left[
i\in H
\ \wedge\
|t_i^{\rm pred}-t_i^{\rm GT}|\le\tau
\right].
\end{equation}

\begin{figure}[H]
\centering
\includegraphics[width=0.62\linewidth]{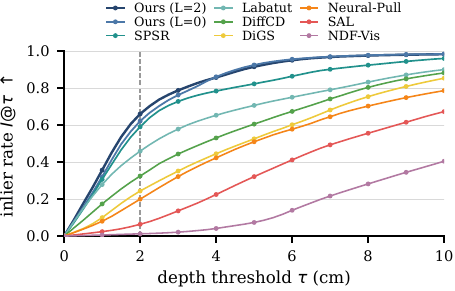}
\caption{ARKitScenes held-out depth inlier rate \(I@\tau\) as a function of
the threshold \(\tau\).}
\label{fig:arkit-inlier-curve}
\end{figure}

\section{Theory and Guarantees}
\label{app:theory}

\subsection{Lipschitz preliminaries}
\label{app:lipschitz-def}

For \(f:\mathbb{R}^d\to\mathbb{R}^m\), define
\begin{equation}
\Lip(f)
=
\sup_{x\neq x'}
\frac{\|f(x)-f(x')\|_2}{\|x-x'\|_2}.
\end{equation}
The function \(f\) is Lipschitz if \(\Lip(f)<\infty\), and
\(C\)-Lipschitz, for \(C\ge0\), if \(\Lip(f)\le C\).

\subsection{sphere tracing convergence}
\label{app:proof-sphere tracing}

We first show that \(t_k\le t_i(\theta)\). This holds for the base case
\(t_0=0\). If \(t_k<t_i(\theta)\), since \(\|d_i\|_2=1\) and
\(f_\theta(c_i+t_i(\theta)d_i)=0\),
\begin{equation}
\begin{aligned}
\left|f_\theta(c_i+t_kd_i)\right|
&=
\left|
f_\theta(c_i+t_kd_i)
-
f_\theta(c_i+t_i(\theta)d_i)
\right|
\\
&\le
t_i(\theta)-t_k.
\end{aligned}
\end{equation}
Hence
\begin{equation}
t_{k+1}
=
t_k+\left|f_\theta(c_i+t_kd_i)\right|
\le
t_i(\theta).
\end{equation}
Thus, \((t_k)\) is non-decreasing and bounded above, hence converges to a limit
\(\bar t\le t_i(\theta)\). Since
\begin{equation}
\left|f_\theta(c_i+t_kd_i)\right|
=
t_{k+1}-t_k
\longrightarrow 0,
\end{equation}
continuity gives
\begin{equation}
f_\theta(c_i+\bar t d_i)=0.
\end{equation}
By definition of the first root,
\(\bar t=t_i(\theta)\).

Furthermore, since \(t\mapsto f_\theta(c_i+t d_i)\) is differentiable at
\(t_i(\theta)\) and
\begin{equation}
\nabla_x
f_\theta(c_i+t_i(\theta)d_i)^\top d_i
\neq 0,
\end{equation}
there exist constants \(\kappa\in(0,1]\) and \(\gamma>0\) such that
\begin{equation}
\left|f_\theta(c_i+td_i)\right|
\ge
\kappa(t_i(\theta)-t)
\end{equation}
for all \(t\in(t_i(\theta)-\gamma,t_i(\theta))\). 
Therefore, once \(t_k\) enters this neighborhood, the error
\(e_k=t_i(\theta)-t_k\) satisfies
\begin{equation}
e_{k+1}
=
e_k-\left|f_\theta(c_i+t_kd_i)\right|
\le
(1-\kappa)e_k.
\end{equation}
Hence sphere tracing converges locally at a linear rate.
\hfill\(\square\)

\subsection{Incidence-weighted residual}
\label{app:incidence-correction}

\begin{proof}
After a rigid change of coordinates, assume
\(x_{\theta,i}=0\) and \(n_{\theta,i}=e_3\). Since
\(p_i-x_{\theta,i}=(t_i-t_i(\theta))d_i\) and \(\|d_i\|_2=1\), writing
\(p_i=(a_i,b_i)\in\mathbb R^2\times\mathbb R\) yields
\begin{equation}
\|(a_i,b_i)\|_2
=
|t_i-t_i(\theta)|,
\qquad
|b_i|
=
w_{\theta,i}|t_i-t_i(\theta)|.
\label{eq:incidence-proof-coordinates}
\end{equation}

Locally, \(\mathcal S_\theta\) is the graph
\(\{(a,g(a))\}\), where \(g(0)=0\) and \(\nabla g(0)=0\).
Let \((a_i^\star,g(a_i^\star))\) be a closest point to \(p_i\).
Since \(0\in\mathcal S_\theta\),
\begin{equation}
\|(a_i^\star,g(a_i^\star))\|_2
\le
\operatorname{dist}(p_i,\mathcal S_\theta)+\|p_i\|_2
\le
2\|p_i\|_2,
\label{eq:closest-point-local}
\end{equation}
so \(a_i^\star=O(|t_i-t_i(\theta)|)\), and the closest point lies in this
graph neighborhood for \(t_i\) sufficiently close to \(t_i(\theta)\).
Moreover,
\begin{equation}
|b_i|-|g(a_i^\star)|
\le
\operatorname{dist}(p_i,\mathcal S_\theta)
\le
|b_i-g(a_i)|
\le
|b_i|+|g(a_i)|.
\label{eq:incidence-proof-bounds}
\end{equation}

If \(f_\theta\) is CPWL and affine near \(x_{\theta,i}\), then \(g=0\)
locally and
\begin{equation}
\operatorname{dist}(p_i,\mathcal S_\theta)=|b_i|.
\end{equation}
Combining this with
Eqs.~\eqref{eq:incidence-proof-coordinates}--%
\eqref{eq:incidence-proof-bounds}
proves the claim.
\end{proof}

\begin{remark}
If \(\mathcal S_\theta\) is \(C^1\) near \(x_{\theta,i}\), then
\begin{equation}
\operatorname{dist}(p_i,\mathcal S_\theta)
=
w_{\theta,i}|t_i-t_i(\theta)|
+
o\!\left(|t_i-t_i(\theta)|\right),
\end{equation}
since \(g(a)=o(\|a\|_2)\). If \(\mathcal S_\theta\) is \(C^2\), Taylor's
theorem gives \(g(a)=O(\|a\|_2^2)\), hence
\begin{equation}
\operatorname{dist}(p_i,\mathcal S_\theta)
=
w_{\theta,i}|t_i-t_i(\theta)|
+
O\!\left(|t_i-t_i(\theta)|^2\right).
\end{equation}
\end{remark}

\subsection{Volumetric free space induced by ray supervision}
\label{app:certified-free-space}

Although certification is initially available only along the finite set of acquisition rays, the 1-Lipschitz constraint extends this guarantee to a three-dimensional neighborhood.

Let
$\Omega_\tau
=
B(0,\tau r_{\mathrm{init}})$
 with $\tau\in(0,1]$
denote a concentric sub-ball of the initialization ball \(\Omega_1\).

For each \(i\) such that \(t_i(\theta)<\infty\), let
\(x_{i,0},\ldots,x_{i,k(i)}\) denote its sphere tracing iterates.

\begin{proposition}[Certified free space]
\label{prop:certified-free-space}
Assume \(f_\theta\) is \(1\)-Lipschitz. Then
\begin{equation}
B^\circ\!\left(
x_{i,j},
|f_\theta(x_{i,j})|
\right)
\cap
\mathcal S_\theta
=
\varnothing.
\label{eq:certified-single-ball}
\end{equation}
Consequently,
\begin{equation}
\mathcal C_\theta(\mathcal D;\Omega_\tau)
=
\Omega_\tau
\cap
\bigcup_{i:t_i(\theta)<\infty}
\bigcup_{j=0}^{k(i)-1}
B^\circ\!\left(
x_{i,j},
|f_\theta(x_{i,j})|
\right)
\label{eq:certified-free-space}
\end{equation}
satisfies
\begin{equation}
\mathcal C_\theta(\mathcal D;\Omega_\tau)
\cap
\mathcal S_\theta
=
\varnothing.
\end{equation}
\end{proposition}

\begin{proof}
For any
\(
y\in
B^\circ(x_{i,j},|f_\theta(x_{i,j})|)
\),
the \(1\)-Lipschitz property gives
\begin{equation}
|f_\theta(y)|
\ge
|f_\theta(x_{i,j})|
-
\|y-x_{i,j}\|_2
>
0.
\end{equation}
Taking the union over all pre-hit iterates and intersecting with
\(\Omega_\tau\) proves the result.
\end{proof}

Let \(E^\star\subset\mathbb R^3\) denote the ground-truth solid.
We report the certified free-space volume normalized by the amount of
ground-truth free space contained in \(\Omega_\tau\):
\begin{equation}
\operatorname{CFV}_{\tau}(\theta;\mathcal D)
=
\frac{
\operatorname{Vol}\!\left(
\mathcal C_\theta(\mathcal D;\Omega_\tau)
\cap(\mathbb R^3\setminus E^\star)
\right)
}{
\operatorname{Vol}\!\left(
\Omega_\tau\cap(\mathbb R^3\setminus E^\star)
\right)
}.
\label{eq:cfv}
\end{equation}

\begin{table}[H]
\caption{Certified free-space volume
\(\operatorname{CFV}_\tau\) (\%) for
Ours ($L = 0$, no rescale) within concentric sub-balls of the
initialization domain; \(\tau=1\) denotes the full initialization ball.}
\label{tab:certified-free-space}
\centering
\begin{tabular}{ccccc}
\toprule
Relative radius \(\tau\)
& \(\lambda=0\)
& \(\lambda=1\)
& \(\alpha=0\)
& \(\alpha=0.2\) \\
\midrule
1.0 & 76.7 & 72.2 & 81.5 & 88.1 \\
0.9 & 80.9 & 75.9 & 85.3 & 86.3 \\
0.8 & 83.1 & 77.6 & 86.8 & 82.5 \\
0.7 & 83.9 & 77.6 & 87.0 & 78.2 \\
0.6 & 82.4 & 75.3 & 85.1 & 72.9 \\
\bottomrule
\end{tabular}
\end{table}

\begin{figure}[t]
\centering
\includegraphics[width=0.6\linewidth]{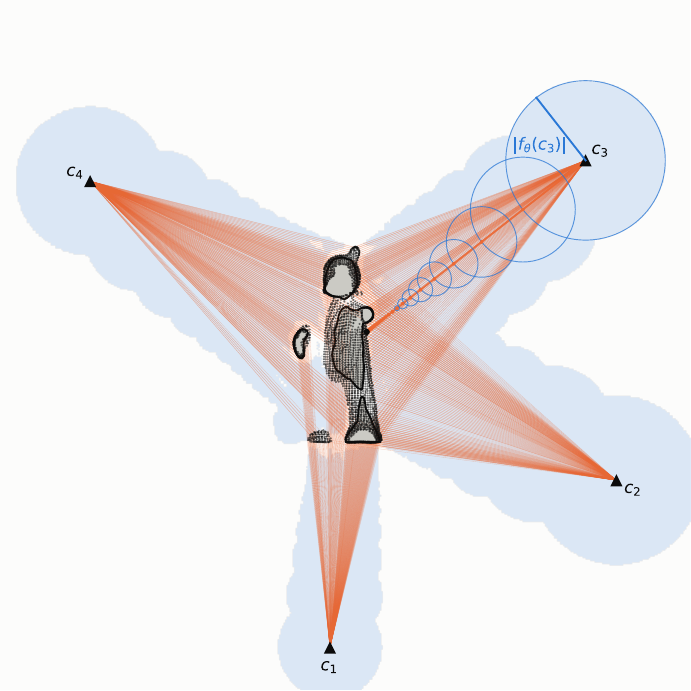}
\caption{Certified volumetric free-space regions on a single SRB shape.}
\label{fig:certified-slice}
\end{figure}

\section{Ablations and Diagnostics}
\label{app:ablations-diagnostics}

\subsection{Objective ablations}

\begin{table}[H]
\caption{Additional SRB objective ablations. Results use a single seed and are
averaged over the five SRB shapes. ReLU + eikonal uses the same first-hit loss and differs from Ours only by its architecture.}
\label{tab:full-ablations}
\centering
\scriptsize
\setlength{\tabcolsep}{4.5pt}
\renewcommand{\arraystretch}{1.08}

\begin{tabular}{lcccccc}
\toprule
\textbf{Variant}
& \multicolumn{2}{c}{\textbf{SRB \(\lambda=0\)}}
& \multicolumn{2}{c}{\textbf{SRB \(\lambda=1\)}}
& \multicolumn{2}{c}{\textbf{SRB \(\alpha=0.2\)}} \\
\cmidrule(lr){2-3}
\cmidrule(lr){4-5}
\cmidrule(lr){6-7}
&
\(\boldsymbol{\mathrm{CD}}_1\downarrow\)
& \(\mathbf{NC}\uparrow\)
& \(\boldsymbol{\mathrm{CD}}_1\downarrow\)
& \(\mathbf{NC}\uparrow\)
& \(\boldsymbol{\mathrm{CD}}_1\downarrow\)
& \(\mathbf{NC}\uparrow\) \\
\midrule

ReLU + eikonal
& 0.366 & 0.906
& 0.421 & 0.891
& 1.249 & 0.786 \\

CPL + eikonal
& 3.013 & 0.685
& 2.946 & 0.689
& 4.963 & 0.628 \\

CPL + hKR
& 0.576 & 0.940
& 0.719 & 0.940
& 1.226 & 0.915 \\

\midrule

\textbf{Ours ($L=0$)}
& \textbf{0.158} & \textbf{0.959}
& \textbf{0.296} & \textbf{0.911}
& \textbf{0.165} & \textbf{0.956} \\

\bottomrule
\end{tabular}
\end{table}
The ReLU + eikonal control consistently underperforms Ours ($L=0$), suggesting that eikonal regularization alone does not recover the benefits of the architectural Lipschitz constraint.
\\
For the miss-penalty ablation, rays for which sphere tracing returns no
intersection incur an additional penalty \(f_\theta(p_i)^2\). The objective is
\begin{equation}
\mathcal L(\theta)
+
\lambda_{\mathrm{miss}}
\mathcal L_{\mathrm{miss}}(\theta),
\end{equation}
with \(\lambda_{\mathrm{miss}}=5\).

\subsection{Huber-scale sensitivity}
\label{app:huber-delta}

We set
\begin{equation}
\delta
=
\begin{cases}
0.005,
&
N^{-1}\sum_i
\mathbf{1}[s_i>4\widetilde s]
>
0.5\%,
\\
0.05,
&
\text{otherwise},
\end{cases}
\end{equation}
where \(s_i\) is the nearest-neighbor distance of \(p_i\) and
\(\widetilde s\) its median.

\begin{figure}[H]
\centering
\includegraphics[width=0.48\linewidth]{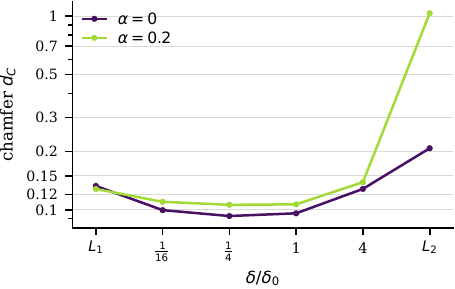}
\caption{Sensitivity of Ours ($L=0$) to the Huber scale \(\delta\) on a single SRB
shape, with \(\delta\) normalized by its default value \(\delta_0\). We use
\(\delta_0=0.05\) on clean data and \(\delta_0=0.005\) with outliers
(\(\alpha=0.2\)); \(\ell_1\) and \(\ell_2\) losses are shown for reference.}
\label{fig:huber-delta}
\end{figure}

\subsection{Backbone architecture}
\label{sec:sandwich-layer}
Since our formulation is agnostic to the particular $1$-Lipschitz architecture, we evaluate replacing the CPL backbone
\(\Phi_\theta\) with Sandwich layers~\citep{wang2023direct},
\begin{equation}
z
\mapsto
\sqrt{2}\,
U^\top\Psi\,
\sigma\!\left(
\sqrt{2}\Psi^{-1}Vz+b
\right),
\qquad
UU^\top+VV^\top=I,
\end{equation}
while keeping the training protocol unchanged.
Table~\ref{tab:cpl-sandwich} shows that the first-hit formulation remains
effective with this alternative \(1\)-Lipschitz backbone.

\begin{table}[!ht]
\caption{Preliminary single seed comparison of CPL and Sandwich backbones on SRB. }
\label{tab:cpl-sandwich}
\centering
\scriptsize
\setlength{\tabcolsep}{4pt}
\renewcommand{\arraystretch}{1.05}

\begin{tabular}{llcccc}
\toprule
\textbf{Setting} & \textbf{Backbone}
&$\mathrm{CD}_1$  $\downarrow$
& $\mathrm{NC}\uparrow$
& $\mathrm{CD}_1^{\mathrm{obs}}\downarrow$
& $\mathrm{NC}^{\mathrm{obs}}\uparrow$
\\
\midrule

$\lambda=0$
& CPL
& 0.1584 & 0.9597 & 0.0984 & 0.9698 \\
& Sandwich
& \textbf{0.1307} & \textbf{0.9728}
& \textbf{0.0801} & \textbf{0.9820} \\

\midrule

$\lambda=1$
& CPL
& 0.2950 & \textbf{0.9139} & \textbf{0.2205} & \textbf{0.9271} \\
& Sandwich
& \textbf{0.2818} & 0.9084 & 0.2229 & 0.9204 \\

\midrule

$\alpha=0.2$
& CPL
& 0.1650 & 0.9571 & 0.1062 & 0.9697 \\
& Sandwich
& \textbf{0.1311} & \textbf{0.9687}
& \textbf{0.0882} & \textbf{0.9791} \\

\bottomrule
\end{tabular}
\end{table}

\subsection{Depth residual and projection distance}
\label{app:projection-distance}

\begin{figure}[H]
\centering
\includegraphics[width=\linewidth]{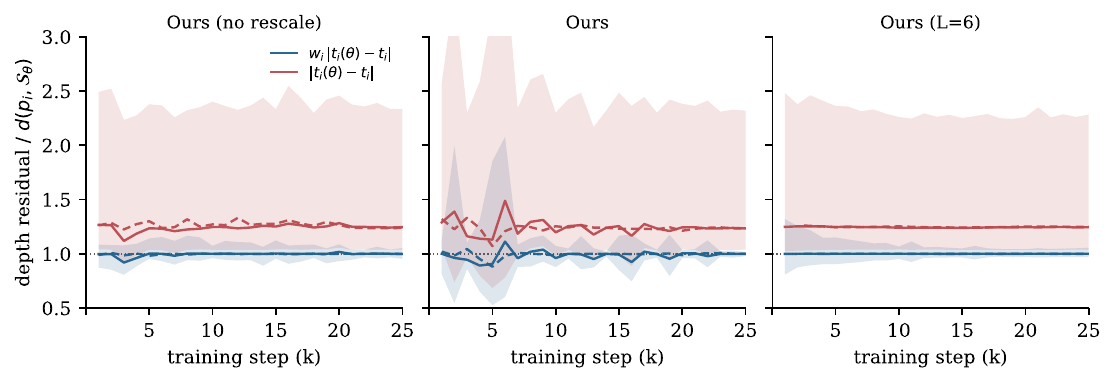}
\caption{Ratio of the depth residual to the true projection distance
\(d(p_i,\mathcal S_\theta)\) during a 25$k$-step fit on a single SRB shape with $\lambda=0$
(median and \(p_{10}\)--\(p_{90}\) over 4$k$ rays). Solid curves use incidence weighting; dashed curves
remove it.}
\label{fig:projection-ratio}
\end{figure}

\subsection{Gradient norms}
\label{app:grad-norm}

\begin{figure}[H]
\centering
\includegraphics[width=\linewidth]{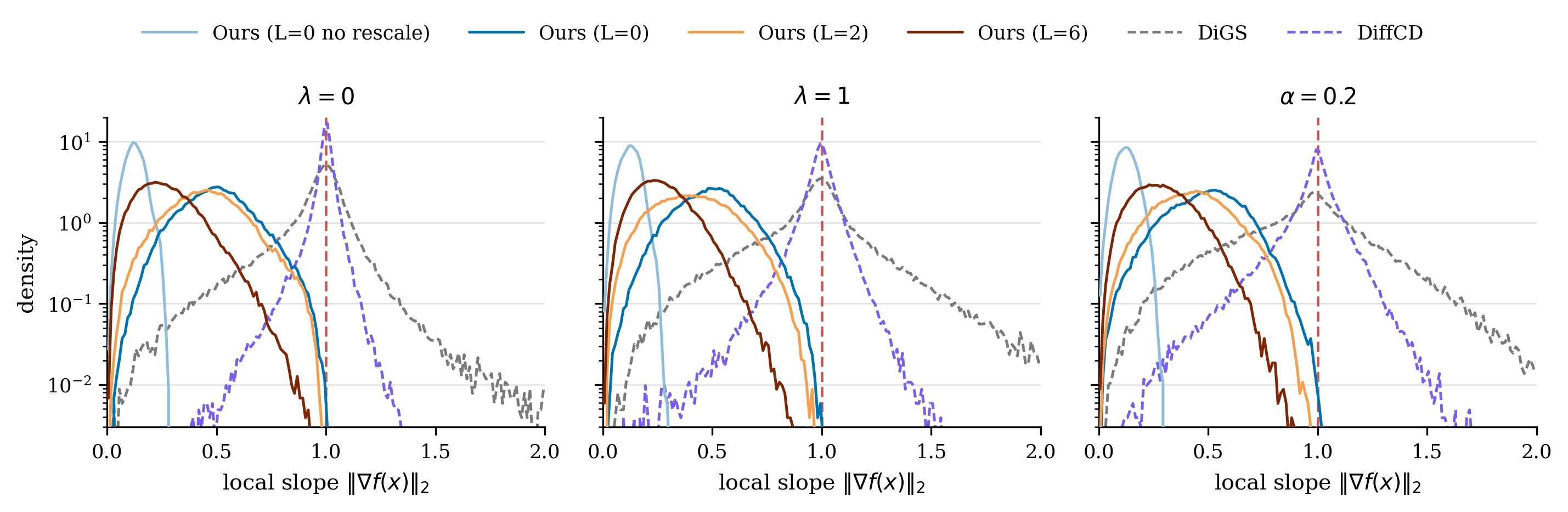}
\caption{Density of
\(\|\nabla_x f_\theta(x)\|_2\) at points sampled along the traced rays on the five SRB
shapes under \(\lambda=0\), \(\lambda=1\), and \(\alpha=0.2\)
($82k$ samples per curve and $16,384$ per shape). The axis is
clipped at \(2\).
}
\label{fig:grad-slope-curves}
\end{figure}

\begin{table}[!ht]
\centering
\caption{Effect of empirical rescaling for the $L=2$ and $L=6$ variants under noise and outliers. Values are $\Delta$(rescale $-$ no rescale), averaged over the five SRB shapes for a single seed.}
\label{tab:pe-rescaling}
\small
\begin{tabular}{llccc}
\toprule
$L$ & Setting & No rescale & Rescaled & $\Delta$ \\
\midrule
2 & $\lambda=0$  & 0.124 & 0.128 & +0.004 {\scriptsize $(+3.2\%)$} \\
2 & $\lambda=1$  & 0.425 & 0.381 & -0.044 {\scriptsize $(-10.4\%)$} \\
2 & $\alpha=0.2$ & 0.125 & 0.122 & -0.003 {\scriptsize $(-2.4\%)$} \\
6 & $\lambda=0$  & 0.134 & 0.133 & -0.001 {\scriptsize $(-0.7\%)$} \\
6 & $\lambda=1$  & 0.543 & 0.542 & -0.001 {\scriptsize $(-0.2\%)$} \\
6 & $\alpha=0.2$ & 0.489 & 0.414 & -0.075 {\scriptsize $(-15.3\%)$} \\
\bottomrule
\end{tabular}
\end{table}

\subsection{Outlier-ray diagnostics}

\begin{table}[H]
\caption{Ray statistics at convergence on SRB with outliers
(\(\alpha=0.2\)), averaged over the five shapes. Dropped rays are those for which sphere tracing finds no valid hit. Non-grazing rays are partitioned by the Huber regime of their
weighted residual.}
\label{tab:outlier_ray_stats}
\centering
\small
\setlength{\tabcolsep}{6pt}
\renewcommand{\arraystretch}{1.05}

\begin{tabular}{llccc}
\toprule
\textbf{Variant} & \textbf{Rays}
& \textbf{Dropped}
& \textbf{Huber linear}
& \textbf{Huber quadratic} \\
\midrule

Ours ($L=6$)
& Inlier  & 2.4\%  & 5.9\%  & 91.7\% \\
& Outlier & 41.0\% & 57.0\% & 2.1\% \\

Ours ($L=0$)
& Inlier  & 0.8\%  & 8.0\%  & 91.2\% \\
& Outlier & 39.3\% & 60.1\% & 0.6\% \\

\bottomrule
\end{tabular}
\end{table}

\subsection{Sensitivity to sensor provenance}
\label{sec:sensor-shuffle}
The
perturbation scale \(\sigma\) of the sensor position $c_i$ is expressed in multiples of the point cloud's
median nearest-neighbor spacing \(h_{\mathrm{NN}}\); the corresponding mean
position error ranges from \(0.9\%\) to \(7.4\%\) of the scanner standoff $\operatorname{median}_i \|p_i-c_i\|_2$.

\begin{table}[H]
\caption{Sensitivity to sensor-position perturbations on a single SRB shape at $\lambda=0$.}
\label{tab:sensor-position-perturbation}
\centering
\small
\setlength{\tabcolsep}{6pt}

\begin{tabular}{lccc}
\toprule
\(\sigma/h_{\mathrm{NN}}\)
& Position error / standoff
& CD\(\downarrow\)
& NC\(\uparrow\) \\
\midrule
0  & 0.0\% & 0.071 & 0.976 \\
2  & 0.9\% & 0.071 & 0.976 \\
4  & 1.8\% & 0.077 & 0.973 \\
8  & 3.7\% & 0.079 & 0.975 \\
16 & 7.4\% & 0.098 & 0.971 \\
\bottomrule
\end{tabular}
\end{table}

Performance is nearly unchanged up to \(3.7\%\) position error relative to the
scanner standoff, and reconstruction still succeeds at \(7.4\%\). These
perturbations are deliberately substantial and are not intended as a precise
model of scanner calibration errors.

As an extreme control, we randomly permute sensor origins, replacing each pair
\((c_i,p_i)\) by \((c_{\pi(i)},p_i)\), where \(\pi\) is a permutation of
\(\{1,\ldots,N\}\) with \(c_{\pi(i)}\neq c_i\) for all \(i\). All 15 fits (five SRB shapes under
\(\lambda=0\), \(\lambda=1\), and \(\alpha=0.2\)) reach a state with no valid
traced hits after a few hundred optimization steps. Under the shuffled
associations, \(62\)--\(65\%\) of rays are occluded before their measured
endpoint, showing that most shuffled ray--point pairs are incompatible with
the observed visibility.

\subsection{Fidelity--robustness analysis}

\paragraph{Input fidelity.}
\label{sec:data_fidelity}

Let \(P\) denote the input point cloud and \(\widehat{\mathcal S}\) the
reconstructed mesh. We measure fidelity to the observations by the mean exact
input-to-surface distance
\begin{equation}
d_{\mathrm C}^{\rightarrow}(P,\widehat{\mathcal S})
=
\frac{1}{|P|}
\sum_{p\in P}
\min_{x\in\widehat{\mathcal S}}
\|p-x\|_2.
\end{equation}

\begin{table}[H]
\caption{\textbf{Clean-SRB data fidelity.}
Mean input-to-surface distance; best/second-best neural implicit results are
bold/underlined. Labatut is omitted since its Delaunay surface is constructed
from the inputs.}
\label{tab:data-fidelity}
\centering
\scriptsize
\setlength{\tabcolsep}{2.2pt}
\renewcommand{\arraystretch}{0.95}

\resizebox{\linewidth}{!}{%
\begin{tabular}{lccccccccc|c}
\toprule
& \multicolumn{9}{c|}{\textit{Neural implicit}}
& \textit{Classical} \\
\cmidrule(lr){2-10}\cmidrule(l){11-11}
\textbf{Shape}
& \textbf{$L=0$}
& \textbf{$L=2$}
& \textbf{$L=4$}
& \textbf{$L=6$}
& \textbf{DiffCD}
& \textbf{N.-Pull}
& \textbf{DiGS}
& \textbf{NDF-vis}
& \textbf{SAL}
& \textbf{SPSR} \\
\midrule

Anchor
& 0.0309 & 0.0186 & 0.0124 & \underline{0.0118}
& \textbf{0.0095} & 0.0248 & 0.0389 & 0.3024 & 0.4796 & 0.0413 \\

Daratech
& 0.1156 & 0.0366 & \underline{0.0160} & \textbf{0.0134}
& 0.0337 & 0.3983 & 0.0256 & 0.4017 & 0.4271 & 0.1027 \\

DC
& 0.0407 & 0.0238 & 0.0119 & \textbf{0.0085}
& \underline{0.0114} & 0.3917 & 0.0278 & 0.5725 & 0.4795 & 0.0580 \\

Gargoyle
& 0.0825 & 0.0441 & 0.0249 & \textbf{0.0136}
& \underline{0.0198} & 0.0324 & 0.0390 & 0.6962 & 0.6387 & 0.0639 \\

Lord Quas
& 0.0291 & 0.0148 & \underline{0.0102} & \textbf{0.0064}
& 0.0132 & 0.0143 & 0.0306 & 0.2317 & 0.4496 & 0.0339 \\

\midrule
\textbf{Mean}
& 0.0597 & 0.0276 & \underline{0.0151} & \textbf{0.0107}
& 0.0175 & 0.1723 & 0.0324 & 0.4409 & 0.4949 & 0.0599 \\

\bottomrule
\end{tabular}%
}
\end{table}

\begin{figure}[!t]
    \centering
    \includegraphics[width=\linewidth]{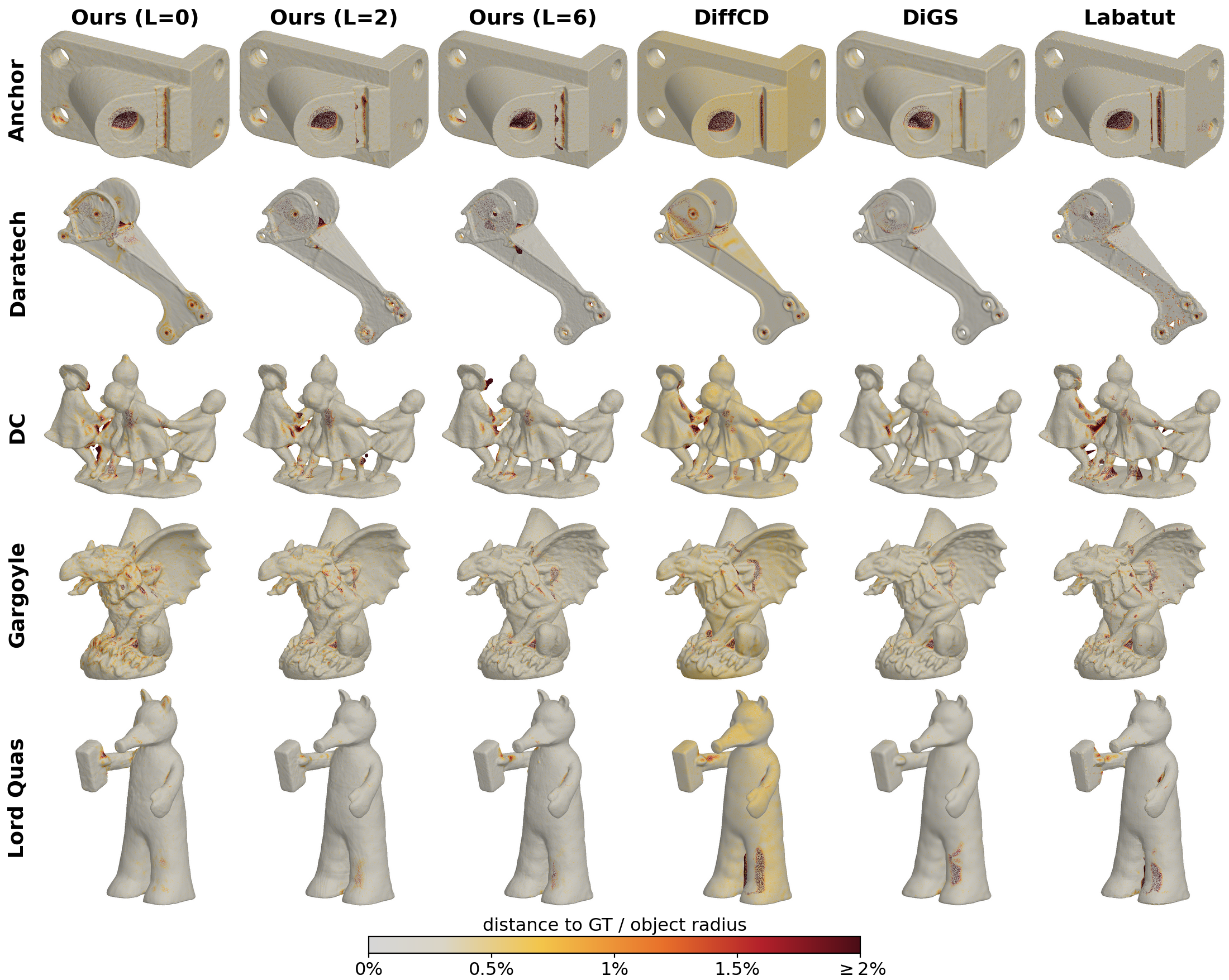}
    \caption{Clean-SRB surface error maps across the five shapes.
    Colors show distance to the ground-truth surface normalized by object radius.}
    \label{fig:clean-srb-error}
\end{figure}

\begin{figure}[H]
\centering
\includegraphics[width=0.49\linewidth]{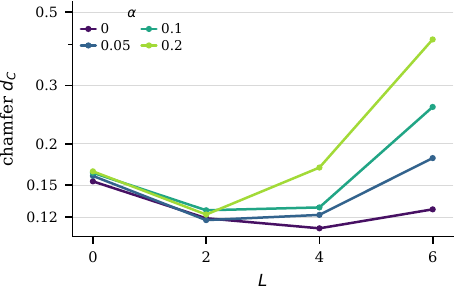}
\hfill
\includegraphics[width=0.49\linewidth]{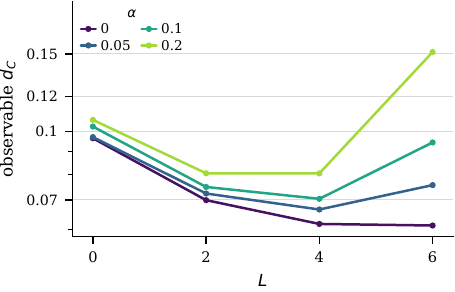}
\caption{Mean $\mathrm{CD}_1$ (left) and $\mathrm{CD}_1^\mathrm{obs}$ (right)
 as a function of \(L\) under
outlier fractions \(\alpha\) on SRB.}
\label{fig:peL-robustness_outlier}
\end{figure}

\begin{figure}[H]
\vspace{-3pt}
\centering

\includegraphics[width=0.49\linewidth]{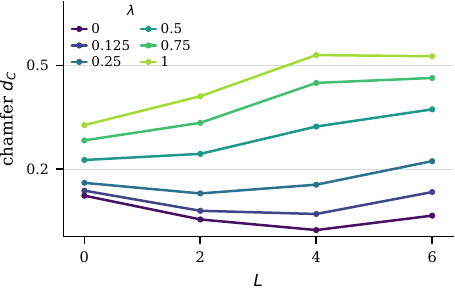}
\hfill
\includegraphics[width=0.49\linewidth]{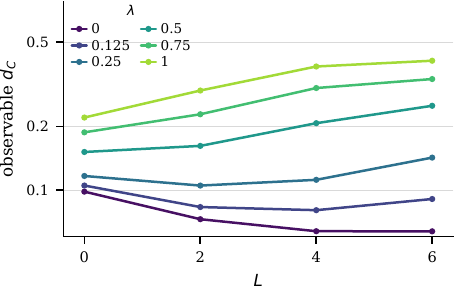}

\caption{Mean $\mathrm{CD}_1$ (left) and $\mathrm{CD}_1^\mathrm{obs}$ (right) as a function of $L$ under
scanner noise $\lambda$ on SRB, on a single seed.}
\label{fig:peL-robustness_noise}
\end{figure}

\begin{figure}[H]
    \centering
    \includegraphics[width=0.8\linewidth]{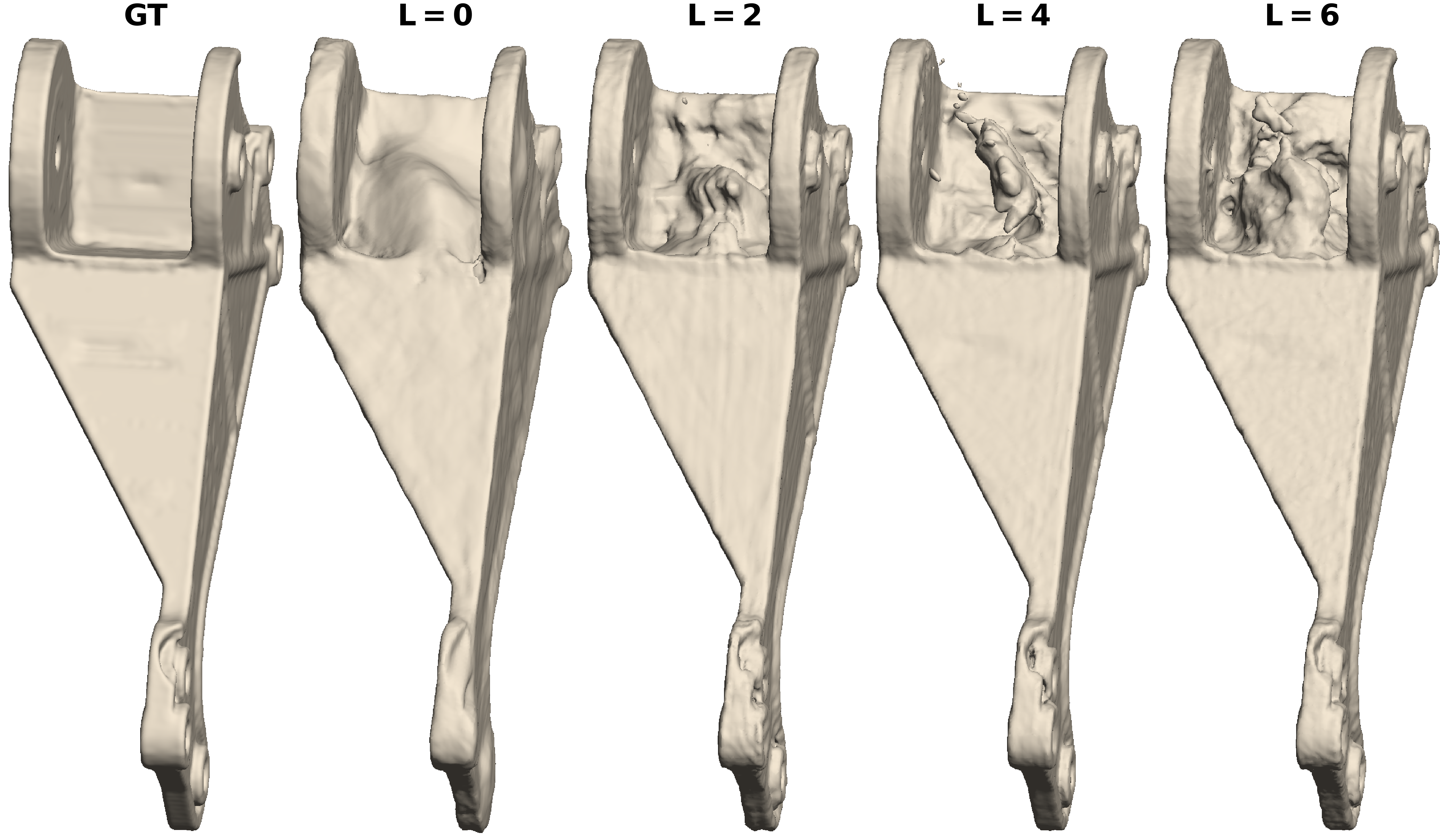}
    \caption{Effect of \(L\) in a weakly observed concave region.}
    \label{fig:pe-unobserved}
\end{figure}

\begin{figure}[!ht]
\centering
\includegraphics[width=\linewidth]{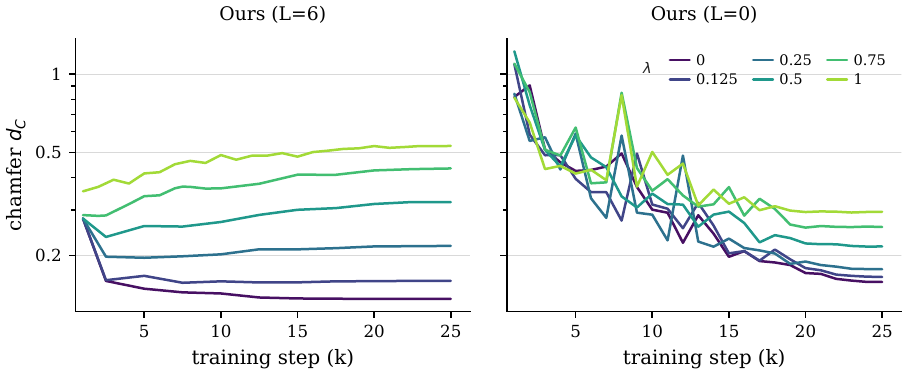}

\includegraphics[width=\linewidth]{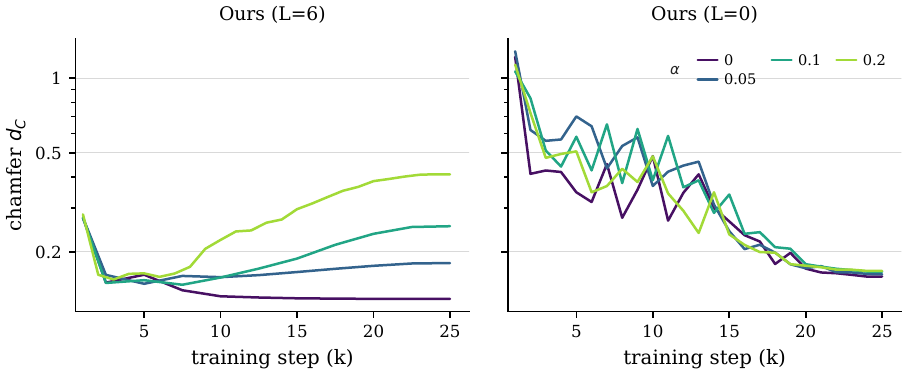}

\caption{Mean SRB \(\mathrm{CD}_1\) with one seed during training under scanner
noise (top) and outlier fractions (bottom).}
\label{fig:checkpoint_robustness}
\end{figure}

\subsection{Largest-component filtering}

Following prior observations that surface reconstruction methods may produce
spurious disconnected components~\citep{berger2013benchmark}, including SPSR
under strong corruption~\citep{jensen2021surface}, we additionally evaluate all
methods after retaining only their largest connected component.

\begin{table}[H]
\caption{Effect of largest-connected-component filtering on SRB with outliers
($\alpha=0.2$).}
\label{tab:lcc}
\centering
\small

\begin{tabular}{lcccc}
\toprule
Method
& Raw $\mathrm{CD}_1\downarrow$
& LCC $\mathrm{CD}_1\downarrow$
& Raw NC$\uparrow$
& LCC NC$\uparrow$ \\
\midrule
Ours ($L=0$) & 0.173 & 0.172 & 0.954 & 0.955 \\
Ours ($L=2$) & 0.122 & 0.122 & 0.968 & 0.969 \\
Ours ($L=6$) & 0.414 & 0.397 & 0.818 & 0.820 \\
NDF-vis      & 5.492 & 5.472 & 0.723 & 0.723 \\
Neural-Pull  & 3.420 & 6.680 & 0.756 & 0.714 \\
DiffCD       & 3.354 & 3.319 & 0.806 & 0.809 \\
SAL          & 5.731 & 5.850 & 0.652 & 0.650 \\
DiGS         & 2.927 & 1.127 & 0.814 & 0.898 \\
SPSR         & 4.917 & 0.158 & 0.855 & 0.962 \\
Labatut      & 0.144 & 0.141 & 0.946 & 0.947 \\
\bottomrule
\end{tabular}
\end{table}

\subsection{Computational cost}
\label{sec:compute}
The CPL network has $\underbrace{8(256^2+256)}_{\Phi_\theta:\,8\ \mathrm{CPL\ layers}\ (W,b)}
+\underbrace{256}_{\widetilde v}
+\underbrace{1}_{\beta}
=526{,}593$ trainable parameters. 

We report wall-clock training time and peak memory usage for a single SRB
shape.

\begin{table}[H]
\caption{
\textbf{Left:} wall-clock training time and peak GPU memory
on one SRB shape using an NVIDIA RTX 6000.
\textbf{Right:} $\mathrm{CD}_1$ averaged over the five SRB shapes at intermediate mesh resolution $256^3$ for Ours ($L=6$) every $5k$ training steps.}
\label{tab:compute}
\centering
\scriptsize
\setlength{\tabcolsep}{4pt}
\renewcommand{\arraystretch}{1.05}

\begin{tabular}{@{}lcc@{\hspace{1.8em}}rccc@{}}
\multicolumn{3}{c}{\textbf{(a) Method comparison}}
&
\multicolumn{4}{c}{\textbf{(b) Step budget}}
\\
\cmidrule(r){1-3}
\cmidrule(l){4-7}

Method
& Time $\downarrow$
& Peak GPU (GB) $\downarrow$
& Steps
& Budget
& CD$_1\downarrow$
& $\Delta$ \\
\midrule

SAL          & $\sim$1 min   & 0.73 & 5k  & 20\%  & 0.171 & +21.3\% \\
Neural-Pull  & $\sim$5 min   & 0.42 & 10k & 40\%  & 0.145 & +2.8\% \\
NDF-vis      & $\sim$10 min  & 0.63 & 15k & 60\%  & 0.144 & +2.1\% \\
DiGS         & $\sim$20 min  & 2.94 & 20k & 80\%  & 0.141 & +0.0\% \\
DiffCD       & $\sim$40 min  & 2.52 & 25k & 100\% & 0.141 & --- \\
Ours ($L=0$) & $\sim$1h30min & 0.84 &     &       &       &     \\
Ours ($L=6$) & $\sim$1h40min      & 0.84 &     &       &       &     \\

\bottomrule
\end{tabular}
\end{table}

For each ray \(i\), let \(k_i\le K\) denote the number of sphere tracing
iterations actually executed. We use active-ray compaction: terminated rays
are removed from subsequent network evaluations. The average tracing workload
is therefore characterized by \(\mathbb E[k]/K\).

\begin{table}[H]
\caption{sphere tracing workload over SRB training rays with \(K=192\).
Convergence rates are min/mean/max over the five SRB shapes.}
\label{tab:tracing_workload}
\centering
\small
\setlength{\tabcolsep}{5pt}

\begin{tabular}{lcccc}
\toprule
\textbf{Model}
& $\mathbb E[k]/K$
& \textbf{Mean $k$}
& \textbf{Reaching $K$}
& \textbf{Converged (\%)} \\
\midrule
Ours ($L=0$, no rescale)
& 40.8\% & 78.4 & 2.13\%
& $99.92 / \mathbf{99.96} / 100.00$ \\

Ours ($L=0$)
& 9.4\% & 18.1 & 0.002\%
& $99.91 / \mathbf{99.96} / 99.99$ \\

Ours ($L=6$)
& 22.2\% & 42.6 & 0.08\%
& $98.73 / \mathbf{99.37} / 99.92$ \\
\bottomrule
\end{tabular}
\end{table}

\section{Extended Results}
\label{app:extended-results}

\subsection{SRB scanner noise}
\label{app:srb-noise-sweep}

\begin{figure}[!ht]
    \centering
    \includegraphics[width=\linewidth]{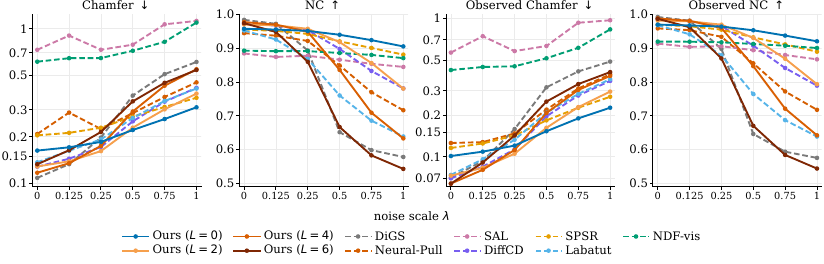}
    \caption{Mean SRB \(\mathrm{CD}_1\) and NC across increasing scanner-noise levels \(\lambda\).}
    \label{fig:srb-noise-sweep}
\end{figure}

\begin{table}[H]
\caption{Per-shape SRB reconstruction under clean input (\(\lambda=0\)) and
scanner noise (\(\lambda=1\)). Stochastic results use three seeds; best
neural implicit results are shown in bold.}
\label{tab:srb-results-noise}
\centering
\scriptsize
\setlength{\tabcolsep}{3.0pt}
\renewcommand{\arraystretch}{1.05}

\resizebox{\textwidth}{!}{%
\begin{tabular}{llcccccc@{\hspace{5mm}}cccccc}
\toprule
&
&
\multicolumn{6}{c}{\textbf{Clean} (\(\lambda=0\))}
&
\multicolumn{6}{c}{\textbf{Noise} (\(\lambda=1\))}
\\
\cmidrule(lr){3-8}
\cmidrule(lr){9-14}

\textbf{Method}
& \textbf{Metric}
& Anchor
& Daratech
& DC
& Gargoyle
& L. Quas
& \textbf{Mean}
& Anchor
& Daratech
& DC
& Gargoyle
& L. Quas
& \textbf{Mean}
\\
\midrule

\multicolumn{14}{l}{\textit{Neural implicit}} \\

\multirow{2}{*}{Ours ($L=0$)}
&$\mathrm{CD}_1$  $\downarrow$
& 0.195 & 0.286 & 0.107 & 0.149 & 0.077
& 0.163 {\scriptsize $\pm$ 0.005}
& \textbf{0.337} & \textbf{0.412} & \textbf{0.261} & \textbf{0.291} & \textbf{0.252}
& \textbf{0.310} {\scriptsize $\pm$ 0.016}
\\
& $\mathrm{NC}\uparrow$
& 0.978 & 0.903 & 0.973 & 0.952 & 0.983
& 0.958 {\scriptsize $\pm$ 0.0011}
& \textbf{0.939} & \textbf{0.851} & \textbf{0.915} & \textbf{0.902} & 0.920
& \textbf{0.906} {\scriptsize $\pm$ 0.0060}
\\[1mm]

\multirow{2}{*}{Ours ($L=2$)}
&$\mathrm{CD}_1$  $\downarrow$
& 0.186 & 0.200 & 0.088 & 0.108 & 0.055
& 0.127 {\scriptsize $\pm$ 0.002}
& 0.403 & 0.445 & 0.368 & 0.376 & 0.312
& 0.381 {\scriptsize $\pm$ 0.002}
\\
& $\mathrm{NC}\uparrow$
& 0.979 & 0.934 & 0.979 & 0.969 & 0.991
& 0.970 {\scriptsize $\pm$ 0.001}
& 0.865 & 0.676 & 0.757 & 0.818 & 0.788
& 0.781 {\scriptsize $\pm$ 0.002}
\\[1mm]

\multirow{2}{*}{Ours ($L=6$)}
&$\mathrm{CD}_1$  $\downarrow$
& 0.201 & 0.236 & 0.080 & \textbf{0.092} & \textbf{0.054}
& 0.133 {\scriptsize $\pm$ 0.003}
& 0.664 & 0.488 & 0.516 & 0.549 & 0.490
& 0.542 {\scriptsize $\pm$ 0.009}
\\
& $\mathrm{NC}\uparrow$
& 0.973 & 0.949 & 0.983 & 0.975 & 0.991
& 0.974 {\scriptsize $\pm$ 0.002}
& 0.542 & 0.534 & 0.550 & 0.544 & 0.542
& 0.543 {\scriptsize $\pm$ 0.008}
\\[1mm]

\multirow{2}{*}{NDF-vis}
&$\mathrm{CD}_1$  $\downarrow$
& 0.582 & 0.600 & 0.777 & 0.815 & 0.290
& 0.613 {\scriptsize $\pm$ 0.063}
& 1.069 & 1.395 & 1.076 & 1.162 & 0.788
& 1.098 {\scriptsize $\pm$ 0.194}
\\
& $\mathrm{NC}\uparrow$
& 0.909 & 0.861 & 0.879 & 0.863 & 0.953
& 0.893 {\scriptsize $\pm$ 0.0029}
& 0.896 & 0.820 & 0.862 & 0.851 & 0.925
& 0.871 {\scriptsize $\pm$ 0.0030}
\\[1mm]

\multirow{2}{*}{Neural-Pull}
&$\mathrm{CD}_1$  $\downarrow$
& \textbf{0.163} & 0.508 & 0.188 & 0.113 & 0.070
& 0.208 {\scriptsize $\pm$ 0.051}
& 0.522 & 0.464 & 0.441 & 0.449 & 0.375
& 0.450 {\scriptsize $\pm$ 0.003}
\\
& $\mathrm{NC}\uparrow$
& 0.977 & 0.848 & 0.946 & 0.969 & 0.985
& 0.945 {\scriptsize $\pm$ 0.004}
& 0.767 & 0.667 & 0.701 & 0.719 & 0.729
& 0.717 {\scriptsize $\pm$ 0.003}
\\[1mm]

\multirow{2}{*}{DiffCD}
&$\mathrm{CD}_1$  $\downarrow$
& 0.203 & 0.175 & 0.079 & 0.106 & 0.076
& 0.128 {\scriptsize $\pm$ 0.001}
& 0.456 & 0.458 & 0.408 & 0.397 & 0.355
& 0.415 {\scriptsize $\pm$ 0.006}
\\
& $\mathrm{NC}\uparrow$
& 0.979 & 0.947 & 0.984 & 0.973 & 0.982
& 0.973 {\scriptsize $\pm$ 0.0010}
& 0.824 & 0.808 & 0.792 & 0.753 & 0.728
& 0.781 {\scriptsize $\pm$ 0.0016}
\\[1mm]

\multirow{2}{*}{SAL}
&$\mathrm{CD}_1$  $\downarrow$
& 0.786 & 0.976 & 0.551 & 0.840 & 0.504
& 0.731 {\scriptsize $\pm$ 0.148}
& 1.169 & 2.059 & 1.070 & 0.726 & 0.601
& 1.125 {\scriptsize $\pm$ 0.322}
\\
& $\mathrm{NC}\uparrow$
& 0.901 & 0.788 & 0.890 & 0.894 & 0.952
& 0.885 {\scriptsize $\pm$ 0.0129}
& 0.888 & 0.688 & 0.846 & 0.865 & \textbf{0.937}
& 0.845 {\scriptsize $\pm$ 0.0026}
\\[1mm]

\multirow{2}{*}{DiGS}
&$\mathrm{CD}_1$  $\downarrow$
& 0.179 & \textbf{0.125} & \textbf{0.076} & 0.098 & 0.064
& \textbf{0.108} {\scriptsize $\pm$ 0.004}
& 0.712 & 0.505 & 0.577 & 0.608 & 0.650
& 0.610 {\scriptsize $\pm$ 0.006}
\\
& $\mathrm{NC}\uparrow$
& \textbf{0.982} & \textbf{0.970} & \textbf{0.992} & \textbf{0.983} & \textbf{0.994}
& \textbf{0.984} {\scriptsize $\pm$ 0.0006}
& 0.594 & 0.556 & 0.578 & 0.575 & 0.587
& 0.578 {\scriptsize $\pm$ 0.0037}
\\

\midrule
\multicolumn{14}{l}{\textit{Classical}} \\

\multirow{2}{*}{SPSR}
&$\mathrm{CD}_1$  $\downarrow$
& 0.269 & 0.333 & 0.165 & 0.160 & 0.095
& 0.204
& 0.412 & 0.512 & 0.325 & 0.316 & 0.223
& 0.358
\\
& $\mathrm{NC}\uparrow$
& 0.968 & 0.913 & 0.964 & 0.954 & 0.980
& 0.956
& 0.897 & 0.836 & 0.891 & 0.869 & 0.914
& 0.881
\\[1mm]

\multirow{2}{*}{Labatut}
&$\mathrm{CD}_1$  $\downarrow$
& 0.202 & 0.198 & 0.101 & 0.109 & 0.073
& 0.137
& 0.467 & 0.454 & 0.390 & 0.410 & 0.331
& 0.410
\\
& $\mathrm{NC}\uparrow$
& 0.971 & 0.895 & 0.962 & 0.959 & 0.977
& 0.953
& 0.656 & 0.596 & 0.651 & 0.643 & 0.647
& 0.639
\\

\bottomrule
\end{tabular}%
}
\end{table}


\begin{table}[H]
\caption{Per-shape observation-supported SRB results under clean input
(\(\lambda=0\)) and scanner noise (\(\lambda=1\)), using the \(2h_0\)
support threshold. Stochastic results use three seeds; best neural implicit
results are shown in bold.}
\label{tab:srb-results-noise-observed}
\centering
\scriptsize
\setlength{\tabcolsep}{3.0pt}
\renewcommand{\arraystretch}{1.05}

\resizebox{\textwidth}{!}{%
\begin{tabular}{llcccccc@{\hspace{5mm}}cccccc}
\toprule
&
&
\multicolumn{6}{c}{\textbf{Clean} (\(\lambda=0\))}
&
\multicolumn{6}{c}{\textbf{Noise} (\(\lambda=1\))}
\\
\cmidrule(lr){3-8}
\cmidrule(lr){9-14}

\textbf{Method}
& \textbf{Metric}
& Anchor
& Daratech
& DC
& Gargoyle
& L. Quas
& \textbf{Mean}
& Anchor
& Daratech
& DC
& Gargoyle
& L. Quas
& \textbf{Mean}
\\
\midrule

\multicolumn{14}{l}{\textit{Neural implicit}} \\

\multirow{2}{*}{Ours ($L=0$)}
& $\mathrm{CD}_1^{\mathrm{obs}}\downarrow$
& 0.078 & 0.146 & 0.090 & 0.126 & 0.066
& 0.101 {\scriptsize $\pm$ 0.003}
& \textbf{0.194} & \textbf{0.262} & \textbf{0.234} & \textbf{0.245} & \textbf{0.196}
& \textbf{0.226} {\scriptsize $\pm$ 0.005}
\\
& $\mathrm{NC}^{\mathrm{obs}}\uparrow$
& 0.991 & 0.932 & 0.980 & 0.961 & 0.987
& 0.970 {\scriptsize $\pm$ 0.0007}
& \textbf{0.956} & \textbf{0.875} & \textbf{0.926} & \textbf{0.916} & 0.931
& \textbf{0.921} {\scriptsize $\pm$ 0.004}
\\[1mm]

\multirow{2}{*}{Ours ($L=2$)}
& $\mathrm{CD}_1^{\mathrm{obs}}\downarrow$
& 0.070 & 0.077 & 0.073 & 0.094 & 0.050
& 0.073 {\scriptsize $\pm$ 0.001}
& 0.248 & 0.325 & 0.327 & 0.303 & 0.277
& 0.296 {\scriptsize $\pm$ 0.004}
\\
& $\mathrm{NC}^{\mathrm{obs}}\uparrow$
& 0.993 & 0.962 & 0.987 & 0.976 & 0.994
& 0.982 {\scriptsize $\pm$ 0.0002}
& 0.889 & 0.686 & 0.768 & 0.833 & 0.794
& 0.794 {\scriptsize $\pm$ 0.002}
\\[1mm]

\multirow{2}{*}{Ours ($L=6$)}
& $\mathrm{CD}_1^{\mathrm{obs}}\downarrow$
& 0.068 & \textbf{0.063} & 0.065 & \textbf{0.076} & \textbf{0.047}
& \textbf{0.064} {\scriptsize $\pm$ 0.002}
& 0.441 & 0.354 & 0.422 & 0.446 & 0.382
& 0.409 {\scriptsize $\pm$ 0.007}
\\
& $\mathrm{NC}^{\mathrm{obs}}\uparrow$
& 0.993 & 0.977 & 0.989 & 0.983 & 0.995
& 0.987 {\scriptsize $\pm$ 0.003}
& 0.545 & 0.534 & 0.551 & 0.546 & 0.544
& 0.544 {\scriptsize $\pm$ 0.006}
\\[1mm]

\multirow{2}{*}{NDF-vis}
& $\mathrm{CD}_1^{\mathrm{obs}}\downarrow$
& 0.318 & 0.356 & 0.578 & 0.620 & 0.245
& 0.423 {\scriptsize $\pm$ 0.037}
& 0.760 & 0.960 & 0.851 & 0.947 & 0.650
& 0.834 {\scriptsize $\pm$ 0.149}
\\
& $\mathrm{NC}^{\mathrm{obs}}\uparrow$
& 0.943 & 0.896 & 0.907 & 0.890 & 0.963
& 0.920 {\scriptsize $\pm$ 0.003}
& 0.926 & 0.867 & 0.888 & 0.876 & 0.947
& 0.901 {\scriptsize $\pm$ 0.0004}
\\[1mm]

\multirow{2}{*}{Neural-Pull}
& $\mathrm{CD}_1^{\mathrm{obs}}\downarrow$
& 0.074 & 0.256 & 0.157 & 0.089 & 0.052
& 0.126 {\scriptsize $\pm$ 0.031}
& 0.383 & 0.401 & 0.414 & 0.421 & 0.352
& 0.394 {\scriptsize $\pm$ 0.001}
\\
& $\mathrm{NC}^{\mathrm{obs}}\uparrow$
& 0.992 & 0.875 & 0.957 & 0.979 & 0.991
& 0.959 {\scriptsize $\pm$ 0.005}
& 0.771 & 0.667 & 0.702 & 0.722 & 0.727
& 0.718 {\scriptsize $\pm$ 0.003}
\\[1mm]

\multirow{2}{*}{DiffCD}
& $\mathrm{CD}_1^{\mathrm{obs}}\downarrow$
& \textbf{0.065} & 0.092 & \textbf{0.063} & 0.078 & 0.051
& 0.070 {\scriptsize $\pm$ 0.0003}
& 0.307 & 0.375 & 0.381 & 0.371 & 0.334
& 0.354 {\scriptsize $\pm$ 0.003}
\\
& $\mathrm{NC}^{\mathrm{obs}}\uparrow$
& \textbf{0.996} & 0.966 & 0.992 & 0.985 & 0.993
& 0.986 {\scriptsize $\pm$ 0.0003}
& 0.840 & 0.821 & 0.795 & 0.759 & 0.731
& 0.789 {\scriptsize $\pm$ 0.001}
\\[1mm]

\multirow{2}{*}{SAL}
& $\mathrm{CD}_1^{\mathrm{obs}}\downarrow$
& 0.522 & 0.676 & 0.454 & 0.723 & 0.459
& 0.567 {\scriptsize $\pm$ 0.122}
& 0.889 & 1.803 & 0.952 & 0.663 & 0.546
& 0.970 {\scriptsize $\pm$ 0.319}
\\
& $\mathrm{NC}^{\mathrm{obs}}\uparrow$
& 0.932 & 0.849 & 0.914 & 0.910 & 0.964
& 0.914 {\scriptsize $\pm$ 0.010}
& 0.914 & 0.727 & 0.868 & 0.879 & \textbf{0.949}
& 0.867 {\scriptsize $\pm$ 0.001}
\\[1mm]

\multirow{2}{*}{DiGS}
& $\mathrm{CD}_1^{\mathrm{obs}}\downarrow$
& 0.081 & 0.064 & 0.068 & 0.083 & 0.056
& 0.071 {\scriptsize $\pm$ 0.001}
& 0.515 & 0.434 & 0.496 & 0.526 & 0.467
& 0.488 {\scriptsize $\pm$ 0.001}
\\
& $\mathrm{NC}^{\mathrm{obs}}\uparrow$
& 0.994 & \textbf{0.986} & \textbf{0.996} & \textbf{0.991} & \textbf{0.997}
& \textbf{0.993} {\scriptsize $\pm$ 0.0002}
& 0.592 & 0.549 & 0.575 & 0.572 & 0.589
& 0.575 {\scriptsize $\pm$ 0.004}
\\

\midrule
\multicolumn{14}{l}{\textit{Classical}} \\

\multirow{2}{*}{SPSR}
& $\mathrm{CD}_1^{\mathrm{obs}}\downarrow$
& 0.104 & 0.155 & 0.120 & 0.127 & 0.073
& 0.116
& 0.239 & 0.349 & 0.283 & 0.283 & 0.203
& 0.271
\\
& $\mathrm{NC}^{\mathrm{obs}}\uparrow$
& 0.983 & 0.940 & 0.975 & 0.965 & 0.987
& 0.970
& 0.907 & 0.851 & 0.897 & 0.876 & 0.920
& 0.890
\\[1mm]

\multirow{2}{*}{Labatut}
& $\mathrm{CD}_1^{\mathrm{obs}}\downarrow$
& 0.070 & 0.086 & 0.075 & 0.085 & 0.054
& 0.074
& 0.379 & 0.362 & 0.367 & 0.392 & 0.316
& 0.363
\\
& $\mathrm{NC}^{\mathrm{obs}}\uparrow$
& 0.986 & 0.926 & 0.974 & 0.970 & 0.986
& 0.968
& 0.654 & 0.600 & 0.649 & 0.642 & 0.649
& 0.639
\\

\bottomrule
\end{tabular}%
}
\end{table}

\subsection{SRB outliers}
\label{app:srb-outlier-sweep}

\begin{figure}[H]
    \centering
    \includegraphics[width=\linewidth]{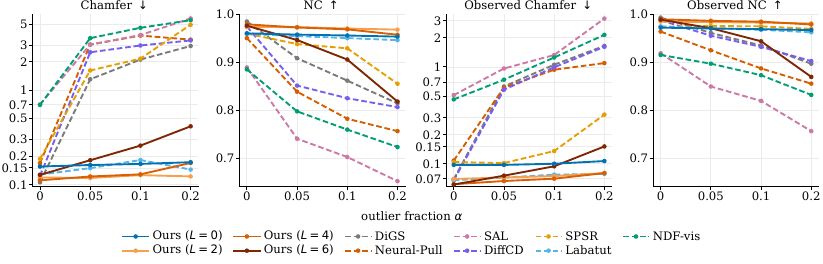}
    \caption{Mean SRB \(\mathrm{CD}_1\) and NC across increasing outlier levels \(\alpha\).}
    \label{fig:srb-outlier-sweep}
\end{figure}

\begin{figure}[H]
    \centering
    \includegraphics[width=\linewidth]
    {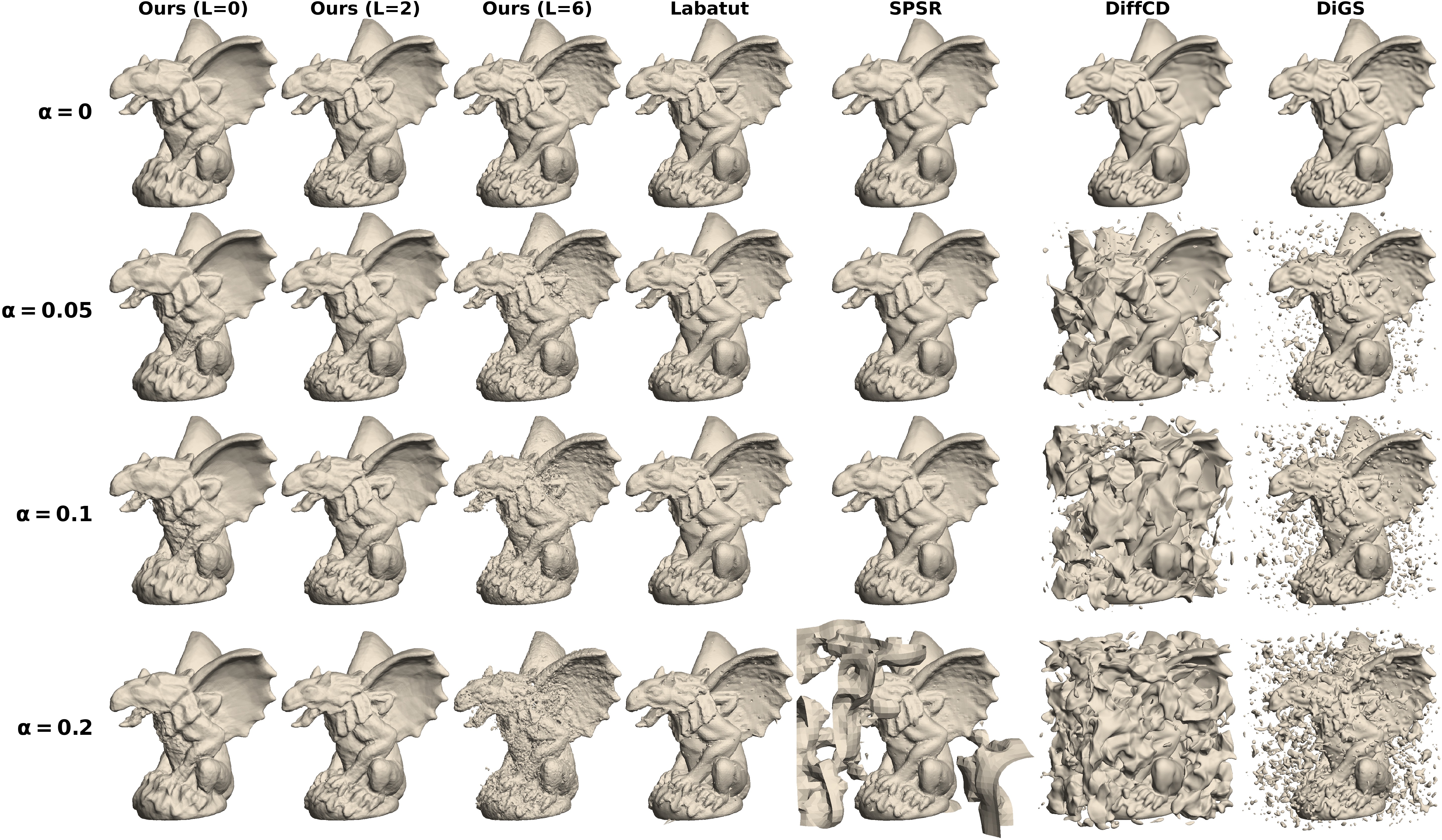}
    \caption{Qualitative Gargoyle reconstructions across increasing outlier levels \(\alpha\).}
    \label{fig:gargoyle-outlier-sweep}
\end{figure}


\begin{table}[H]
\caption{Per-shape SRB reconstruction without outliers (\(\alpha=0\)) and
with outliers (\(\alpha=0.2\)). Stochastic results use three seeds; best
neural implicit results are shown in bold.}
\label{tab:srb-results-outliers-full}
\centering
\scriptsize
\setlength{\tabcolsep}{3.0pt}
\renewcommand{\arraystretch}{1.05}

\resizebox{\textwidth}{!}{%
\begin{tabular}{llcccccc@{\hspace{5mm}}cccccc}
\toprule
&
&
\multicolumn{6}{c}{\textbf{Clean} ($\alpha=0$)}
&
\multicolumn{6}{c}{\textbf{Outliers} ($\alpha=0.2$)}
\\
\cmidrule(lr){3-8}
\cmidrule(lr){9-14}

\textbf{Method}
& \textbf{Metric}
& Anchor
& Daratech
& DC
& Gargoyle
& L. Quas
& \textbf{Mean}
& Anchor
& Daratech
& DC
& Gargoyle
& L. Quas
& \textbf{Mean}
\\
\midrule

\multicolumn{14}{l}{\textit{Neural implicit}} \\

\multirow{2}{*}{Ours ($L=0$)}
&$\mathrm{CD}_1$  $\downarrow$
& 0.196 & 0.274 & 0.098 & 0.141 & 0.067
& 0.155 {\scriptsize $\pm$ 0.003}
& 0.202 & 0.275 & 0.118 & 0.194 & 0.073
& 0.173 {\scriptsize $\pm$ 0.013}
\\
& $\mathrm{NC}\uparrow$
& 0.977 & 0.907 & 0.976 & 0.956 & 0.987
& 0.961 {\scriptsize $\pm$ 0.0015}
& 0.971 & 0.910 & 0.966 & 0.937 & 0.986
& 0.954 {\scriptsize $\pm$ 0.0032}
\\[1mm]

\multirow{2}{*}{Ours ($L=2$)}
&$\mathrm{CD}_1$  $\downarrow$
& 0.195 & 0.153 & 0.084 & 0.109 & 0.055
& 0.119 {\scriptsize $\pm$ 0.003}
& \textbf{0.149} & \textbf{0.154} & \textbf{0.100} & \textbf{0.142} & \textbf{0.065}
& \textbf{0.122} {\scriptsize $\pm$ 0.004}
\\
& $\mathrm{NC}\uparrow$
& 0.977 & 0.951 & 0.982 & 0.969 & 0.992
& 0.974 {\scriptsize $\pm$ 0.001}
& \textbf{0.982} & \textbf{0.939} & \textbf{0.974} & \textbf{0.958} & \textbf{0.990}
& \textbf{0.968} {\scriptsize $\pm$ 0.001}
\\[1mm]

\multirow{2}{*}{Ours ($L=6$)}
& $\mathrm{CD}_1\downarrow$
& 0.195 & 0.223 & 0.073 & \textbf{0.089} & \textbf{0.053}
& 0.127 {\scriptsize $\pm$ 0.004}
& 0.520 & 0.271 & 0.453 & 0.523 & 0.302
& 0.414 {\scriptsize $\pm$ 0.018}
\\
& $\mathrm{NC}\uparrow$
& 0.975 & 0.955 & 0.985 & 0.976 & 0.992
& 0.977 {\scriptsize $\pm$ 0.0009}
& 0.838 & 0.814 & 0.771 & 0.819 & 0.849
& 0.818 {\scriptsize $\pm$ 0.0018}
\\[1mm]

\multirow{2}{*}{NDF-vis}
&$\mathrm{CD}_1$  $\downarrow$
& 0.670 & 0.568 & 0.927 & 1.038 & 0.284
& 0.697 {\scriptsize $\pm$ 0.085}
& 5.835 & 6.375 & 5.355 & 6.042 & 3.851
& 5.492 {\scriptsize $\pm$ 0.280}
\\
& $\mathrm{NC}\uparrow$
& 0.907 & 0.850 & 0.858 & 0.856 & 0.956
& 0.885 {\scriptsize $\pm$ 0.0029}
& 0.733 & 0.656 & 0.709 & 0.747 & 0.769
& 0.723 {\scriptsize $\pm$ 0.0097}
\\[1mm]

\multirow{2}{*}{Neural-Pull}
&$\mathrm{CD}_1$  $\downarrow$
& \textbf{0.142} & 0.403 & 0.119 & 0.109 & 0.070
& 0.169 {\scriptsize $\pm$ 0.025}
& 2.132 & 5.526 & 2.790 & 3.663 & 3.002
& 3.42 {\scriptsize $\pm$ 0.15}
\\
& $\mathrm{NC}\uparrow$
& \textbf{0.982} & 0.853 & 0.966 & 0.970 & 0.983
& 0.951 {\scriptsize $\pm$ 0.003}
& 0.814 & 0.708 & 0.754 & 0.720 & 0.785
& 0.756 {\scriptsize $\pm$ 0.015}
\\[1mm]

\multirow{2}{*}{DiffCD}
&$\mathrm{CD}_1$  $\downarrow$
& 0.206 & 0.164 & \textbf{0.073} & 0.102 & 0.072
& 0.123 {\scriptsize $\pm$ 0.0024}
& 3.625 & 4.565 & 2.752 & 3.448 & 2.380
& 3.354 {\scriptsize $\pm$ 0.028}
\\
& $\mathrm{NC}\uparrow$
& 0.979 & 0.949 & 0.987 & 0.976 & 0.984
& 0.975 {\scriptsize $\pm$ 0.0017}
& 0.818 & 0.768 & 0.808 & 0.805 & 0.833
& 0.806 {\scriptsize $\pm$ 0.0011}
\\[1mm]

\multirow{2}{*}{SAL}
&$\mathrm{CD}_1$  $\downarrow$
& 0.895 & 0.884 & 0.666 & 0.693 & 0.382
& 0.704 {\scriptsize $\pm$ 0.068}
& 7.685 & 5.960 & 4.708 & 6.813 & 3.491
& 5.731 {\scriptsize $\pm$ 2.134}
\\
& $\mathrm{NC}\uparrow$
& 0.898 & 0.825 & 0.881 & 0.891 & 0.953
& 0.889 {\scriptsize $\pm$ 0.0094}
& 0.652 & 0.595 & 0.655 & 0.661 & 0.694
& 0.652 {\scriptsize $\pm$ 0.0314}
\\[1mm]

\multirow{2}{*}{DiGS}
&$\mathrm{CD}_1$  $\downarrow$
& 0.178 & \textbf{0.111} & 0.073 & 0.095 & 0.071
& \textbf{0.106} {\scriptsize $\pm$ 0.005}
& 3.369 & 4.319 & 2.159 & 2.388 & 2.402
& 2.927 {\scriptsize $\pm$ 0.063}
\\
& $\mathrm{NC}\uparrow$
& \textbf{0.982} & \textbf{0.976} & \textbf{0.993} & \textbf{0.984} & \textbf{0.993}
& \textbf{0.985} {\scriptsize $\pm$ 0.0006}
& 0.798 & 0.793 & 0.831 & 0.832 & 0.816
& 0.814 {\scriptsize $\pm$ 0.0020}
\\

\midrule
\multicolumn{14}{l}{\textit{Classical}} \\

\multirow{2}{*}{SPSR}
&$\mathrm{CD}_1$  $\downarrow$
& 0.253 & 0.314 & 0.136 & 0.149 & 0.086
& 0.187
& 4.661 & 8.320 & 4.747 & 4.841 & 2.017
& 4.917
\\
& $\mathrm{NC}\uparrow$
& 0.970 & 0.919 & 0.970 & 0.959 & 0.982
& 0.960
& 0.878 & 0.764 & 0.853 & 0.852 & 0.929
& 0.855
\\[1mm]

\multirow{2}{*}{Labatut}
&$\mathrm{CD}_1$  $\downarrow$
& 0.196 & 0.176 & 0.090 & 0.104 & 0.072
& 0.128
& 0.228 & 0.210 & 0.092 & 0.112 & 0.079
& 0.144
\\
& $\mathrm{NC}\uparrow$
& 0.972 & 0.919 & 0.969 & 0.962 & 0.977
& 0.960
& 0.958 & 0.901 & 0.958 & 0.950 & 0.965
& 0.946
\\

\bottomrule
\end{tabular}%
}
\end{table}


\begin{table}[H]
\caption{Per-shape observation-supported SRB results without outliers
(\(\alpha=0\)) and with outliers (\(\alpha=0.2\)), using the \(2h_0\)
support threshold. Stochastic results use three seeds; best neural implicit
results are shown in bold.}
\label{tab:srb-results-outliers-observed}
\centering
\scriptsize
\setlength{\tabcolsep}{3.0pt}
\renewcommand{\arraystretch}{1.05}

\resizebox{\textwidth}{!}{%
\begin{tabular}{llcccccc@{\hspace{5mm}}cccccc}
\toprule
&
&
\multicolumn{6}{c}{\textbf{Clean} ($\alpha=0$)}
&
\multicolumn{6}{c}{\textbf{Outliers} ($\alpha=0.2$)}
\\
\cmidrule(lr){3-8}
\cmidrule(lr){9-14}

\textbf{Method}
& \textbf{Metric}
& Anchor
& Daratech
& DC
& Gargoyle
& L. Quas
& \textbf{Mean}
& Anchor
& Daratech
& DC
& Gargoyle
& L. Quas
& \textbf{Mean}
\\
\midrule

\multicolumn{14}{l}{\textit{Neural implicit}} \\

\multirow{2}{*}{Ours ($L=0$)}
& $\mathrm{CD}_1^{\mathrm{obs}}\downarrow$
& 0.078 & 0.141 & 0.086 & 0.123 & 0.060
& 0.098 {\scriptsize $\pm$ 0.002}
& 0.096 & 0.135 & 0.099 & 0.137 & 0.068
& 0.107 {\scriptsize $\pm$ 0.005}
\\
& $\mathrm{NC}^{\mathrm{obs}}\uparrow$
& 0.991 & 0.936 & 0.982 & 0.963 & 0.990
& 0.973 {\scriptsize $\pm$ 0.0009}
& 0.984 & 0.939 & 0.974 & 0.952 & 0.988
& 0.968 {\scriptsize $\pm$ 0.002}
\\[1mm]

\multirow{2}{*}{Ours ($L=2$)}
& $\mathrm{CD}_1^{\mathrm{obs}}\downarrow$
& 0.069 & 0.067 & 0.070 & 0.094 & 0.050
& 0.070 {\scriptsize $\pm$ 0.002}
& \textbf{0.078} & \textbf{0.078} & \textbf{0.081} & \textbf{0.105} & \textbf{0.060}
& \textbf{0.080} {\scriptsize $\pm$ 0.002}
\\
& $\mathrm{NC}^{\mathrm{obs}}\uparrow$
& 0.993 & 0.974 & 0.988 & 0.977 & 0.994
& 0.985 {\scriptsize $\pm$ 0.001}
& \textbf{0.993} & \textbf{0.968} & \textbf{0.983} & \textbf{0.971} & \textbf{0.992}
& \textbf{0.981} {\scriptsize $\pm$ 0.001}
\\[1mm]

\multirow{2}{*}{Ours ($L=6$)}
& $\mathrm{CD}_1^{\mathrm{obs}}\downarrow$
& 0.067 & \textbf{0.058} & 0.062 & \textbf{0.074} & \textbf{0.046}
& \textbf{0.061} {\scriptsize $\pm$ 0.002}
& 0.171 & 0.131 & 0.174 & 0.160 & 0.120
& 0.151 {\scriptsize $\pm$ 0.007}
\\
& $\mathrm{NC}^{\mathrm{obs}}\uparrow$
& 0.994 & 0.981 & 0.991 & 0.985 & 0.995
& 0.989 {\scriptsize $\pm$ 0.0006}
& 0.906 & 0.858 & 0.824 & 0.879 & 0.881
& 0.870 {\scriptsize $\pm$ 0.0003}
\\[1mm]

\multirow{2}{*}{NDF-vis}
& $\mathrm{CD}_1^{\mathrm{obs}}\downarrow$
& 0.380 & 0.334 & 0.618 & 0.743 & 0.238
& 0.463 {\scriptsize $\pm$ 0.054}
& 2.027 & 1.259 & 2.629 & 3.222 & 1.586
& 2.144 {\scriptsize $\pm$ 0.268}
\\
& $\mathrm{NC}^{\mathrm{obs}}\uparrow$
& 0.943 & 0.886 & 0.893 & 0.890 & 0.966
& 0.916 {\scriptsize $\pm$ 0.0027}
& 0.851 & 0.818 & 0.808 & 0.815 & 0.868
& 0.832 {\scriptsize $\pm$ 0.0090}
\\[1mm]

\multirow{2}{*}{Neural-Pull}
& $\mathrm{CD}_1^{\mathrm{obs}}\downarrow$
& 0.071 & 0.246 & 0.088 & 0.085 & 0.051
& 0.108 {\scriptsize $\pm$ 0.017}
& 0.802 & 1.865 & 0.685 & 0.970 & 1.160
& 1.096 {\scriptsize $\pm$ 0.128}
\\
& $\mathrm{NC}^{\mathrm{obs}}\uparrow$
& 0.994 & 0.880 & 0.976 & 0.981 & 0.992
& 0.964 {\scriptsize $\pm$ 0.005}
& 0.896 & 0.841 & 0.865 & 0.811 & 0.865
& 0.855 {\scriptsize $\pm$ 0.014}
\\[1mm]

\multirow{2}{*}{DiffCD}
& $\mathrm{CD}_1^{\mathrm{obs}}\downarrow$
& \textbf{0.065} & 0.089 & \textbf{0.060} & 0.075 & 0.049
& 0.068 {\scriptsize $\pm$ 0.0016}
& 1.649 & 2.211 & 1.332 & 1.682 & 1.199
& 1.614 {\scriptsize $\pm$ 0.0074}
\\
& $\mathrm{NC}^{\mathrm{obs}}\uparrow$
& \textbf{0.997} & 0.968 & 0.994 & 0.987 & 0.994
& 0.988 {\scriptsize $\pm$ 0.0013}
& 0.924 & 0.871 & 0.902 & 0.901 & 0.915
& 0.903 {\scriptsize $\pm$ 0.0014}
\\[1mm]

\multirow{2}{*}{SAL}
& $\mathrm{CD}_1^{\mathrm{obs}}\downarrow$
& 0.553 & 0.576 & 0.547 & 0.564 & 0.320
& 0.512 {\scriptsize $\pm$ 0.054}
& 5.855 & 1.760 & 2.417 & 3.983 & 1.780
& 3.159 {\scriptsize $\pm$ 2.164}
\\
& $\mathrm{NC}^{\mathrm{obs}}\uparrow$
& 0.931 & 0.883 & 0.910 & 0.907 & 0.964
& 0.919 {\scriptsize $\pm$ 0.0073}
& 0.747 & 0.745 & 0.759 & 0.723 & 0.806
& 0.756 {\scriptsize $\pm$ 0.0343}
\\[1mm]

\multirow{2}{*}{DiGS}
& $\mathrm{CD}_1^{\mathrm{obs}}\downarrow$
& 0.080 & 0.062 & 0.067 & 0.082 & 0.057
& 0.069 {\scriptsize $\pm$ 0.0019}
& 1.762 & 2.316 & 1.322 & 1.596 & 1.257
& 1.650 {\scriptsize $\pm$ 0.0175}
\\
& $\mathrm{NC}^{\mathrm{obs}}\uparrow$
& 0.995 & \textbf{0.989} & \textbf{0.996} & \textbf{0.991} & \textbf{0.997}
& \textbf{0.994} {\scriptsize $\pm$ 0.0003}
& 0.902 & 0.890 & 0.897 & 0.893 & 0.904
& 0.897 {\scriptsize $\pm$ 0.0014}
\\

\midrule
\multicolumn{14}{l}{\textit{Classical}} \\

\multirow{2}{*}{SPSR}
& $\mathrm{CD}_1^{\mathrm{obs}}\downarrow$
& 0.096 & 0.141 & 0.101 & 0.117 & 0.067
& 0.104
& 0.262 & 0.682 & 0.234 & 0.324 & 0.109
& 0.322
\\
& $\mathrm{NC}^{\mathrm{obs}}\uparrow$
& 0.986 & 0.945 & 0.980 & 0.970 & 0.989
& 0.974
& 0.983 & 0.934 & 0.977 & 0.968 & 0.988
& 0.970
\\[1mm]

\multirow{2}{*}{Labatut}
& $\mathrm{CD}_1^{\mathrm{obs}}\downarrow$
& 0.068 & 0.072 & 0.069 & 0.080 & 0.050
& 0.068
& 0.081 & 0.094 & 0.075 & 0.087 & 0.057
& 0.079
\\
& $\mathrm{NC}^{\mathrm{obs}}\uparrow$
& 0.988 & 0.948 & 0.980 & 0.974 & 0.987
& 0.975
& 0.975 & 0.931 & 0.969 & 0.963 & 0.976
& 0.963
\\

\bottomrule
\end{tabular}%
}
\end{table}
\subsection{Ray-aligned outliers}
\label{sec:ray-aligned-outliers}

We additionally corrupt the measured range while preserving each acquisition
ray direction. Table~\ref{tab:srb_ro0p1} reports results for $10\%$ corrupted rays.

\begin{table}[H]
\caption{\textbf{SRB with ray-aligned outliers.}
Left: reconstruction metrics with $10\%$ of rays having corrupted range.
Right: qualitative results on Gargoyle.}
\label{tab:srb_ro0p1}
\centering

\begin{minipage}[c]{0.67\linewidth}
\centering
\scriptsize
\setlength{\tabcolsep}{2.2pt}
\renewcommand{\arraystretch}{1.03}

\resizebox{\linewidth}{!}{%
\begin{tabular}{lcccccc}
\toprule
\textbf{Method}
& \textbf{Anchor}
& \textbf{Daratech}
& \textbf{DC}
& \textbf{Gargoyle}
& \textbf{Lord Quas}
& \textbf{Mean} \\
\midrule

\multicolumn{7}{l}{\textit{$\mathrm{CD}_1\downarrow$}} \\
Labatut
& 1.166 & 0.213 & 0.107 & 0.353 & 0.082 & 0.384 \\
Ours ($L=0$)
& 0.264 & 0.271 & 0.109 & 0.161 & 0.078 & 0.177 \\
Ours ($L=2$)
& \textbf{0.206} & \textbf{0.177} & \textbf{0.099}
& \textbf{0.134} & \textbf{0.068} & \textbf{0.137} \\

\midrule
\multicolumn{7}{l}{\textit{$\mathrm{NC}\uparrow$}} \\
Labatut
& 0.948 & 0.894 & 0.955 & 0.948 & 0.966 & 0.942 \\
Ours ($L=0$)
& 0.966 & 0.911 & 0.972 & 0.943 & 0.981 & 0.954 \\
Ours ($L=2$)
& \textbf{0.979} & \textbf{0.931} & \textbf{0.974}
& \textbf{0.956} & \textbf{0.988} & \textbf{0.965} \\

\bottomrule
\end{tabular}%
}
\end{minipage}
\hfill
\begin{minipage}[c]{0.30\linewidth}
\centering
\includegraphics[width=\linewidth]{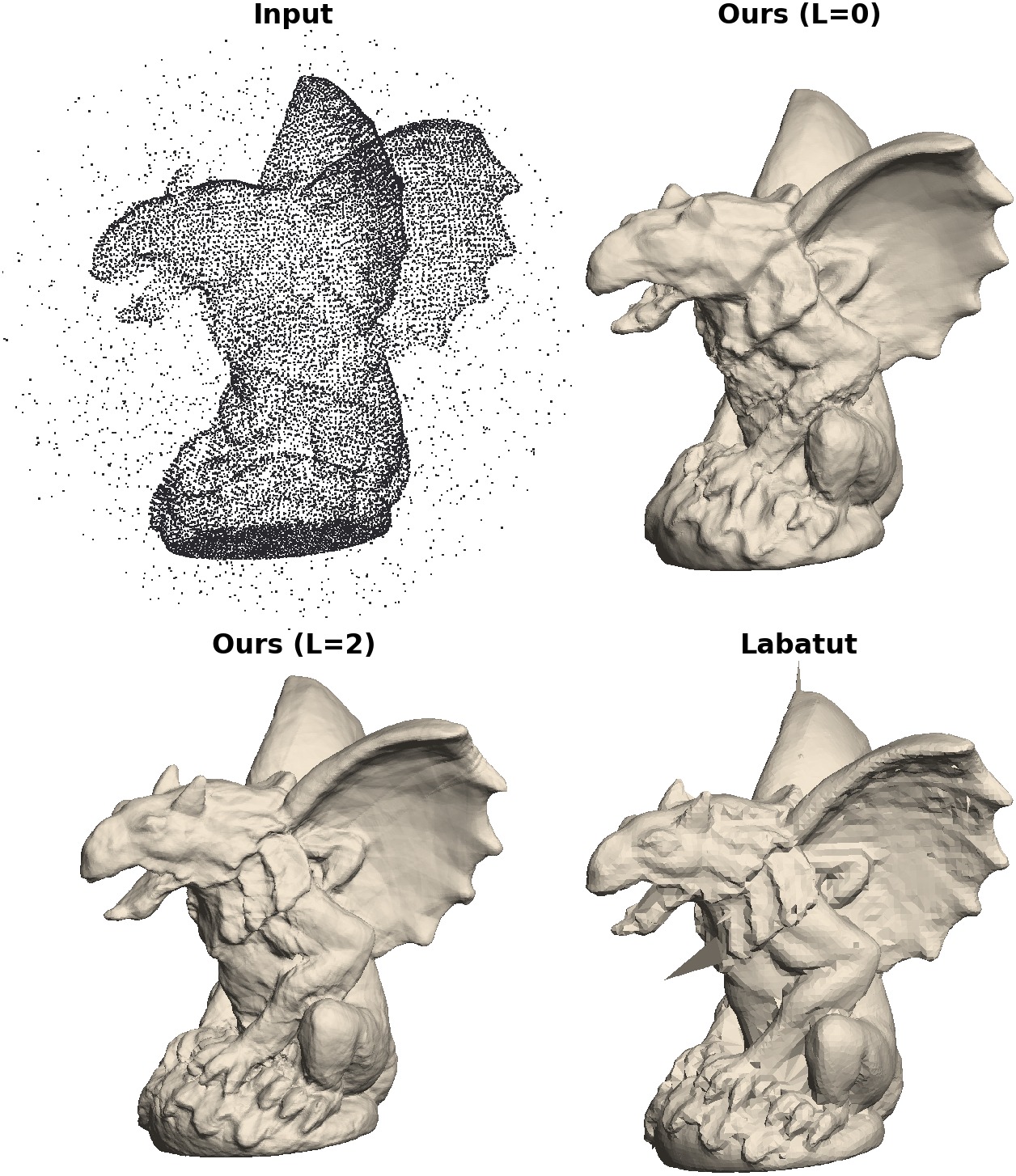}
\end{minipage}
\end{table}

\subsection{ShapeNet extended results}
\label{app:shapenet-extended}


\begin{table}[H]
\caption{ShapeNet category-wise \(\mathrm{CD}_2\)  under isotropic and
ray-aligned Gaussian noise with 3K input points. Best neural implicit results
are shown in bold.}
\label{tab:shapenet-categories-cd2}
\centering
\footnotesize
\setlength{\tabcolsep}{3pt}
\renewcommand{\arraystretch}{1.05}

\resizebox{\textwidth}{!}{%
\begin{tabular}{llccccccccc}
\toprule
\textbf{Category} & \textbf{Noise}
& \textbf{Ours ($L=0$)} & \textbf{Ours ($L=6$)}
& \textbf{NDF-vis} & \textbf{DiffCD}
& \textbf{N.-P} & \textbf{SAL} & \textbf{DiGS}
& \textbf{SPSR} & \textbf{Labat.} \\
\midrule

Airplane & Iso.
& $\boldsymbol{3.72\mathrm{e}^{-5}}$ & $5.16\mathrm{e}^{-5}$
& $4.32\mathrm{e}^{-4}$ & $5.57\mathrm{e}^{-5}$
& $1.75\mathrm{e}^{-4}$ & $7.70\mathrm{e}^{-3}$ & $1.07\mathrm{e}^{-4}$
& $9.66\mathrm{e}^{-5}$ & $8.93\mathrm{e}^{-5}$ \\
& Ray
& $\boldsymbol{2.91\mathrm{e}^{-5}}$ & $3.78\mathrm{e}^{-5}$
& $1.18\mathrm{e}^{-3}$ & $4.30\mathrm{e}^{-5}$
& $1.12\mathrm{e}^{-3}$ & $6.81\mathrm{e}^{-3}$ & $1.24\mathrm{e}^{-4}$
& $9.36\mathrm{e}^{-5}$ & $1.25\mathrm{e}^{-4}$ \\
\addlinespace

Bench & Iso.
& $1.39\mathrm{e}^{-4}$ & $\boldsymbol{1.33\mathrm{e}^{-4}}$
& $3.09\mathrm{e}^{-4}$ & $1.74\mathrm{e}^{-4}$
& $1.75\mathrm{e}^{-3}$ & $1.93\mathrm{e}^{-3}$ & $2.52\mathrm{e}^{-4}$
& $1.08\mathrm{e}^{-2}$ & $1.73\mathrm{e}^{-4}$ \\
& Ray
& $1.59\mathrm{e}^{-4}$ & $\boldsymbol{1.08\mathrm{e}^{-4}}$
& $3.43\mathrm{e}^{-4}$ & $1.65\mathrm{e}^{-4}$
& $5.29\mathrm{e}^{-4}$ & $1.15\mathrm{e}^{-2}$ & $2.13\mathrm{e}^{-4}$
& $1.11\mathrm{e}^{-2}$ & $2.11\mathrm{e}^{-4}$ \\
\addlinespace

Cabinet & Iso.
& $9.73\mathrm{e}^{-4}$ & $\boldsymbol{8.52\mathrm{e}^{-4}}$
& $1.27\mathrm{e}^{-3}$ & $1.12\mathrm{e}^{-3}$
& $1.79\mathrm{e}^{-3}$ & $3.69\mathrm{e}^{-3}$ & $1.41\mathrm{e}^{-3}$
& $3.22\mathrm{e}^{-3}$ & $9.58\mathrm{e}^{-4}$ \\
& Ray
& $9.40\mathrm{e}^{-4}$ & $\boldsymbol{7.55\mathrm{e}^{-4}}$
& $1.51\mathrm{e}^{-3}$ & $1.12\mathrm{e}^{-3}$
& $1.82\mathrm{e}^{-3}$ & $4.06\mathrm{e}^{-3}$ & $1.39\mathrm{e}^{-3}$
& $3.10\mathrm{e}^{-3}$ & $9.85\mathrm{e}^{-4}$ \\
\addlinespace

Car & Iso.
& $3.75\mathrm{e}^{-4}$ & $3.72\mathrm{e}^{-4}$
& $7.06\mathrm{e}^{-4}$ & $\boldsymbol{3.11\mathrm{e}^{-4}}$
& $5.38\mathrm{e}^{-4}$ & $7.37\mathrm{e}^{-4}$ & $3.85\mathrm{e}^{-4}$
& $5.56\mathrm{e}^{-4}$ & $3.34\mathrm{e}^{-4}$ \\
& Ray
& $3.69\mathrm{e}^{-4}$ & $3.89\mathrm{e}^{-4}$
& $9.39\mathrm{e}^{-4}$ & $\boldsymbol{3.05\mathrm{e}^{-4}}$
& $7.29\mathrm{e}^{-4}$ & $7.85\mathrm{e}^{-4}$ & $3.16\mathrm{e}^{-4}$
& $5.54\mathrm{e}^{-4}$ & $3.57\mathrm{e}^{-4}$ \\
\addlinespace

Chair & Iso.
& $3.30\mathrm{e}^{-4}$ & $\boldsymbol{2.55\mathrm{e}^{-4}}$
& $8.58\mathrm{e}^{-4}$ & $3.73\mathrm{e}^{-4}$
& $3.13\mathrm{e}^{-3}$ & $9.42\mathrm{e}^{-3}$ & $2.50\mathrm{e}^{-3}$
& $9.73\mathrm{e}^{-3}$ & $8.97\mathrm{e}^{-4}$ \\
& Ray
& $3.48\mathrm{e}^{-4}$ & $\boldsymbol{2.18\mathrm{e}^{-4}}$
& $7.67\mathrm{e}^{-4}$ & $3.50\mathrm{e}^{-4}$
& $2.86\mathrm{e}^{-3}$ & $8.81\mathrm{e}^{-3}$ & $3.00\mathrm{e}^{-3}$
& $9.27\mathrm{e}^{-3}$ & $6.79\mathrm{e}^{-4}$ \\
\addlinespace

Display & Iso.
& $6.76\mathrm{e}^{-5}$ & $\boldsymbol{5.10\mathrm{e}^{-5}}$
& $1.78\mathrm{e}^{-4}$ & $8.09\mathrm{e}^{-5}$
& $5.23\mathrm{e}^{-4}$ & $2.79\mathrm{e}^{-3}$ & $2.44\mathrm{e}^{-4}$
& $1.16\mathrm{e}^{-2}$ & $7.96\mathrm{e}^{-5}$ \\
& Ray
& $\boldsymbol{4.43\mathrm{e}^{-5}}$ & $4.71\mathrm{e}^{-5}$
& $2.06\mathrm{e}^{-4}$ & $7.65\mathrm{e}^{-5}$
& $5.70\mathrm{e}^{-4}$ & $2.27\mathrm{e}^{-3}$ & $1.51\mathrm{e}^{-4}$
& $1.17\mathrm{e}^{-2}$ & $8.62\mathrm{e}^{-5}$ \\
\addlinespace

Lamp & Iso.
& $\boldsymbol{1.45\mathrm{e}^{-4}}$ & $1.94\mathrm{e}^{-4}$
& $6.01\mathrm{e}^{-4}$ & $1.73\mathrm{e}^{-4}$
& $3.04\mathrm{e}^{-3}$ & $3.90\mathrm{e}^{-2}$ & $6.60\mathrm{e}^{-4}$
& $1.86\mathrm{e}^{-3}$ & $2.11\mathrm{e}^{-4}$ \\
& Ray
& $\boldsymbol{1.30\mathrm{e}^{-4}}$ & $2.02\mathrm{e}^{-4}$
& $4.85\mathrm{e}^{-4}$ & $1.99\mathrm{e}^{-4}$
& $7.13\mathrm{e}^{-3}$ & $1.29\mathrm{e}^{-2}$ & $5.35\mathrm{e}^{-4}$
& $1.74\mathrm{e}^{-3}$ & $2.26\mathrm{e}^{-4}$ \\
\addlinespace

Loudspeaker & Iso.
& $4.13\mathrm{e}^{-4}$ & $\boldsymbol{4.09\mathrm{e}^{-4}}$
& $5.22\mathrm{e}^{-4}$ & $4.11\mathrm{e}^{-4}$
& $5.09\mathrm{e}^{-4}$ & $6.34\mathrm{e}^{-4}$ & $6.54\mathrm{e}^{-4}$
& $6.55\mathrm{e}^{-4}$ & $4.35\mathrm{e}^{-4}$ \\
& Ray
& $4.06\mathrm{e}^{-4}$ & $\boldsymbol{4.02\mathrm{e}^{-4}}$
& $5.66\mathrm{e}^{-4}$ & $4.05\mathrm{e}^{-4}$
& $6.03\mathrm{e}^{-4}$ & $5.50\mathrm{e}^{-4}$ & $5.41\mathrm{e}^{-4}$
& $6.36\mathrm{e}^{-4}$ & $3.89\mathrm{e}^{-4}$ \\
\addlinespace

Rifle & Iso.
& $5.12\mathrm{e}^{-5}$ & $1.17\mathrm{e}^{-4}$
& $4.15\mathrm{e}^{-4}$ & $\boldsymbol{4.13\mathrm{e}^{-5}}$
& $7.23\mathrm{e}^{-5}$ & $1.62\mathrm{e}^{-3}$ & $9.39\mathrm{e}^{-5}$
& $4.62\mathrm{e}^{-5}$ & $5.56\mathrm{e}^{-5}$ \\
& Ray
& $4.08\mathrm{e}^{-5}$ & $9.71\mathrm{e}^{-5}$
& $7.06\mathrm{e}^{-4}$ & $\boldsymbol{3.50\mathrm{e}^{-5}}$
& $1.35\mathrm{e}^{-4}$ & $7.21\mathrm{e}^{-3}$ & $9.50\mathrm{e}^{-5}$
& $4.67\mathrm{e}^{-5}$ & $5.15\mathrm{e}^{-5}$ \\
\addlinespace

Sofa & Iso.
& $\boldsymbol{2.38\mathrm{e}^{-4}}$ & $3.46\mathrm{e}^{-4}$
& $5.47\mathrm{e}^{-4}$ & $3.43\mathrm{e}^{-4}$
& $6.38\mathrm{e}^{-4}$ & $2.36\mathrm{e}^{-3}$ & $3.97\mathrm{e}^{-4}$
& $1.29\mathrm{e}^{-3}$ & $4.01\mathrm{e}^{-4}$ \\
& Ray
& $\boldsymbol{2.28\mathrm{e}^{-4}}$ & $3.15\mathrm{e}^{-4}$
& $4.64\mathrm{e}^{-4}$ & $3.20\mathrm{e}^{-4}$
& $1.75\mathrm{e}^{-3}$ & $2.09\mathrm{e}^{-3}$ & $3.72\mathrm{e}^{-4}$
& $1.33\mathrm{e}^{-3}$ & $4.24\mathrm{e}^{-4}$ \\
\addlinespace

Table & Iso.
& $7.82\mathrm{e}^{-5}$ & $\boldsymbol{4.73\mathrm{e}^{-5}}$
& $6.11\mathrm{e}^{-4}$ & $6.91\mathrm{e}^{-5}$
& $6.05\mathrm{e}^{-4}$ & $1.54\mathrm{e}^{-2}$ & $2.74\mathrm{e}^{-4}$
& $8.29\mathrm{e}^{-3}$ & $1.55\mathrm{e}^{-4}$ \\
& Ray
& $5.48\mathrm{e}^{-5}$ & $\boldsymbol{3.09\mathrm{e}^{-5}}$
& $5.68\mathrm{e}^{-4}$ & $4.27\mathrm{e}^{-5}$
& $1.14\mathrm{e}^{-3}$ & $1.18\mathrm{e}^{-1}$ & $2.08\mathrm{e}^{-4}$
& $8.31\mathrm{e}^{-3}$ & $2.21\mathrm{e}^{-4}$ \\
\addlinespace

Telephone & Iso.
& $\boldsymbol{2.74\mathrm{e}^{-5}}$ & $4.48\mathrm{e}^{-5}$
& $1.30\mathrm{e}^{-4}$ & $6.54\mathrm{e}^{-5}$
& $1.52\mathrm{e}^{-4}$ & $7.97\mathrm{e}^{-2}$ & $2.04\mathrm{e}^{-4}$
& $1.10\mathrm{e}^{-2}$ & $4.10\mathrm{e}^{-5}$ \\
& Ray
& $\boldsymbol{1.66\mathrm{e}^{-5}}$ & $2.47\mathrm{e}^{-5}$
& $6.52\mathrm{e}^{-5}$ & $4.89\mathrm{e}^{-5}$
& $3.00\mathrm{e}^{-4}$ & $1.75\mathrm{e}^{-2}$ & $1.65\mathrm{e}^{-4}$
& $9.84\mathrm{e}^{-3}$ & $4.21\mathrm{e}^{-5}$ \\
\addlinespace

Watercraft & Iso.
& $4.78\mathrm{e}^{-4}$ & $3.34\mathrm{e}^{-4}$
& $5.81\mathrm{e}^{-4}$ & $\boldsymbol{7.21\mathrm{e}^{-5}}$
& $2.60\mathrm{e}^{-4}$ & $2.34\mathrm{e}^{-3}$ & $2.62\mathrm{e}^{-4}$
& $7.80\mathrm{e}^{-4}$ & $1.27\mathrm{e}^{-4}$ \\
& Ray
& $1.33\mathrm{e}^{-4}$ & $2.11\mathrm{e}^{-4}$
& $8.56\mathrm{e}^{-4}$ & $\boldsymbol{6.62\mathrm{e}^{-5}}$
& $5.86\mathrm{e}^{-4}$ & $9.79\mathrm{e}^{-4}$ & $2.15\mathrm{e}^{-4}$
& $7.41\mathrm{e}^{-4}$ & $1.25\mathrm{e}^{-4}$ \\

\bottomrule
\end{tabular}%
}
\end{table}


\begin{table}[H]
\caption{ShapeNet category-wise normal consistency under isotropic and
ray-aligned Gaussian noise with 3K input points. Best neural implicit results
are shown in bold.}
\label{tab:shapenet-categories-nc}

\centering
\footnotesize
\setlength{\tabcolsep}{3pt}
\renewcommand{\arraystretch}{1.05}

\resizebox{\textwidth}{!}{%
\begin{tabular}{llccccccccc}
\toprule
\textbf{Category} & \textbf{Noise}
& \textbf{Ours ($L=0$)} & \textbf{Ours ($L=6$)}
& \textbf{NDF-vis} & \textbf{DiffCD}
& \textbf{N.-P} & \textbf{SAL} & \textbf{DiGS}
& \textbf{SPSR} & \textbf{Labat.} \\
\midrule

Airplane & Iso.
& \textbf{0.90} & 0.87 & 0.86 & 0.82 & 0.73 & 0.79 & 0.62 & 0.85 & 0.73 \\
& Ray
& \textbf{0.92} & 0.90 & 0.84 & 0.87 & 0.72 & 0.79 & 0.66 & 0.84 & 0.75 \\
\addlinespace

Bench & Iso.
& 0.88 & 0.88 & 0.88 & 0.83 & 0.75 & \textbf{0.89} & 0.65 & 0.76 & 0.77 \\
& Ray
& \textbf{0.88} & 0.88 & 0.88 & 0.86 & 0.78 & 0.78 & 0.68 & 0.76 & 0.79 \\
\addlinespace

Cabinet & Iso.
& 0.86 & \textbf{0.87} & 0.83 & 0.84 & 0.76 & 0.83 & 0.66 & 0.81 & 0.79 \\
& Ray
& 0.87 & \textbf{0.87} & 0.83 & 0.85 & 0.78 & 0.83 & 0.68 & 0.81 & 0.80 \\
\addlinespace

Car & Iso.
& \textbf{0.89} & 0.88 & 0.89 & 0.81 & 0.80 & 0.87 & 0.66 & 0.86 & 0.80 \\
& Ray
& \textbf{0.90} & 0.89 & 0.88 & 0.84 & 0.83 & 0.87 & 0.72 & 0.87 & 0.82 \\
\addlinespace

Chair & Iso.
& \textbf{0.90} & 0.90 & 0.89 & 0.89 & 0.69 & 0.83 & 0.70 & 0.79 & 0.81 \\
& Ray
& 0.91 & \textbf{0.91} & 0.89 & 0.91 & 0.75 & 0.84 & 0.73 & 0.79 & 0.82 \\
\addlinespace

Display & Iso.
& 0.95 & \textbf{0.96} & 0.95 & 0.89 & 0.84 & 0.94 & 0.72 & 0.86 & 0.86 \\
& Ray
& \textbf{0.96} & 0.96 & 0.96 & 0.92 & 0.86 & 0.94 & 0.77 & 0.86 & 0.87 \\
\addlinespace

Lamp & Iso.
& \textbf{0.89} & 0.87 & 0.88 & 0.83 & 0.73 & 0.84 & 0.67 & 0.85 & 0.77 \\
& Ray
& \textbf{0.91} & 0.90 & 0.89 & 0.87 & 0.76 & 0.87 & 0.68 & 0.86 & 0.78 \\
\addlinespace

Loudspeaker & Iso.
& 0.95 & 0.94 & \textbf{0.95} & 0.88 & 0.87 & 0.94 & 0.71 & 0.92 & 0.87 \\
& Ray
& \textbf{0.96} & 0.95 & 0.96 & 0.91 & 0.89 & 0.95 & 0.76 & 0.92 & 0.90 \\
\addlinespace

Rifle & Iso.
& \textbf{0.88} & 0.85 & 0.80 & 0.76 & 0.74 & 0.85 & 0.68 & 0.86 & 0.71 \\
& Ray
& \textbf{0.89} & 0.88 & 0.79 & 0.81 & 0.74 & 0.83 & 0.65 & 0.87 & 0.74 \\
\addlinespace

Sofa & Iso.
& 0.91 & 0.91 & \textbf{0.92} & 0.86 & 0.80 & 0.87 & 0.72 & 0.86 & 0.81 \\
& Ray
& 0.92 & 0.91 & \textbf{0.93} & 0.90 & 0.80 & 0.88 & 0.75 & 0.87 & 0.83 \\
\addlinespace

Table & Iso.
& 0.91 & \textbf{0.93} & 0.90 & 0.89 & 0.69 & 0.69 & 0.67 & 0.79 & 0.78 \\
& Ray
& 0.92 & \textbf{0.95} & 0.90 & 0.93 & 0.75 & 0.73 & 0.73 & 0.79 & 0.80 \\
\addlinespace

Telephone & Iso.
& \textbf{0.97} & 0.96 & 0.96 & 0.90 & 0.89 & 0.85 & 0.68 & 0.85 & 0.86 \\
& Ray
& \textbf{0.97} & 0.97 & 0.97 & 0.93 & 0.89 & 0.87 & 0.72 & 0.85 & 0.87 \\
\addlinespace

Watercraft & Iso.
& \textbf{0.90} & 0.89 & 0.90 & 0.83 & 0.77 & 0.85 & 0.67 & 0.86 & 0.77 \\
& Ray
& \textbf{0.92} & 0.90 & 0.89 & 0.88 & 0.82 & 0.89 & 0.68 & 0.87 & 0.80 \\

\bottomrule
\end{tabular}%
}
\end{table}


\begin{table}[H]
\caption{ShapeNet category-wise IoU under isotropic and ray-aligned Gaussian
noise with 3K input points. Best neural implicit results are shown in bold.}
\label{tab:shapenet-categories-iou}

\centering
\footnotesize
\setlength{\tabcolsep}{3pt}
\renewcommand{\arraystretch}{1.05}

\resizebox{\textwidth}{!}{%
\begin{tabular}{llccccccccc}
\toprule
\textbf{Category} & \textbf{Noise}
& \textbf{Ours ($L=0$)} & \textbf{Ours ($L=6$)}
& \textbf{NDF-vis} & \textbf{DiffCD}
& \textbf{N.-P} & \textbf{SAL} & \textbf{DiGS}
& \textbf{SPSR} & \textbf{Labat.} \\
\midrule

Airplane & Iso.
& \textbf{0.82} & 0.79 & 0.57 & 0.75 & 0.63 & 0.57 & 0.23 & 0.74 & 0.73 \\
& Ray
& \textbf{0.86} & 0.83 & 0.43 & 0.79 & 0.44 & 0.60 & 0.52 & 0.75 & 0.69 \\
\addlinespace

Bench & Iso.
& \textbf{0.82} & 0.80 & 0.70 & 0.62 & 0.58 & 0.75 & 0.38 & 0.40 & 0.76 \\
& Ray
& \textbf{0.83} & 0.82 & 0.70 & 0.69 & 0.63 & 0.47 & 0.50 & 0.41 & 0.77 \\
\addlinespace

Cabinet & Iso.
& 0.82 & \textbf{0.85} & 0.76 & 0.74 & 0.71 & 0.71 & 0.58 & 0.66 & 0.81 \\
& Ray
& 0.84 & \textbf{0.85} & 0.74 & 0.74 & 0.71 & 0.69 & 0.62 & 0.66 & 0.81 \\
\addlinespace

Car & Iso.
& \textbf{0.89} & 0.88 & 0.81 & 0.85 & 0.81 & 0.82 & 0.57 & 0.86 & 0.87 \\
& Ray
& \textbf{0.90} & 0.89 & 0.82 & 0.87 & 0.83 & 0.83 & 0.78 & 0.86 & 0.89 \\
\addlinespace

Chair & Iso.
& 0.76 & \textbf{0.77} & 0.61 & 0.57 & 0.35 & 0.45 & 0.44 & 0.37 & 0.70 \\
& Ray
& 0.77 & \textbf{0.80} & 0.62 & 0.59 & 0.48 & 0.47 & 0.47 & 0.38 & 0.71 \\
\addlinespace

Display & Iso.
& \textbf{0.91} & 0.91 & 0.86 & 0.86 & 0.68 & 0.81 & 0.71 & 0.61 & 0.88 \\
& Ray
& \textbf{0.94} & 0.92 & 0.87 & 0.86 & 0.77 & 0.85 & 0.81 & 0.60 & 0.90 \\
\addlinespace

Lamp & Iso.
& \textbf{0.79} & 0.74 & 0.60 & 0.72 & 0.43 & 0.48 & 0.33 & 0.61 & 0.74 \\
& Ray
& \textbf{0.82} & 0.78 & 0.61 & 0.68 & 0.36 & 0.54 & 0.39 & 0.62 & 0.75 \\
\addlinespace

Loudspeaker & Iso.
& \textbf{0.94} & 0.93 & 0.92 & 0.92 & 0.91 & 0.89 & 0.73 & 0.91 & 0.93 \\
& Ray
& \textbf{0.95} & 0.94 & 0.92 & 0.93 & 0.92 & 0.91 & 0.81 & 0.91 & 0.94 \\
\addlinespace

Rifle & Iso.
& \textbf{0.82} & 0.75 & 0.57 & 0.75 & 0.71 & 0.65 & 0.23 & 0.79 & 0.76 \\
& Ray
& \textbf{0.84} & 0.80 & 0.48 & 0.78 & 0.59 & 0.56 & 0.31 & 0.79 & 0.74 \\
\addlinespace

Sofa & Iso.
& \textbf{0.84} & 0.83 & 0.77 & 0.73 & 0.69 & 0.68 & 0.65 & 0.68 & 0.78 \\
& Ray
& \textbf{0.85} & 0.84 & 0.79 & 0.78 & 0.59 & 0.70 & 0.69 & 0.68 & 0.80 \\
\addlinespace

Table & Iso.
& 0.75 & \textbf{0.79} & 0.60 & 0.64 & 0.35 & 0.19 & 0.40 & 0.40 & 0.63 \\
& Ray
& 0.79 & \textbf{0.85} & 0.56 & 0.73 & 0.36 & 0.21 & 0.52 & 0.40 & 0.64 \\
\addlinespace

Telephone & Iso.
& \textbf{0.92} & 0.89 & 0.83 & 0.84 & 0.79 & 0.64 & 0.54 & 0.49 & 0.86 \\
& Ray
& \textbf{0.93} & 0.92 & 0.88 & 0.86 & 0.76 & 0.64 & 0.66 & 0.50 & 0.87 \\
\addlinespace

Watercraft & Iso.
& \textbf{0.79} & 0.78 & 0.72 & 0.76 & 0.73 & 0.59 & 0.29 & 0.71 & 0.79 \\
& Ray
& \textbf{0.83} & 0.82 & 0.70 & 0.81 & 0.74 & 0.63 & 0.53 & 0.73 & 0.81 \\

\bottomrule
\end{tabular}%
}
\end{table}

\end{document}